\documentclass[twoside,11pt]{article}

\usepackage{subcaption}
\usepackage{amsmath}
\usepackage{multirow}
\usepackage{amssymb,ulem}

\usepackage{xspace}
\usepackage{adjustbox,enumitem}
\usepackage{microtype}
\usepackage{graphicx}
\usepackage{booktabs} 

\newenvironment{proofarg}[1]{\par\noindent{\bf #1\ }}{\hfill\BlackBox\\[2mm]}
\newenvironment{proofsketch}[1]{\par\noindent{\bf Proof Sketch\ }}{\hfill\BlackBox\\[2mm]}

\usepackage{algorithm}
\usepackage{algorithmic}

\DeclareMathOperator*{\argmin}{arg\,min}

\usepackage{color}

\usepackage{jmlr2e}
\jmlrheading{}{}{}{}{}{xu18regression}{Yichong Xu, Sivaraman Balakrishnan, Aarti Singh and Artur Dubrawski}

\ShortHeadings{Regression with Comparisons}{Xu, Balakrishnan, Singh and Dubrawski}
\firstpageno{1}

\newcommand{\regfunc}{f}
\newcommand{\estfunc}{\widehat{\regfunc}}

\newcommand{\RR}{R$^2$\xspace}
\newcommand{\labsam}[1]{t_{#1}}
\newcommand{\err}{\text{err}}
\newcommand{\var}{\text{var}}

\newcommand{\gtweight}{w^*}
\newcommand{\labelnoise}{\varepsilon}
\newcommand{\featureRV}{X}
\newcommand{\labelRV}{y}
\newcommand{\compRV}{Z}
\newcommand{\featureV}{x}
\newcommand{\labelV}{y}

\newcommand{\distrX}{\mathbb{P}_X}
\newcommand{\stdlabelnoise}{\sigma}
\newcommand{\gtweightnorm}{r^*}
\newcommand{\gtweightvec}{v^*}
\newcommand{\estweightnorm}{\widehat{r}}
\newcommand{\estweightvec}{\widehat{v}}
\newcommand{\estweight}{\widehat{w}}

\newcommand{\inprod}[2]{\ensuremath{\langle #1 , \, #2 \rangle}}
\newcommand{\numlabel}{m}
\newcommand{\numcomp}{n}
\newcommand{\ball}[2]{B(#1,#2)}
\newcommand{\errorcomppart}{\varepsilon}

\newcommand{\dimension}{d}
\newcommand{\comporacle}{\mathcal{O}_c}
\newcommand{\labeloracle}{\mathcal{O}_l}

\newcommand{\targeterror}{\gamma}
\newcommand{\baseval}{T}

\newcommand{\smsp}{\vspace{0mm}}
\newcommand{\medsp}{\vspace{0mm}}
\newcommand{\sign}{\text{sign}}
\newcommand{\Ologlog}{\widehat{O}}
\newcommand{\Olog}{\widetilde{O}}
\newcommand{\minimax}{\mathfrak{M}}

\begin{document}
	
	\title{Regression with Comparisons:
		Escaping the Curse of Dimensionality with Ordinal Information}
	
	\author{\name Yichong Xu \email yichongx@cs.cmu.edu \\
		\addr Machine Learning Department
		\AND
		\name Sivaraman Balakrishnan \email siva@stat.cmu.edu \\
		\addr Department of Statistics and Data Science \\
		\addr Machine Learning Department
		\AND
		\name Aarti Singh \email aarti@cs.cmu.edu \\
		\addr Machine Learning Department
		\AND
		\name Artur Dubrawski \email awd@cs.cmu.edu \\
		\addr Auton Lab, The Robotics Institute\\
		Carnegie Mellon University\\
		Pittsburgh, PA 15213, USA}
	
	\editor{}
	\maketitle
	\normalem
	\begin{abstract}
		In supervised learning, we typically leverage a fully labeled dataset to design methods
		for function estimation or prediction.
		In many practical situations, we are able to obtain
		alternative feedback, possibly at a low cost. A broad goal is to understand the usefulness of, and to design algorithms to exploit, this alternative feedback. 
		In this paper, we consider a semi-supervised regression setting, where we obtain additional \emph{ordinal} (or \emph{comparison}) information for the unlabeled samples. We consider
		ordinal feedback of varying qualities where we have either a perfect ordering of the samples, a noisy ordering of the samples or noisy pairwise comparisons between the samples.
		We provide a precise quantification of the usefulness of these types of ordinal feedback in both nonparametric and linear regression, showing that in many cases it is possible to accurately estimate an underlying function with a very small labeled set, effectively \emph{escaping the curse of dimensionality}.
		We also present lower bounds, that establish fundamental limits for the task and show that our algorithms are optimal in a variety of settings. Finally, we present extensive 
		experiments on new datasets that demonstrate the efficacy and practicality of our algorithms and investigate their robustness to various sources of noise and model misspecification.
	\end{abstract}
	
	
	\section{Introduction \label{sec:intro}}
	
	
	Classical regression is centered around the 
	development and analysis of methods that use labeled observations, 
	$\{(X_1, y_1),\ldots,(X_n, y_n)\}$, 
	where each $(X_i,y_i) \in \mathbb{R}^d \times \mathbb{R}$, 
	in various tasks of estimation and inference. 	
Nonparametric and high-dimensional methods are appealing in practice owing to their
flexibility, and the relatively weak a-priori structural assumptions that they impose on the unknown regression function. 
However, the price we pay is that these methods
typically require a large amount of labeled data to estimate complex target functions,
scaling exponentially with the dimension for fully nonparametric methods and scaling linearly with the dimension
for high-dimensional parametric methods -- the so-called curse of dimensionality. This has
motivated research on structural constraints -- for instance,
sparsity or manifold constraints -- as well as research on
active learning and semi-supervised learning where labeled
samples are used judiciously. 

We consider a complementary approach, motivated by applications in material science,
crowdsourcing, and healthcare, where we are able to supplement a small labeled dataset with a potentially larger dataset
of \emph{ordinal} information. Such ordinal information is obtained
either in the form of a (noisy) ranking of unlabeled points
or in the form of (noisy) pairwise comparisons between
function values at unlabeled points.
Some illustrative applications of the methods we develop in this paper include:\\
	
	\vspace{0.05cm}
	
	\textbf{Example 1: Crowdsourcing.} 
	In crowdsourcing we rely on human labeling effort, and in
many cases humans are able to provide more accurate ordinal feedback with substantially less effort (see for instance~\cite{tsukida2011analyze,shah2015estimation,shah2016permutation}). 
	When crowdsourced workers are asked to give numerical price estimates, they typically have difficulty giving a precise answer, resulting in a high variance between worker responses. On the other hand, when presented two products or listings side by side, workers may be more accurate at comparing them. We conduct experiments on the task of price estimation in Section \ref{sec:expr_AirBnB}.\\
	
	\vspace{0.05cm}
	
	\textbf{Example 2: Material Synthesis.} In material synthesis, the broad goal is to design complex new materials and machine learning approaches are gaining popularity \citep{materialml1,materialml2}. Typically, given a setting of input parameters (temperature, pressure etc.) we are able to perform a synthesis experiment and measure the quality of resulting synthesized material. Understanding this landscape of material quality is essentially a task of high-dimensional function estimation. Synthesis experiments can be costly and material scientists when presented with pairs of input parameters are often able to cheaply and reasonably accurately provide comparative assessments of quality.\\
	
	\vspace{0.05cm}
	
	\textbf{Example 3: Patient Diagnosis.} In clinical settings, precise assessment of each individual patient's health status can be difficult, expensive, and risky 
	 but comparing the relative
status of two patients may be relatively easy and accurate. 
	
		\vspace{0.05cm}

	In each of these settings, it is important to develop methods for function estimation that combine standard supervision with (potentially) cheaper and abundant ordinal or comparative supervision.\\


	\subsection{Our Contributions\label{sec:contribution}}
	We consider both linear and nonparametric regression with both direct (cardinal) and 
	comparison (ordinal) information.
	In both cases, we consider the standard statistical learning setup, where the samples $X$ are drawn i.i.d. from a distribution $\mathbb{P}_X$ on $\mathbb{R}^d$. The labels $y$ are related to the features $X$ as,
	\begin{align*}
	\labelV=\regfunc(\featureRV)+\labelnoise,
	\end{align*} 
	where $\regfunc = \mathbb{E}[y | X]$ is the underlying regression function of interest, and $\labelnoise$ is the mean-zero label noise. Our goal is to construct an estimator $\estfunc$ of $\regfunc$ that has low risk or mean squared error (MSE),
	\begin{align*}
	R(\widehat{f}, f) = \mathbb{E}(\estfunc(X) - \regfunc(X))^2,
	\end{align*}
	where the expectation is taken over the labeled and unlabeled training samples, as well as a new test point $X$.  We also study the fundamental information-theoretic limits of estimation with classical and ordinal supervision by establishing lower (and upper) bounds on the minimax risk.
	Letting $\eta$ denote various problem dependent parameters, which we introduce more formally in the sequel, the minimax risk:
	\begin{align}
	\label{eqn:minimax}
	\minimax(m,n; \eta) = \inf_{\widehat{f}} \sup_{f \in \mathcal{F}_{\eta}} R(\widehat{f},f),
	\end{align}
	provides an information-theoretic benchmark to assess the performance of an estimator.
	
	First, focusing on nonparametric regression, we develop a novel Ranking-Regression (\RR) algorithm for nonparametric regression that can leverage ordinal information, in addition to direct labels. We make the following contributions:
	\begin{itemize}
		\item To establish the usefulness of ordinal information in nonparametric regression, in Section~\ref{sec:perfect_rank} we consider the idealized setting where we obtain a perfect ordering of the unlabeled set. We show that the risk of the R$^2$ algorithm can be bounded with high-probability as $\widetilde{O}(m^{-2/3}+n^{-2/d})$\footnote{We use the standard big-O notation throughout this paper, and use $\widetilde{O}$ when we suppress log-factors, and $\Ologlog$ when we suppress log-log factors.}, where $m$ denotes the number of labeled samples and $n$ the number of ranked samples. To achieve an MSE of $\varepsilon$, the number of labeled samples required by \RR is \emph{independent} of the dimensionality of the input features.
		This result establishes that sufficient ordinal information of high quality can allow us to effectively circumvent
		the curse of dimensionality.
		\item In Sections~\ref{sec:noisy_rank} and~\ref{sec:with_comp} we analyze the \RR algorithm when using either a noisy ranking of the samples or noisy pairwise comparisons between them. For noisy ranking, we show that the MSE is bounded by $\widetilde{O}(m^{-2/3}+\sqrt{\nu}+n^{-2/d})$, where $\nu$ is the Kendall-Tau distance between the true and noisy ranking. As a corollary, we combine \RR with algorithms for ranking from pairwise comparisons \citep{braverman2009sorting} to obtain an MSE of $\widetilde{O}(m^{-2/3}+n^{-2/d})$ when $d\geq 4$, when the comparison noise is bounded.
	\end{itemize}
	Turning our attention to the setting of linear regression with ordinal information, we develop an algorithm that uses
	active (or adaptive) comparison queries in order to reduce both the label complexity and total query complexity. 
	\begin{itemize}
		\item In Section \ref{sec:lreg} we develop and analyze an interactive learning algorithm that estimates a linear predictor
		using both labels and comparison queries. Given a budget of $\numlabel$ label queries and $\numcomp$ comparison queries, we show that MSE of our algorithm decays at the rate of $\widetilde{O}\left(1/\numlabel+\exp(-\numcomp/\dimension)\right)$. Once again we see that when sufficiently many comparisons are available, the label complexity of our algorithm is \emph{independent} of the dimension $\dimension$. 
		\end{itemize}
		To complement these results we also 
		give information-theoretic lower bounds to characterize the fundamental limits of combining ordinal and standard supervision.	
		\begin{itemize}
		\item 
 For nonparametric regression, we show that the \RR algorithm is optimal up to log factors 
		both when it has access to a perfect ranking, as well as when the comparisons have bounded noise. For linear regression, we show that the rate of $O(1/\numlabel)$, and the total number of queries, are not improvable up to log factors.
	\end{itemize}
	On the empirical side we comprehensively evaluate the algorithms we propose, on simulated and real datasets.
	\begin{itemize}
	\item We use simulated data, and study the performance of our algorithms as we vary the noise in the labels and in the ranking. 
	\item Second, we consider a practical application of predicting people's ages from photographs. For this dataset we obtain comparisons using people's apparent ages (as opposed to their true biological age).
	\item Finally, we curate a new dataset using crowdsourced data obtained through Amazon's Mechanical Turk. We provide workers AirBnB listings and attempt to estimate property asking prices. We obtain both direct and comparison queries for the listings, and also study the time taken by workers to  provide these different types of feedback. We find that, our algorithms which combine 
direct and comparison queries are able to achieve significantly better accuracy than standard supervised regression methods, for a fixed time budget.
	\end{itemize}

	\subsection{Related Work}
	
	As a classical way to reduce the labeling effort, active learning has been mostly focused on classification \citep{hanneke2009theoretical}. For regression, it is known that, in many natural settings, the 
	ability to make active queries does not 
	lead to improved rates over passive baselines. For example, \cite{chaudhuri2015convergence} shows that when the underlying model is linear, the ability to make active queries can only improve the rate of convergence by a constant factor, and leads to no improvement when the feature distribution is spherical Gaussian. In \cite{willett2006faster}, the authors show a similar result that in nonparametric settings, active queries do not lead to a faster rate for smooth regression functions except when the regression function is piecewise smooth.
	
	There is considerable work in supervised and unsupervised learning on incorporating additional types of feedback beyond labels. For instance, \cite{zou2015,poulis2017} study the benefits of different types of ``feature feedback'' in clustering and supervised learning respectively. \cite{dasgupta2017learning} consider learning with partial corrections, where the user provides corrective feedback to the algorithm when the algorithm makes an incorrect prediction. \cite{ramaswamy2018consistent} consider multiclass classification where the user can choose to abstain from making predictions at a cost.
		
	There is also a vast literature on models and methods for analyzing pairwise comparison data \citep{tsukida2011analyze,agarwal2017learning}, like the classical Bradley-Terry \citep{bradley1952rank} and Thurstone \citep{thurstone1927law} models. In these literature, the typical focus is on ranking or quality estimation for a fixed set of objects. In contrast, we focus on function estimation and the resulting models and methods are quite different. We build on the work on ``noisy sorting'' \citep{braverman2009sorting,shah2016stochastically} to extract a consensus ranking from noisy pairwise comparisons.
	
	Perhaps the closest in spirit to this work are the two recent papers \citep{kane2017active,xu2017noise} that consider binary classification with ordinal information. These works differ from the proposed approach in their focus on classification, emphasis on active querying strategies and binary-search-based methods.
	
	Given ordinal information of sufficient fidelity, the problem of nonparametric regression is related to the problem of regression with shape constraints, or more specifically isotonic regression \citep{barlow1972statistical,zhang2002risk}. Accordingly, we leverage such algorithms in our work and we comment further on the connections in Section~\ref{sec:perfect_rank}. Some salient differences between this literature and
our work are that we design methods that work in a semisupervised setting, and further that our target is an unknown
$d$-dimensional (smooth) regression function as opposed to a
univariate shape-constrained function.

	
	
	The rest of this paper is organized as follows. In Section~\ref{sec:npreg}, we consider the problem of combining direct and comparison-based labels for nonparametric regression, providing upper and lower bounds for both noiseless and noisy ordinal models. In Section~\ref{sec:lreg}, we consider the problem of combining adaptively chosen direct and comparison-based labels for linear regression. In Section~\ref{sec:experiments}, we turn our attention to an empirical evaluation of our proposed methods on real and synthetic data.  
	Finally, we conclude in Section~\ref{sec:discussion} with a number of additional results and open directions.
In the Appendix we present detailed proofs for various technical results and a few additional supplementary experimental results.

	\section{Nonparametric Regression with Ordinal Information\label{sec:npreg}}
	We now provide analysis for nonparametric regression. First, in Section~\ref{sec:background_np} we establish the problem setup and notations. Then, we introduce the \RR algorithm in Section~\ref{sec:perfect_rank} and analyze it under perfect rankings. Next, we analyze its performance for noisy rankings and comparisons in Sections~\ref{sec:noisy_rank} and \ref{sec:with_comp}.

	\subsection{Background and Problem Setup\label{sec:background_np}}
	We consider a nonparametric regression model with random design, i.e.\ we suppose first that we are given access to an unlabeled set $\mathcal{U} = \{X_1,\ldots,X_n\}$, where $X_i \in \mathcal{X} \subset [0,1]^d$, and $X_i$ are drawn i.i.d. from a distribution $\mathbb{P}_{X}$ and we assume that $\mathbb{P}_{X}$ has a density $p$ which is upper and lower bounded as $0 < p_{\min} \leq p(x) \leq p_{\max}$ for $x \in \mathcal{X}$. Our goal is to estimate a function $f: \mathcal{X} \mapsto \mathbb{R}$, where
	following classical work \citep{gyorfi2006distribution,tsybakov2009introduction} we assume that $f$
	is bounded in $[-M,M]$ and belongs to a H\"{o}lder ball $\mathcal{F}_{s,L}$, with $0 < s \leq 1$:
	\begin{align*}
	\mathcal{F}_{s,L} = \left\{f : |f(x) - f(y)| \leq L \|x - y\|_2^s, \forall~x,y \in \mathcal{X} \right\}.
	\end{align*}
	For $s = 1$ this is the class of Lipschitz functions. We discuss the estimation of smoother functions (i.e.\ the case when $s > 1$) in~Section \ref{sec:discussion}. We obtain two forms of supervision:
	\begin{enumerate}
		\item {\bf Classical supervision: } For a (uniformly) randomly chosen subset $\mathcal{L} \subseteq \mathcal{U}$ of size $m$ (we assume throughout that $m \leq n$ and focus on settings where $m \ll n$) we make noisy observations of the form:
		\begin{align*}
		y_i = f(X_i) + \epsilon_i,~~i \in \mathcal{L},
		\end{align*}
		$\text{where }\epsilon_i\text{ are i.i.d.}$ Gaussian with $\mathbb{E}[\epsilon_i] = 0, \text{Var}[\epsilon_i] = 1$. We denote the indices of the labeled samples as $\{t_1,\ldots, t_m\} \subset \{1,\ldots,n\}$.\item {\bf Ordinal supervision: } For the given dataset 
		$\{X_1,\ldots,X_n\}$ 
		we let $\pi$ denote the \emph{true ordering}, i.e. $\pi$ is a permutation of $\{1,\ldots,n\}$ such that for $i,j \in \{1,\ldots,n\}$, with $\pi(i) \leq \pi(j)$ we have that $f(X_{i}) \leq f(X_{j})$. We assume access to one of the following types of ordinal supervision:
		
		(1) We assume that we are given access to a noisy ranking $\widehat{\pi}$, i.e. for a parameter $\nu \in [0,1]$ we assume that the Kendall-Tau distance between $\widehat{\pi}$ and the true-ordering is upper-bounded as:
		\begin{align}
		\label{eqn:kt}
		\sum_{i,j\in [n]} \mathbb{I}[(\pi(i)-\pi(j))(\widehat{\pi}(i)-\widehat{\pi}(j)) < 0]\leq \nu n^2. 
		\end{align}
		
		(2) For each pair of samples $(X_i,X_j)$, with $i < j$ we obtain a comparison $Z_{ij}$ where for some constant $\lambda > 0$:
		\begin{align}
		\label{eqn:nc}
		\mathbb{P}(Z_{ij} = \mathbb{I}(f(X_i) > f(X_j))) \geq \frac{1}{2} + \lambda.
		\end{align}
		As we discuss in Section~\ref{sec:with_comp}, it is straightforward to extend our results to a setting where only a randomly chosen subset of all pairwise comparisons are observed.

	\end{enumerate}
	Although classic supervised learning learns a regression function with labels only and without ordinal supervision, we note that learning cannot happen in the opposite way: That is, we cannot consistently estimate the underlying function with only ordinal supervision and without labels: the underlying function is only identifiable up to certain monotonic transformations. 
	
	As discussed in Section \ref{sec:contribution}, our goal is to design an estimator $\estfunc$ of $\regfunc$ that achieves the minimax mean squared error (\ref{eqn:minimax}), when $\regfunc\in \mathcal{F}_{s,L}$. We conclude this section recalling a well-known fact: given access to only classical supervision the minimax risk $\mathfrak{M}(m; \eta) = \Theta(m^{-\frac{2s}{2s+d}}),$ suffers from an exponential curse of dimensionality.
	
	\subsection{\label{sec:perfect_rank} Nonparametric Regression with Perfect Ranking}
	To ascertain the value of ordinal information we first
	consider an idealized setting, where we are given a perfect ranking $\pi$ of the unlabeled samples in~$\mathcal{U}$. We present our Ranking-Regression (\RR) algorithm with performance guarantees in Section \ref{sec:upper_main}, and a lower bound for it in Section \ref{sec:lower_bound_main} which shows that \RR is optimal up to log factors.
	
	\subsubsection{\label{sec:upper_main}Upper bounds for the \RR Algorithm}
	
	\begin{algorithm}[htb]
		\caption{\RR: Ranking-Regression}
		\begin{algorithmic}[1]		
			\REQUIRE  Unlabeled data $\mathcal{U}=\{X_1,\ldots,X_n\}$, a labeled set 
			of size $m$ and corresponding labels, i.e. samples $\{(X_{t_1},y_{t_1}),\ldots,(X_{t_m},y_{t_m})\},$ and a ranking $\widehat{\pi}$.
			\STATE Order elements in $\mathcal{U}$ as $(X_{\widehat{\pi}(1)},\ldots,X_{\widehat{\pi}(n)})$.
			\STATE Run isotonic regression (see~\eqref{eqn:isotonic}) on
			$\{y_{\labsam{1}},\ldots,y_{\labsam{m}}\}$. Denote the estimated values by $\{\widehat{y}_{\labsam{1}},\ldots,\widehat{y}_{\labsam{m}}\}$.
			\STATE For $i=1,2,\ldots,n$, let $\widetilde{i}=t_k$, where $\widehat{\pi}(\labsam{k})$ is the largest value 
			such that $\widehat{\pi}(\labsam{k})\leq \widehat{\pi}(i), k=0,1,\ldots,m$, and $\widetilde{i} = \star$ if no such $\labsam{k}$ exists.
			Set \begin{align*}
			\widehat{y}_i = \begin{cases} \widehat{y}_{\widetilde{i}}~~&\text{if}~~\widetilde{i} \neq \star \\
			0~~&\text{otherwise}.\end{cases}
			\end{align*}
			\ENSURE Function $\widehat{f} =~$NearestNeighbor$(\{(X_i,\widehat{y}_i)\}_{i=1}^n)$.
		\end{algorithmic}
		\label{algo:compreg_iso}
	\end{algorithm}	
	
	Our nonparametric regression estimator is described in Algorithm \ref{algo:compreg_iso} and Figure \ref{fig:RR_des}. We first rank all the samples in $\mathcal{U}$ according to the (given or estimated) permutation $\widehat{\pi}$. We then run isotonic regression \citep{barlow1972statistical} on the labeled samples in $\mathcal{L}$ to de-noise them and borrow statistical strength. In more detail, we solve the following program to de-noise the labeled samples:
	\begin{align}
	\label{eqn:isotonic}
	\begin{split}
	&\min_{\{\widehat{y}_{\widehat{\pi}(t_1)},\ldots,\widehat{y}_{\widehat{\pi}(t_m)}\}} \sum_{k=1}^m (\widehat{y}_{\widehat{\pi}(t_k)}-y_{\widehat{\pi}(t_k)})^2\\
	&~~~~~~\text{s.t. }~~~\widehat{y}_{t_k} \leq \widehat{y}_{t_l}~~\forall~(k,l)~~\text{such that}~~\widehat{\pi}(t_k) < \widehat{\pi}(t_l)  \\
	&~~~~~~~~~~~~~~-M \leq \{\hat{y}_{\widehat{\pi}(t_1)},\ldots, \hat{y}_{\widehat{\pi}(t_m)} \} \leq M.
	\end{split}
	\end{align}
	We introduce the bounds $\{M,-M\}$ in the above program to ease our analysis. In our experiments, we simply set $M$ to be a large positive value so that it has no influence on our estimator.
	We then leverage the ordinal information in $\widehat{\pi}$ to impute regression estimates for the unlabeled samples in $\mathcal{U}$, by assigning each unlabeled sample the value of the nearest (de-noised) labeled sample which has a smaller function value according to $\widehat{\pi}$. Finally, for a new test point, we use the imputed (or estimated) function value of the nearest neighbor in $\mathcal{U}$.
	\begin{figure}
		\centering
		\includegraphics[width=0.65\textwidth]{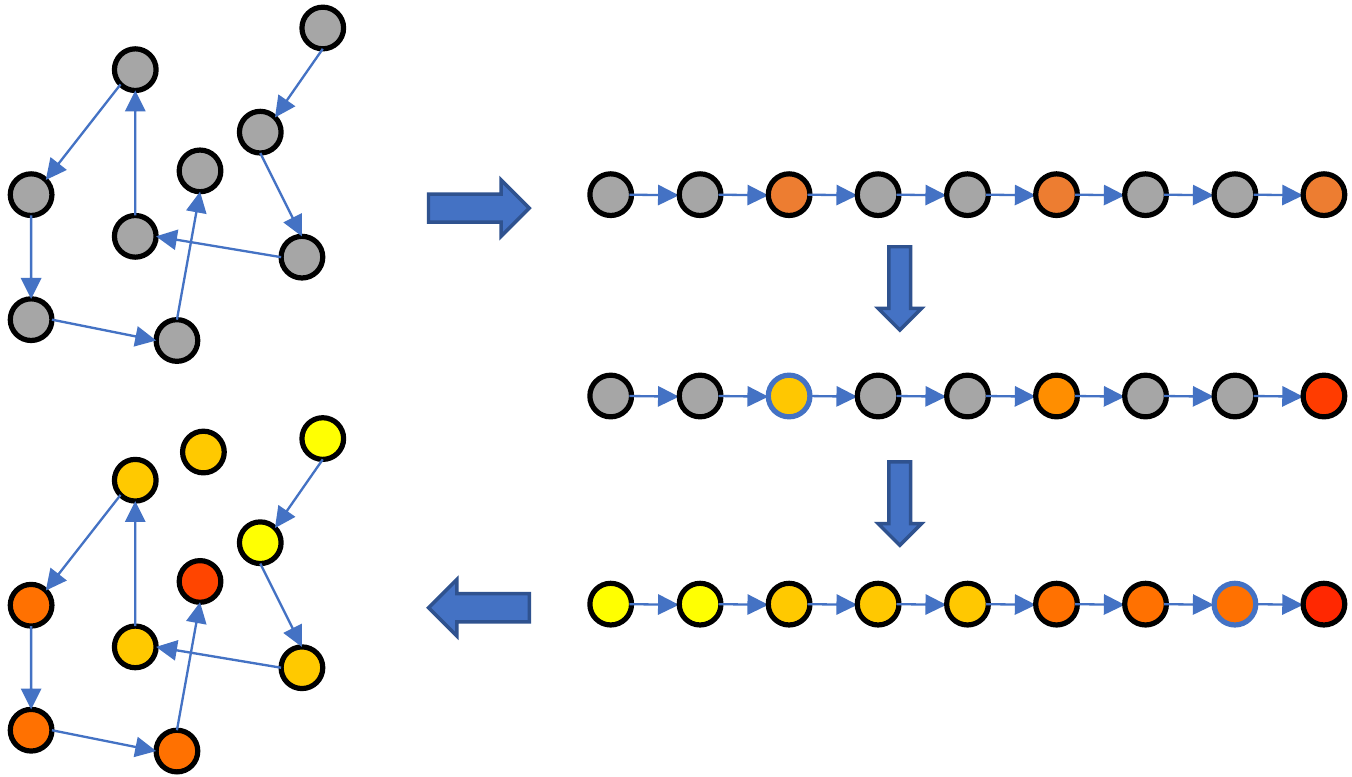}
		\caption{\label{fig:RR_des} Top Left: A group of unlabeled points are ranked according to function values using ordinal information only. Top~Right: We obtain function values of $m$ randomly chosen samples. Middle~Right: The values are adjusted using isotonic regression. Bottom~Right: Function values of other unlabeled points are inferred. Bottom~Left: For a new point, the estimated value is given by the nearest neighbor in $\mathcal{U}$.}
	\end{figure}

	In the setting where we use a perfect ranking the following theorem characterizes the performance of \RR:
	\begin{theorem}
		\label{thm:mainupper}
		For constants $C_1,C_2 > 0$ the MSE of $\widehat{f}$ is bounded by
		\begin{align*}
		\mathbb{E}(\widehat{f}(X)-\regfunc(X))^2\leq&~ C_1m^{-2/3}\log^2 n\log m + C_2 n^{-2s/d}.
		\end{align*}
	\end{theorem}
	Before we turn our attention to the proof of this result, we examine some consequences. \\
	{\bf Remarks: } (1) Theorem \ref{thm:mainupper} shows a surprising dependency on the sizes of the labeled and unlabeled sets ($m$ and $n$). The MSE of nonparametric regression using only the labeled samples is $\Theta(m^{-\frac{2s}{2s+d}})$ which is exponential in $d$ and makes nonparametric regression impractical in high-dimensions. Focusing on the dependence on $m$, Theorem \ref{thm:mainupper} improves the rate to  $m^{-2/3}\text{polylog}(m, n)$, which is no longer exponential. By using enough ordinal information we can  avoid the curse of dimensionality. \\
	(2) On the other hand, the dependence on $n$ (which dictates the amount of ordinal information needed) is still exponential. This illustrates that ordinal information is most beneficial when it is copious. We show in Section \ref{sec:lower_bound_main} that this is unimprovable in an information-theoretic sense. \\
	(3) Somewhat surprisingly, we also observe that the dependence on $n$ is faster than the $n^{-\frac{2s}{2s+d}}$ rate that would be obtained if all the samples were labeled. Intuitively, this is because of the noise in labels: Given the $m$ labels along with the (perfect) ranking, the difference between two neighboring labels is typically very small (around $1/m$).
		Therefore, any unlabeled points in $\mathcal{U}\setminus \mathcal{L}$ will be restricted to an interval much narrower  
		than the constant noise in direct labels.
 In the case where all points are labeled (i.e., $m=n$), the MSE is of order $n^{-2/3}+n^{-2s/d}$, also slightly better rate than when no ordinal information is available. On the other hand, the improvement is stronger when $m \ll n$.\\
	(4) Finally, we also note in passing that the above theorem provides an upper bound on the minimax risk in~\eqref{eqn:minimax}.
	
	\vspace{0.2cm}
	
	\begin{proofsketch}
	
		We provide a brief outline and defer technical details to the Supplementary Material.
		For a randomly drawn point $X\in \mathcal{X}$, we denote by $X_\alpha$ the nearest neighbor of $X$ in $\mathcal{U}$.
		We decompose the MSE as
		\begin{align}
		\mathbb{E}\left[(\widehat{f}(X)-\regfunc(X) )^2\right]
		\leq 2 \mathbb{E}\left[(\widehat{f}(X)-\regfunc(X_\alpha) )^2\right]+2\mathbb{E}\left[(\regfunc(X_\alpha)-\regfunc(X))^2 \right].\label{eqn:decomp_alpha}
		\end{align}
		The second term corresponds roughly to the finite-sample bias induced by the discrepancy between the function value at $X$ and the closest labeled sample. We use standard sample-spacing arguments (see \cite{gyorfi2006distribution}) to bound this term. This term contributes the $n^{-2s/d}$ rate to the final result.
		For the first term, we show a technical result in the Appendix (Lemma~\ref{lem:siva_tech}). Without loss of  generality suppose $f\big(X_{t_{1}}\big)\leq \cdots f\big(X_{t_{m}}\big)$. 
		By conditioning on a probable configuration of the points and enumerating over choices of the nearest neighbor we find that roughly (see Lemma~\ref{lem:siva_tech} for a precise statement):
		\begin{align}
		\label{eqn:tenam}
		&\mathbb{E}\left[(\widehat{f}(X)-\regfunc(X_\alpha) )^2\right] \leq \Big(\frac{\log^2 n \log m}{m}\Big) \times \nonumber \\
		&\mathbb{E}\Big(\sum_{k=1}^m \Big(\big(\widehat{f}(X_{t_k})-f(X_{t_k}) \big)^2+\big(f\big(X_{t_{k+1}}\big)-f\big(X_{t_k}\big)\big)^2 \Big)\Big).
		\end{align}
		Intuitively, these terms are related to the estimation error arising in isotonic regression (first term) and a term that captures the variance of the function values (second term). When the function $f$ is bounded, we show that the dominant term is the isotonic estimation error which is on the order of $m^{1/3}$. Putting these pieces together we obtain the theorem.
	\end{proofsketch}
	
	\subsubsection{\label{sec:lower_bound_main} Lower bounds with Ordinal Data}
	To understand the fundamental limits on the usefulness of ordinal information, as well as to study the optimality of the \RR algorithm we now turn our attention to establishing lower bounds on the minimax risk.
	In our lower bounds we choose $\mathbb{P}_{\mathcal{X}}$ to be uniform on $[0,1]^d$.
	Our estimators $\widehat{f}$ are functions of the labeled samples: $\{(X_{t_1},y_{t_1}),\ldots,(X_{t_m},y_{t_m})\},$ the set $\mathcal{U} = \{X_1,\ldots,X_n\}$ and the true ranking $\pi$.
	We have the following result:
	\begin{theorem}
		\label{thm:lowerboundMSE}
		For any estimator $\widehat{f}$ we have that for a universal constant $C > 0$,
		\begin{align*} \inf_{\widehat{f}} \sup_{f \in \mathcal{F}_{s,L}}\mathbb{E}\left[(f(X)-\widehat{f}(X))^2\right]\geq C(m^{-2/3}+n^{-2s/d}). \end{align*}
	\end{theorem}
	Comparing with the result in Theorem~\ref{thm:mainupper} we conclude that the \RR algorithm is optimal up to log factors, when the ranking is noiseless.
	\begin{proofarg}{Proof Sketch}
		We establish each term in the lower bound separately. Intuitively, for the $n^{-2s/d}$ lower bound we consider the case when all the $n$ points are labeled perfectly (in which case the ranking is redundant) and show that even in this setting the MSE of any estimator is at least $n^{-2s/d}$ due to the finite resolution of the sample.	
		
		To prove the $m^{-2/3}$ lower bound we construct a novel packing set of functions in the class $\mathcal{F}_{s,L}$, and use information-theoretic techniques (Fano's inequality) to establish the lower bound.
		The functions we construct are all increasing functions, and as a result the ranking $\pi$ provides no additional information for these functions, easing the analysis. Figure~\ref{fig:lb_proof} contrasts the classical construction for lower bounds in nonparametric regression (where tiny bumps are introduced to a reference function) with our construction where we additionally ensure the perturbed functions are all increasing. To complete the proof, we provide bounds on the cardinality of the packing set we create, as well as bounds on the Kullback-Leibler divergence between the induced distributions on the labeled samples. We provide the technical details in Section~\ref{sec:proof_lb_np_perfect}.
		
	\end{proofarg}
	\begin{figure}
		\centering
		\begin{subfigure}[b]{0.5\textwidth}
			\centering
			\includegraphics[width=\textwidth]{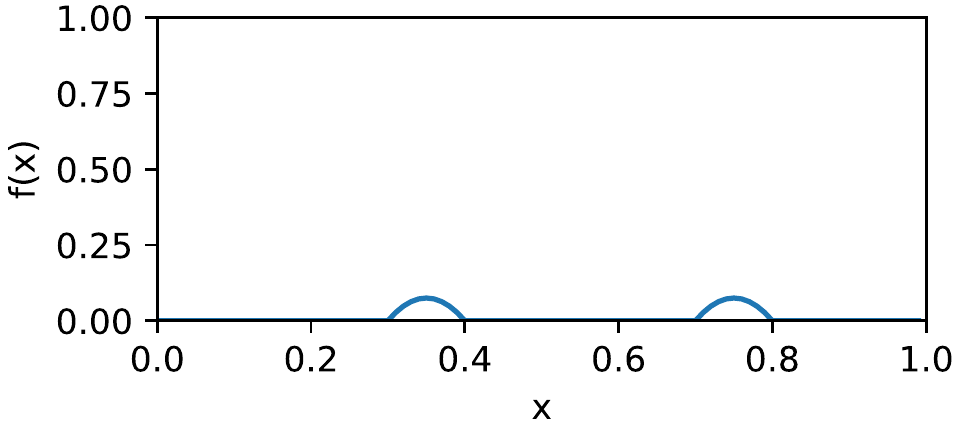}
		\end{subfigure}%
		\hfill
		\begin{subfigure}[b]{0.5\textwidth}
			\centering
			\includegraphics[width=\textwidth]{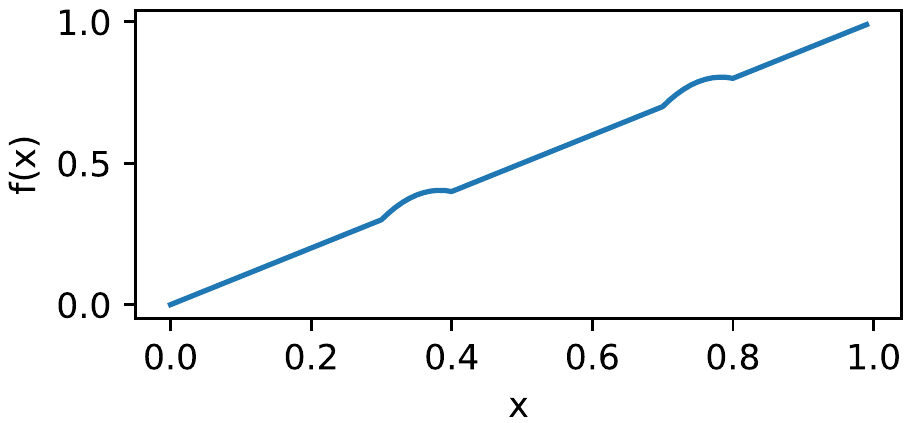}
		\end{subfigure}
		
		\caption[Graphical description]{Original construction for nonparametric regression in 1-d (left), and our construction (right).\label{fig:lb_proof}}

	\end{figure}
	
	\subsection{\label{sec:noisy_rank} Nonparametric Regression using Noisy Ranking}
	In this section, we study the setting where the ordinal information is noisy. We focus here on the 
	setting where as in equation~\eqref{eqn:kt} we obtain a ranking $\widehat{\pi}$ whose Kendall-Tau
	distance from the true ranking $\pi$ is at most $\nu n^2$. We show that the \RR algorithm is quite robust
	to ranking errors and achieves an MSE of $\widetilde{O}(m^{-2/3}+\sqrt{\nu}+n^{-2s/d})$. We establish a complementary lower bound of $\widetilde{O}(m^{-2/3}+\nu^2+n^{-2s/d})$ in Section~\ref{sec:lower_agn}.
	
	\subsubsection{\label{sec:noisy_perm} Upper Bounds for the \RR Algorithm}
	
	We characterize the robustness of \RR to ranking errors, i.e. when $\widehat{\pi}$ satisfies the condition in~\eqref{eqn:kt}, in the following theorem:	
	\begin{theorem}
		\label{thm:agn_upper}
		
		For constants $C_1, C_2 > 0$, the MSE of the \RR estimate $\widehat{f}$ is bounded by
		\begin{align*}
		\mathbb{E}[(\widehat{f}(X)-\regfunc(X))^2]\leq C_1\left(\log^2 n\log m \left(m^{-2/3}+\sqrt{\nu}\right)\right) + C_2 n^{-2s/d}.
		\end{align*}
	\end{theorem}
	{\bf Remarks: } (1) Once again we observe that in the regime where sufficient ordinal information is available, i.e. when $n$ is large, the rate no longer has an exponential dependence on the dimension $d$.\\
	(2) This result also shows that the \RR algorithm is inherently robust to noise in the ranking, and the mean squared error degrades gracefully as a function of the noise parameter $\nu$. We investigate the optimality of the $\sqrt{\nu}$-dependence in the next section. 
	
	We now turn our attention to the proof of this result.
	\begin{proofarg}{Proof Sketch}
		When using an estimated permutation $\widehat{\pi}$ the true function of interest $f$ is no longer 
		an increasing (isotonic) function with respect to $\widehat{\pi}$, and this results in a model-misspecification 
		\emph{bias}. 
		The core technical novelty of our proof is in relating the upper bound on
		the error in $\widehat{\pi}$ to an upper bound on this bias. Concretely, in the Appendix we show the following lemma:
		\begin{lemma}
			\label{lemma:rankerr_to_valueerr}
			For any permutation $\widehat{\pi}$ satisfying the condition in~\eqref{eqn:kt}
			\begin{align*}
			\sum_{i=1}^n (f(X_{\pi^{-1}(i)})-f(X_{\widehat{\pi}^{-1}(i)}))^2\leq 8M^2\sqrt{2\nu}n.
			\end{align*}
		\end{lemma}
		Using this result we bound the minimal error of approximating
		an increasing sequence according to $\pi$ by an increasing sequence according to the estimated (and misspecified) ranking $\widehat{\pi}$. We denote this error on $m$ labeled points by $\Delta$, and using Lemma~\ref{lemma:rankerr_to_valueerr} we show that in expectation (over the random choice of the labeled set) 
		\begin{align*}
		\mathbb{E}[\Delta] \leq  8 M^2\sqrt{2\nu}m. 
		\end{align*}
		With this technical result in place
		we follow the same decomposition and subsequent steps before we arrive at the expression in equation~\eqref{eqn:tenam}. In this case, the first term for some constant $C > 0$ is bounded as:
		\begin{align*}
		\mathbb{E} \Big(\sum_{k=1}^m \big(\widehat{f}(X_{t_k})-f(X_{t_k}) \big)^2\Big) \leq 2\mathbb{E}[\Delta] + C m^{1/3},
		\end{align*}
		where the first term corresponds to the model-misspecification bias and the second corresponds to the usual isotonic regression rate. Putting these terms together in the decomposition in~\eqref{eqn:tenam} we obtain the theorem.
		
	\end{proofarg}
	
	In settings where $\nu$ is large \RR can be led astray by the ordinal information, and a standard nonparametric regressor can achieve the (possibly faster) $O\big(m^{-\frac{2s}{2s+d}}\big)$ rate by ignoring the ordinal information. 
	In this case, a simple and standard cross-validation procedure can combine the benefits of both methods: we estimate the regression function twice, once using \RR and once using k nearest neighbors, and choose the regression function that performs better on a held-out validation set. The following theorem shows guarantees for this method and an upper bound for the minimax risk (\ref{eqn:minimax}):
	\begin{theorem}
		\label{thm:upper_cv}
		Under the same assumptions as Theorem \ref{thm:agn_upper}, there exists an estimator $\widehat{f}$ such that
		\begin{align*}
		\mathbb{E}[(\widehat{f}(X)-\regfunc(X))^2]= \widetilde{O}\left(m^{-2/3}+\min\{\sqrt{\nu},m^{-\frac{2s}{2s+d}} \} + n^{-2s/d}\right).
		\end{align*}		
	\end{theorem}
	The main technical difficulty in analyzing the model-selection procedure we propose in this context is that a na\"{i}ve analysis of the procedure, using classical tail bounds to control the deviation between the empirical risk and the population risk, results in an excess risk of $\widetilde{O}(1/\sqrt{m}).$ However, this rate would overwhelm the $\widetilde{O}(m^{-2/3})$ bound that arises from  isotonic regression.
	We instead follow a more careful analysis 
	outlined in the work of~\cite{barron1991} which exploits properties of the squared-loss to obtain an excess risk bound of $\widetilde{O}(1/m).$ We provide a detailed proof in the Appendix for convenience and note that in our setting, the $y$ values are not assumed to be bounded and this necessitates some minor modifications to the original proof \citep{barron1991}.

	\subsubsection{\label{sec:lower_agn}Lower bounds with Noisy Ordinal Data}
	In this section we turn our attention to lower bounds in the setting with noisy ordinal information. In particular, we construct a permutation  
	$\widehat{\pi}$ such that for a pair $(X_i,X_j)$ of points randomly chosen from $\mathbb{P}_{\mathcal{X}}$:
	\begin{align*}
	\mathbb{P}[(\pi(i)-\pi(j))(\widehat{\pi}(i)-\widehat{\pi}(j))<0]\leq \nu.
	\end{align*}
	We analyze the minimax risk of an estimator which has access to this noisy permutation $\widehat{\pi}$, in addition to the labeled and unlabeled sets (as in~Section~\ref{sec:lower_bound_main}).
	\begin{theorem}
		\label{thm:lowerboundMSE_agn}
		There is a constant $C > 0$ such that for any estimator $\widehat{f}$ taking input $X_1,\ldots,X_n, y_1,\ldots,y_m$ and $\widehat{\pi}$,
		\begin{align*}
		\inf_{\widehat{f}} \sup_{f \in \mathcal{F}_{s,L}} \mathbb{E}\big(f(X)-\widehat{f}(X)\big)^2\geq C(m^{-\frac{2}{3}}+\min\{\nu^2,m^{-\frac{2}{d+2}}\}+n^{-2s/d}).
		\end{align*}
	\end{theorem}
	Comparing this result with our result in Remark 3 following Theorem~\ref{thm:agn_upper}, our upper and lower bounds differ by the gap between $\sqrt{\nu}$ and $\nu^2$, in the case of Lipschitz functions ($s=1$).
	
	\begin{proofarg}{Proof Sketch} We focus on the dependence on $\nu$,
		as the other parts are identical to Theorem \ref{thm:lowerboundMSE}. 
		We construct a packing set of Lipschitz functions, and we subsequently construct a noisy comparison oracle $\widehat{\pi}$ which provides no additional information beyond the labeled samples.
		The construction of our packing set is inspired by the construction of standard lower bounds in nonparametric regression (see Figure~\ref{fig:lb_proof}), but we modify this construction to ensure that $\widehat{\pi}$ is uninformative. In the classical construction we divide $[0,1]^d$ into $u^d$ grid points, with $u=m^{1/(d+2)}$ and add a ``bump'' at a carefully chosen subset of the grid points. Here we instead divide $[0,t]^d$ into a grid with $u^d$ points, and add an increasing function along the first dimension, where $t$ is a parameter we choose in the sequel.
		
		We now describe the ranking oracle which generates the permutation $\widehat{\pi}$: we simply rank 
		sample points according to their first coordinate. 
		This comparison oracle only makes an error when both $x,x'$ lies in $[0,t]^d$, and both $x_1,x_1'$ lie in the same grid segment $[tk/u,t(k+1)/u]$ for some $k\in [u]$. So the Kendall-Tau error of the comparison oracle is $(t^d)^2 \times ((1/u)^2 \times u)=u t^{2d}$. We choose $t$ such that this value is less than $\nu$.
		Once again we complete the proof by lower bounding the cardinality of the packing-set for our stated choice of $t$, upper bounding the Kullback-Leibler divergence between the induced distributions and appealing to Fano's inequality.
	\end{proofarg}

	\subsection{\label{sec:with_comp}Regression with Noisy Pairwise Comparisons}
	In this section we focus on the setting where the ordinal information is obtained in the form of noisy pairwise comparisons, following equation~\eqref{eqn:nc}. We investigate a natural strategy of aggregating the pairwise comparisons to form a consensus ranking $\widehat{\pi}$ and then applying the \RR algorithm with this estimated ranking. We build on results from theoretical computer science, where such aggregation 
	algorithms are studied for their connections to sorting with noisy comparators. In particular, \citet{braverman2009sorting} study noisy sorting algorithms under the noise model described in~\eqref{eqn:nc} and establish the following result:
	\begin{theorem}[\cite{braverman2009sorting}]
		\label{thm:bound_noise_ranking}
		Let $\alpha>0$. There exists a polynomial-time algorithm using noisy pairwise comparisons between $n$ samples, that with probability $1-n^{-\alpha}$, returns a ranking $\widehat{\pi}$ such that for a constant $c(\alpha,\lambda) > 0$ we have that:
		\begin{align*}
		\sum_{i,j\in [n]} \mathbb{I}[(\pi(i)-\pi(j))(\widehat{\pi}(i)-\widehat{\pi}(j)) < 0]\leq c(\alpha,\lambda) n. 
		\end{align*}		
		Furthermore, if allowed a sequential (active) choice of comparisons, the algorithm queries at most $O(n\log n)$ pairs of samples.
	\end{theorem}
	Combining this result with our result on the robustness of \RR 
	we obtain an algorithm for nonparametric regression with access to noisy 
	pairwise comparisons with the following guarantee on its performance:
	\begin{corollary}
		\label{thm:bound_noise_reg} 
		For constants $C_1,C_2 > 0$, \RR with $\widehat{\pi}$ estimated as described above produces an estimator $\widehat{f}$ with MSE
		\begin{align*}
		\mathbb{E}\big(\widehat{f}(X)-f(X)\big)^2 \leq C_1 m^{-2/3}\log^2 n\log m+C_2 \max\{n^{-2s/d}, n^{-1/2} \log^2 n\log m \}.
		\end{align*}	
	\end{corollary}
	{\bf Remarks: } \begin{enumerate}
		\item From a technical standpoint this result is an immediate corollary of Theorems~\ref{thm:agn_upper} and~\ref{thm:bound_noise_ranking}, but the extension is important from a practical standpoint. 
		The ranking error of $O(1/n)$ from the noisy sorting algorithm leads to an additional $\widetilde{O}(1/\sqrt{n})$ term in the MSE.  This error is dominated by the $n^{-2s/d}$ term if $d\geq 4s$, and in this setting the result in Corollary \ref{thm:bound_noise_reg} is also optimal up to log factors (following the lower bound in Section \ref{sec:lower_bound_main}).\\
		\item We also note that the analysis in~\cite{braverman2009sorting} extends in a straightforward way to a setting where only a randomly chosen subset of the pairwise comparisons are obtained. 
	\end{enumerate}
	
	\section{Linear Regression with Comparisons \label{sec:lreg}}
	In this section we investigate another popular setting for regression, that of fitting a linear predictor to data. 
	We show that when we have enough comparisons, it is possible to estimate a linear function even when $\numlabel \ll \dimension$, without making \emph{any} sparsity or structural assumptions. 
	
	For linear regression we follow a different approach when compared to the nonparametric setting we have studied thus far. By exploiting the assumption of linearity, we see that each comparison	
now translates to a constraint on the parameters of the linear model, 
and as a consequence we are able to use comparisons to obtain a good initial estimate. 
However, the unknown linear regression parameters are not fully identified by comparison. For instance, we observe that the two regression vectors $w$ and $2 \times w$ induce the same comparison results for any pairs $(X_1,X_2)$.
This motivates using direct measurements to estimate a global scaling of the unknown regression parameters, i.e the norm of regression vector $w^*$. 
Essentially by using comparisons instead of direct measurements to estimate the weights, we convert the regression problem to a \emph{classification} problem, and therefore can leverage existing algorithms and analyses 
from the passive/active classification literature. 


	We present our assumptions and notation for the linear setup in the next subsection. Then we give the algorithm along with its analysis in Section~\ref{sec:upperbounds_lr}. We also present information-theoretic lower bounds on the minimax rate in Section~\ref{sec:lowerbounds_lr}.
	

	\subsection{Background and Problem Setup\label{sec:background_lr}}
	
	Following some of the previous literature on estimating a linear classifier~\citep{awasthi2014power,awasthi2016learning}, we assume that the distribution $\distrX$ is isotropic and log-concave. In more detail, we assume that coordinates of $\featureRV$ are independent, centered around 0, have covariance $I_{\dimension}$; and that the log of the density function of $\featureRV$ is concave.
This assumption is satisfied by many standard distributions, for instance the uniform and Gaussian distributions~\citep{lovasz2007geometry}.
	We let $\ball{v}{r}$ denote the ball of radius $r$ around vector $v$. 
	Our goal is to estimate a linear function $f(X)=\inprod{\gtweight}{X}$, and for convenience we denote:
	\begin{align*}
	\gtweightnorm=\|\gtweight\|_2~~~\text{and}~~~\gtweightvec=\frac{\gtweight}{\|\gtweight\|_2}.
	\end{align*} Similar to the nonparametric case, we suppose that we have access to two kinds of supervision using labels and comparisons respectively. We represent these using the following two oracles:
	
	\begin{itemize}
		\item \textbf{Label Oracle:} We assume access to a label oracle $\labeloracle$, which takes a sample $X \in \mathbb{R}^\dimension$ and outputs a label $\labelRV = \inprod{\gtweight}{X}+\labelnoise$,
		with $\mathbb{E}[\labelnoise]=0$, $\text{Var}(\labelnoise)=\stdlabelnoise^2$. 
		
		\item \textbf{Comparison Oracle:} We also have access to a (potentially cheaper) comparison oracle $\comporacle$. For each query, the oracle $\comporacle$ receives a pair of samples $(\featureRV,\featureRV')\sim \distrX \times \distrX$, and returns a random variable $\compRV\in \{-1,+1\}$, where $\compRV=1$ indicates that the oracle believes that $f(\featureRV)>f(\featureRV')$, and $\compRV=-1$ otherwise. We assume that the oracle has agnostic noise\footnote{As we discuss in the sequel our model can also be adapted to the bounded noise model case of eqn. (\ref{eqn:nc}) using a different algorithm from active learning; see Section~\ref{sec:discussion} for details.} $\nu$, i.e.:
		\smsp
		\[\mathbb{P}(\compRV\ne \sign(\inprod{\gtweight}{\featureRV - \featureRV'}))\leq \nu. \]
		That is, for a randomly sampled triplet $(\featureRV,\featureRV',\compRV)$ the oracle is wrong with probability at most $\nu$. Note that the error of the oracle for a particular example $(\featureRV,\featureRV')=(\featureV,\featureV')$ can be arbitrary.
	\end{itemize}

	Given a set of unlabeled instances $\mathcal{U}=\{\featureRV_1,\featureRV_2,\ldots\}$ drawn from $\distrX$, we aim to estimate $\gtweight$ by querying $\labeloracle$ and $\comporacle$ with samples in $\mathcal{U}$, using a label and comparison budget of $\numlabel$ and $\numcomp$ respectively. Our algorithm can be either passive, active or semi-supervised; we denote the output of algorithm $\mathcal{A}$ by $\estweight=\mathcal{A}(\mathcal{U},\labeloracle, \comporacle)$. For an algorithm $\mathcal{A}$, we study the minimax risk (\ref{eqn:minimax}), which in this case can be written as

	\begin{align}
	\minimax(\numlabel,\numcomp)=\inf_{\mathcal{A}}\sup_{\gtweight} \mathbb{E}\left[\inprod{\gtweight-\estweight}{X}^2\right].\label{eqn:minimax_lr}
	\end{align}
	
	\newcommand{\clsalgo}{\mathcal{A}^c}
	\newcommand{\erralgo}{\varepsilon}
	
	\noindent Our algorithm relies on a linear classification algorithm $\clsalgo$, which we assume to be a proper classification algorithm (i.e. the output of $\clsalgo$ is a linear classifier which we denote by $\widehat{v}$).
	We let $\clsalgo(\mathcal{U}, \mathcal{O}, \numlabel)$ denote the output (linear) classifier of $\clsalgo$ when it uses the unlabeled data pool $\mathcal{U}$, the label oracle $\mathcal{O}$ and acquires $\numlabel$ labels. $\clsalgo$ can be either passive or active; in the former case the $m$ labels are randomly selected from $\mathcal{U}$, whereas in the latter case $\clsalgo$ decides which labels to query.
	We use $\erralgo_{\clsalgo}(m,\delta)$ to denote (an upper bound on) the $0/1$ error of the algorithm $\clsalgo$ when using $m$ labels, with $1-\delta$ probability, i.e.:
	\[\mathbb{P}[\err(\clsalgo(\mathcal{U},\numlabel))\leq \erralgo_{\clsalgo}(\numlabel,\delta)]\geq 1-\delta. \]
	We note that by leveraging the log-concave assumption on $P_X$, it is straightforward to translate guarantees on the 0/1 error to corresponding guarantees on the $\ell_2$ error $\|\widehat{v} - v^*\|_2$.


	We conclude this section by noting that the classical minimax rate for ordinary least squares (OLS) is of order $O(\dimension/\numlabel)$, where $\numlabel$ is the number of label queries. This rate cannot be improved by active label queries (see for instance \cite{chaudhuri2015convergence}).

	\subsection{Algorithm and Analysis \label{sec:upperbounds_lr}}
	
	Our algorithm, Comparison Linear Regression (CLR), is described in Algorithm \ref{algo:linreg}. 
	We first use the comparisons to construct a classification problem with samples $(\featureRV-\featureRV')$ and oracle $\comporacle$. Here we slightly overload the notation of $\comporacle$ that $\comporacle(X_i-X_j)=\comporacle(X_i,X_j)$. Given these samples we use an active linear classification algorithm to estimate a normalized weight vector $\estweightvec$.
	After classification, we use the estimated $\estweightvec$ along with actual label queries to estimate the norm of the weight vector $\estweightnorm$. Combining these results we obtain our final estimate $\estweight=\estweightnorm\cdot \estweightvec$.

	\begin{algorithm}[htb]
		\caption{Comparison Linear Regression (CLR)}
		\begin{algorithmic}[1]		
			\REQUIRE{comparison budget $\numcomp$, label budget $\numlabel$, unlabeled data pool $\mathcal{U}$, algorithm $\clsalgo$}
			\STATE Construct pairwise pool $\mathcal{U}'=\{X_1-X_2, X_3-X_4, \ldots, X_{n-1}-X_n \}$
			\STATE Run $\clsalgo(\mathcal{U}', \comporacle, \numcomp)$ and obtain a classifier with corresponding weight vector $\estweightvec$
			\STATE Query random samples $\{(\featureRV_i,\labelRV_i)\}_{i=1}^{\numlabel}$
			\STATE Let $\displaystyle \estweightnorm=\frac{\sum_{i=1}^\numlabel \inprod{\estweightvec}{\featureRV_i}\labelRV_i}{\sum_{i=1}^\numlabel \inprod{\estweightvec}{\featureRV_i}^2 }$.
			\ENSURE{$\displaystyle \estweight=\estweightnorm\cdot \estweightvec$.}
		\end{algorithmic}
		\label{algo:linreg}
	\end{algorithm}


	We have the following main result which relates the error of $\widehat{w}$ to the error of the classification algorithm $\clsalgo$.
	\begin{theorem}
		\label{thm:upper_lr}
		There exists some constants $C,M$ such that if $\numlabel>M$, the MSE of Algorithm \ref{algo:linreg} satisfies
		\begin{align*}
		\mathbb{E}[\inprod{\gtweight -\estweight}{\featureRV}^2]
		\leq \Ologlog\left(\frac{1}{\numlabel}+\log^2\numlabel\cdot \varepsilon_{\clsalgo}\left(n,\frac{1}{m}\right)+\nu^2\right).
		\end{align*}
	\end{theorem}

	\noindent We defer a detailed proof to the Appendix and provide a concise proof sketch here. 
	
	\begin{proofsketch}

\noindent		We recall that, by leveraging properties of isotropic log-concave distributions, we can obtain an estimator $\estweightvec$ such that $\|\estweightvec-\gtweightvec\|_2\leq \errorcomppart=	O\left(\erralgo_{\clsalgo}(\numcomp, \delta)\right)$. Now let $\baseval_i=\inprod{\estweightvec}{\featureRV_i}$, and we have
		\begin{align*}
		\estweightnorm= \frac{\sum_{i=1}^\numlabel \baseval_i \labelRV_i}{\sum_{i=1}^\numlabel \baseval_i^2 }
		=\gtweightnorm+\frac{\sum_{i=1}^\numlabel \baseval_i \gtweightnorm \inprod{\gtweightvec-\estweightvec}{\featureRV_i}+\labelnoise_i}{\sum_{i=1}^\numlabel \baseval_i^2 }.
		\end{align*}		
		And thus
		\begin{align*}
		\inprod{\gtweight}{\featureRV}- \inprod{\estweight}{\featureRV}
		=\gtweightnorm \inprod{\gtweightvec-\estweightvec}{\featureRV}-\frac{\sum_{i=1}^\numlabel \baseval_i \gtweightnorm \inprod{\gtweightvec-\estweightvec}{\featureRV_i}}{\sum_{i=1}^\numlabel \baseval_i^2 }\inprod{\estweightvec}{\featureRV}+
		\frac{\sum_{i=1}^\numlabel \baseval_i\labelnoise_i}{\sum_{i=1}^\numlabel \baseval_i^2 }\inprod{\estweightvec}{\featureRV}.
		\end{align*}
		The first term can be bounded using $\|\estweightvec-\gtweightvec\|_2\leq \errorcomppart$; for the latter two terms, using Hoeffding bounds we show that $\sum_{i=1}^\numlabel \baseval_i^2=O(\numlabel)$. Then by decomposing the sums in the latter two terms, we can bound the MSE.
	\end{proofsketch}
	
Leveraging this result we can now use existing results to derive concrete corollaries for particular instantiations of the classification algorithm $\clsalgo$. For example, when $\nu=0$ and we use passive learning, standard VC theory shows that the empirical risk minimizer has error $\erralgo_{\text{ERM}}=O(d\log(1/\delta)/n)$. This leads to the following corollary:
	\begin{corollary}
		\label{corol:upper_erm}
		Suppose that $\nu=0$, and that we use the passive ERM classifier as $\clsalgo$ in Algorithm \ref{algo:linreg}, then the output $\widehat{w}$ has MSE bounded as:
		\begin{align*}
		\mathbb{E}[\inprod{\gtweight -\estweight}{\featureRV}^2]
		\leq \Ologlog\left(\frac{1}{\numlabel}+\frac{d\log^3\numlabel}{\numcomp}\right).
		\end{align*}
	\end{corollary}
	When $\nu>0$, we can use other existing algorithms for either the passive/active case. We give a summary of existing results in Table \ref{tab:cls_results}, and note that (as above) each of these results can be translated in a straightforward way to a guarantee on the MSE when combining direct and comparison queries. 
	We also note that previous results using isotropic log-concave distribution can be directly exploited in our setting since $X-X'$ follows isotropic log-concave distribution if $X$ does \citep{lovasz2007geometry}. Each of these results provide upper bounds on minimax risk \eqref{eqn:minimax_lr} under certain restrictions.
	
	\begin{table}[htb]
		\centering
		\caption{Summary of existing results for passive/active classification for isotropic log-concave $X$ distributions. $C$ denotes some fixed constant; $\varepsilon =\erralgo_{\clsalgo}(\numlabel,\delta)$. The work of \cite{yan2017revisiting} additionally requires that $X-X'$ follow a uniform distribution.}
		\label{tab:cls_results}
		\begin{tabular}{c|c|c|c|c}
			\hline
			Algorithm & Oracle & Requirement & Rate of $\varepsilon$ & Efficient? \\
			\hline
			\multirow{2}*{ERM \citep{hanneke2009theoretical}} & \multirow{2}*{Passive} & None & $\Olog\left(d\sqrt{\frac{1}{\numlabel}}\right)$ & \multirow{2}*{No} \\
			&  & $\nu=O(\varepsilon)$ & $\Olog\left(\frac{d}{m}\right)$ & \\\hline
			\cite{awasthi2016learning} & Passive &$\nu=O(\varepsilon)$& $\Olog\left(\left(\frac{d}{\numlabel}\right)^{1/3}\right)$& Yes\\\hline
			\cite{awasthi2014power} & Active &$\nu=O(\varepsilon)$& $\exp\left(-\frac{C\numlabel}{d+\log(\numlabel/\delta)}\right) $& Yes\\\hline
			\multirow{2}*{\cite{yan2017revisiting}} & Passive &\multirow{2}*{$\nu=O\left(\frac{\varepsilon}{\log d+\log\log\frac{1}{\varepsilon}}\right)$}& $\Olog\left(\frac{d}{m}\right)$& \multirow{2}*{Yes}\\
			& Active && $\exp\left(-\frac{C\numlabel}{d+\log(\numlabel/\delta)}\right)$& \\\hline
		\end{tabular}
		
	\end{table}

	\subsection{Lower Bounds\label{sec:lowerbounds_lr}}
	Now we turn to information-theoretic lower bounds on the minimax risk (\ref{eqn:minimax}). We consider \emph{any} active estimator $\estweight$ with access to the two oracles $\comporacle,\labeloracle$, using $\numcomp$ comparisons and $\numlabel$ labels. In this section, we show that the $1/m$ rate in Theorem \ref{thm:upper_lr} is optimal; we also show a lower bound in the active setting on the total number of queries in the appendix.

	\smsp
	\begin{theorem}
		\label{thm:lowern}
		Suppose that $\featureRV$ is uniform on $[-1,1]^{\dimension}$, and $\labelnoise\sim \mathcal{N}(0,1)$. Then, for any (active) estimator $\estweight$ with access to both label and comparison oracles, there is a universal constant $c > 0$ such that,
		\[\inf_{\estweight}\sup_{\gtweight} \mathbb{E}\left[\inprod{\gtweight-\estweight}{\featureRV}^2\right]\geq \frac{c}{\numlabel}. \]
	\end{theorem}
	\medsp
	Theorem \ref{thm:lowern} shows that the $O(1/\numlabel)$ term in Theorem \ref{thm:upper_lr} is necessary. The proof uses classical information-theoretic techniques (Le Cam's method) applied to two increasing functions with $\dimension=1$, and is included in Appendix~\ref{sec:proof_lower_n_lr}.
	
	\newcommand{\existweight}{\widetilde{w}}
	\smsp
	
	We note that we cannot establish lower bounds that depend solely on number of comparisons $\numcomp$, since we can of course achieve $O(\dimension/\numlabel)$ MSE without using any comparisons. Consequently, we show a lower bound on the \emph{total} number of queries in Appendix~\ref{sec:proof_lowerall_lr}. This bound shows that when using the algorithm in the paper of \cite{awasthi2014power} as $\clsalgo$ in CLR, the total number of queries is optimal up to log factors.
	
	\section{Experiments}
	\label{sec:experiments}
	To verify our theoretical results and test our algorithms in practice, we perform 
	three sets of experiments. 
	First, we use simulated data, where the noise in the labels and ranking can be controlled separately. 
	Second, we consider an application of predicting people's ages from photographs, where we synthesize comparisons from data on people's apparent ages (as opposed to their true biological ages). Finally, we crowdsource data using Amazon's Mechanical Turk to obtain comparisons and direct measurements for the task of estimating rental property asking prices. We then evaluate various algorithms for predicting rental property prices, both in terms of their accuracy, as well as in terms of their time cost.
	
\noindent 	\textbf{Baselines.} In the nonparametric setting, we compare \RR with $k$-NN algorithms in all experiments. We choose
	$k$-NN methods because they are near-optimal theoretically for Lipschitz functions, and are widely used in practice. Also, the \RR method is a nonparametric method built on a nearest neighbor regressor. It may be possible to use the ranking-regression method in conjunction with other nonparametric regressors but we leave this for future work. 
	We choose from a range of different constant values of $k$ 
	in our experiments. 
	
	In the linear regression setting, for our CLR algorithm, we consider both a passive and an active setting for comparison queries. For the passive comparison query setting, we simply use a Support Vector Machine (SVM) as $\mathcal{A}$ in Algorithm \ref{algo:linreg}. For the active comparison query setting, we use an algorithm minimizing hinge loss as described in \cite{awasthi2014power}. We compare CLR to the 
	LASSO and to Linear Support Vector Regression (LinearSVR), where the relevant hyperparameters are chosen based on validation set performance. We choose LASSO and LinearSVR as they are the most prevalent linear regression methods.
	Unless otherwise noted, we repeat each experiment 20 times and report the average MSE\footnote{Our plots are best viewed in color.}.

\noindent 	\textbf{Cost Ratio.} Our algorithms aim at reducing the overall cost of estimating a regression function when comparisons can be more easily available 
	than direct labels. In practice, the cost of obtaining comparisons can vary greatly depending on the task, and we consider two practical setups:
	\begin{enumerate}
		\item In many applications, both direct labels and comparisons can be obtained, but labels cost more than comparisons. Our price estimation task corresponds to this case.  The cost, in this case, depends on the ratio between the cost of comparisons and labels. We suppose that comparisons cost 1, and that labels cost $c$ for some constant $c>1$ and that we have a total budget of $C$. We call $c$ the \emph{cost ratio}. 
Minimizing the risk of our algorithms requires minimizing $\minimax(m,n; \eta)$ as defined in (\ref{eqn:minimax}) subject to $cm+n\leq C$; 
for most cases, we need a small $m$ and large $n$. In experiments with a finite cost ratio, we fix the number of direct measurements to be a small constant $m$ and vary the number of comparisons that we use. 
		\item Direct labels might be substantially harder to acquire because of privacy issues or because of inherent restrictions in the data collection process, whereas comparisons are easier to obtain. 
		Our age prediction task corresponds to this case, where it is conceivable that only some of the biological ages are available due to privacy issues. In this setting, the cost is dominated by the cost of the direct labels and we measure the cost of estimation by the number of labeled samples used. 
	\end{enumerate}  


	\subsection{\label{sec:mod_algo} Modifications to Our Algorithms}

	While \RR and CLR are near optimal from a theoretical standpoint, we adopt the following techniques to improve their empirical performance:
	
\noindent	\textbf{\RR with $k$-NN.} Our analysis considers the case when we use 1-NN after isotonic regression. However, we empirically find that using more than 1 nearest neighbor can also improve the performance. So in our experiments, we use $k$-NN in the final step of \RR, where $k$ is a small fixed constant. We note in passing that 
our theoretical results remain valid in this slightly more general setting. 
	
\noindent	\textbf{\RR with comparisons.} When \RR uses passively collected comparisons, we would need $O(n^2)$ pairs to have a ranking with $O(1/n)$ error in the Kendall-Tau metric if we use the algorithm from \cite{braverman2009sorting}. We instead choose to take advantage of the feature space structure when we use \RR with comparisons. 
	Specifically, we build a nonparametric rankSVM \citep{joachims2002optimizing} 
	to score each sample using pairwise comparisons. We then rank samples according to their scores given by the rankSVM. 
	 We discuss another potential method, which uses nearest neighbors based on Borda counts, in Appendix \ref{sec:app_expr}.
	
\noindent	\textbf{CLR with feature augmentation.} Using the directly labeled data only to estimate the norm of the weights corresponds to 
using linear regression with the direct labels with a \emph{single} feature $\inprod{\estweightvec}{x}$ from Algorithm \ref{algo:linreg}. Empirically, we find that using all the features together with
the estimated $\inprod{\estweightvec}{x}$ results in better performance. Concretely, we use a linear SVR with input features $(x;\inprod{\estweightvec}{x})$, and use the output as our prediction.


	
	\subsection{Simulated Data\label{sec:expr_sim}}
We generate different synthetic datasets for nonparametric and linear regression settings in order to verify our theory. 
	\subsubsection{Simulated Data for \RR \label{sec:expr_sim_rr}}
	\begin{figure}[htb!]
		\centering
		\begin{subfigure}[b]{0.45\textwidth}
			\centering
			\includegraphics[width=\textwidth]{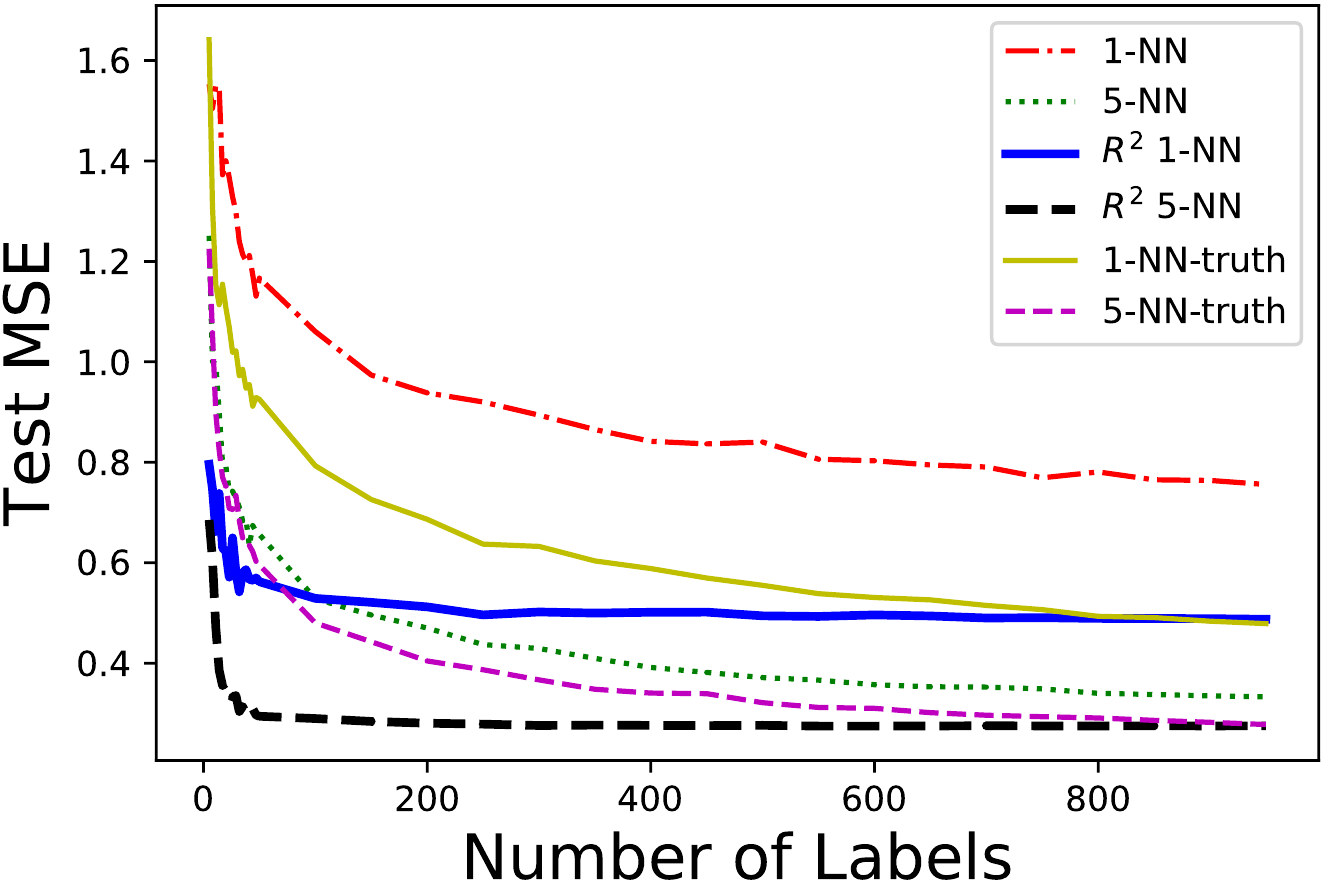}
			\caption{Perfect ranking}
			\label{fig:res_sim_perf_rank}
		\end{subfigure}%
	\;		
		\begin{subfigure}[b]{0.45\textwidth}
			\centering
			\includegraphics[width=\textwidth]{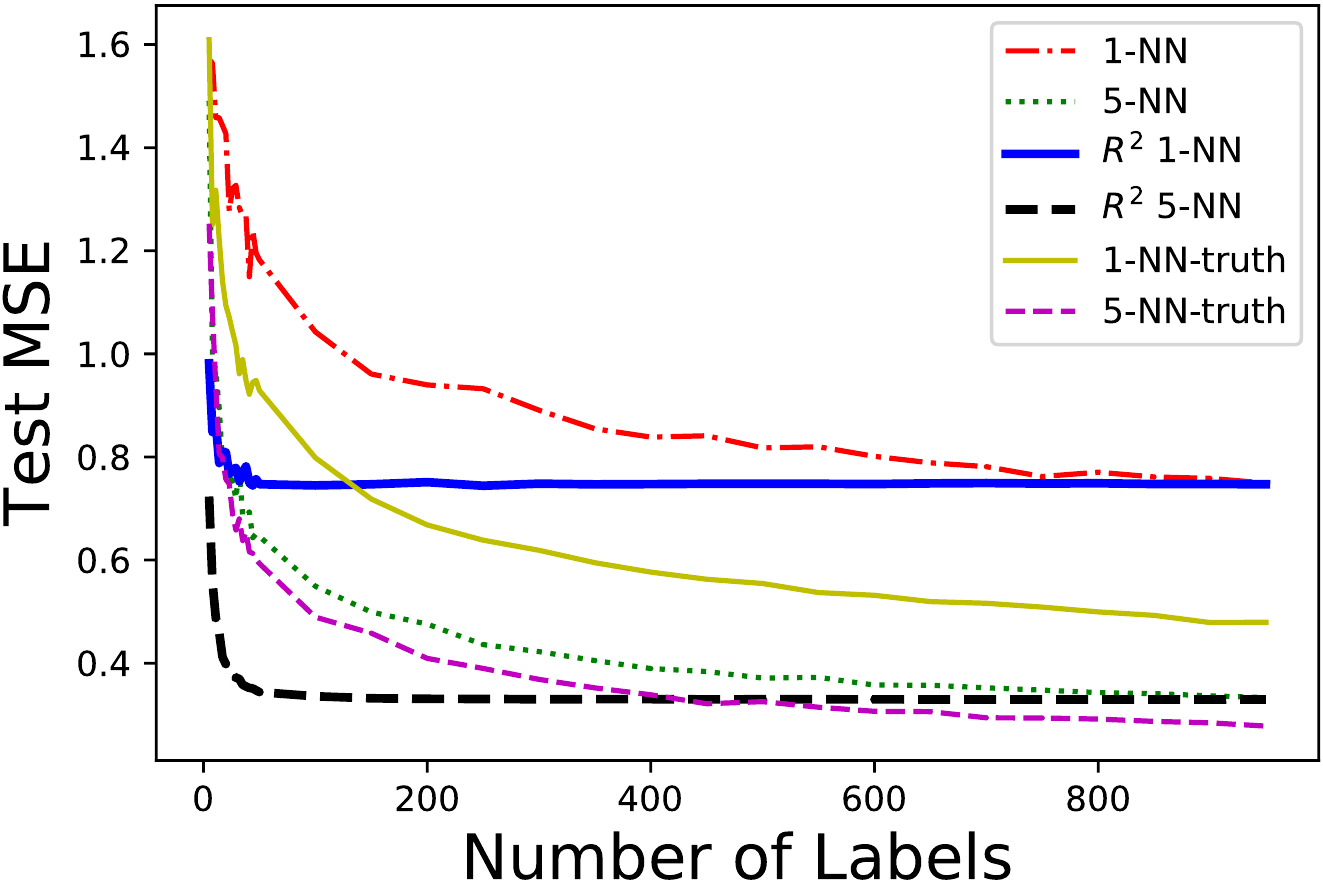}
			\caption{Noisy ranking}
			\label{fig:res_sim_noisy_rank}
		\end{subfigure}\\
	\vspace{4mm}
		\begin{subfigure}[b]{0.45\textwidth}
			\centering
			\includegraphics[width=\textwidth]{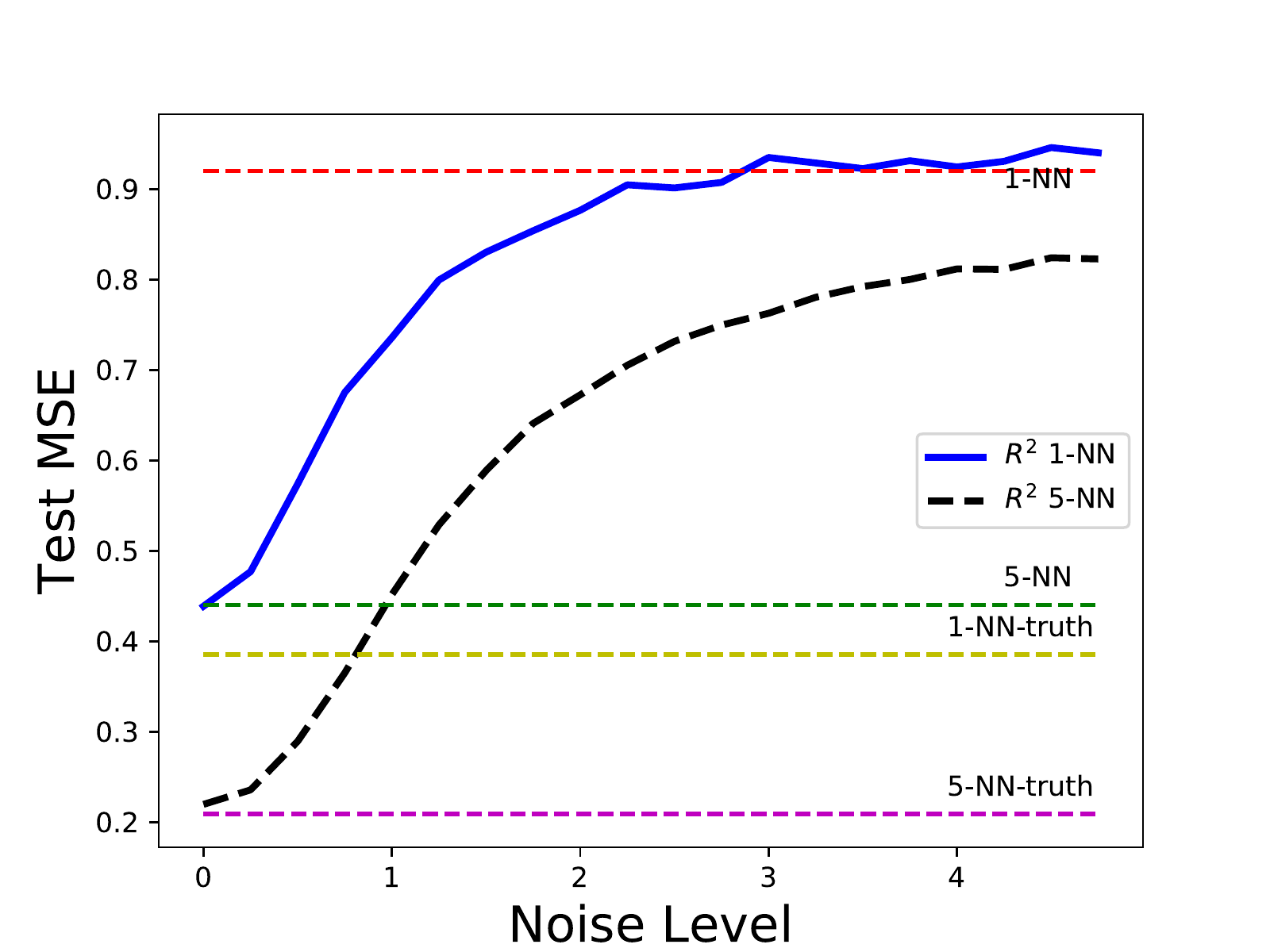}
			\caption{Effect of ranking noise, $m=100$}
			\label{fig:res_sim_rank_noise}
		\end{subfigure}		
			\begin{subfigure}[b]{0.45\textwidth}
		\centering
		\includegraphics[width=\textwidth]{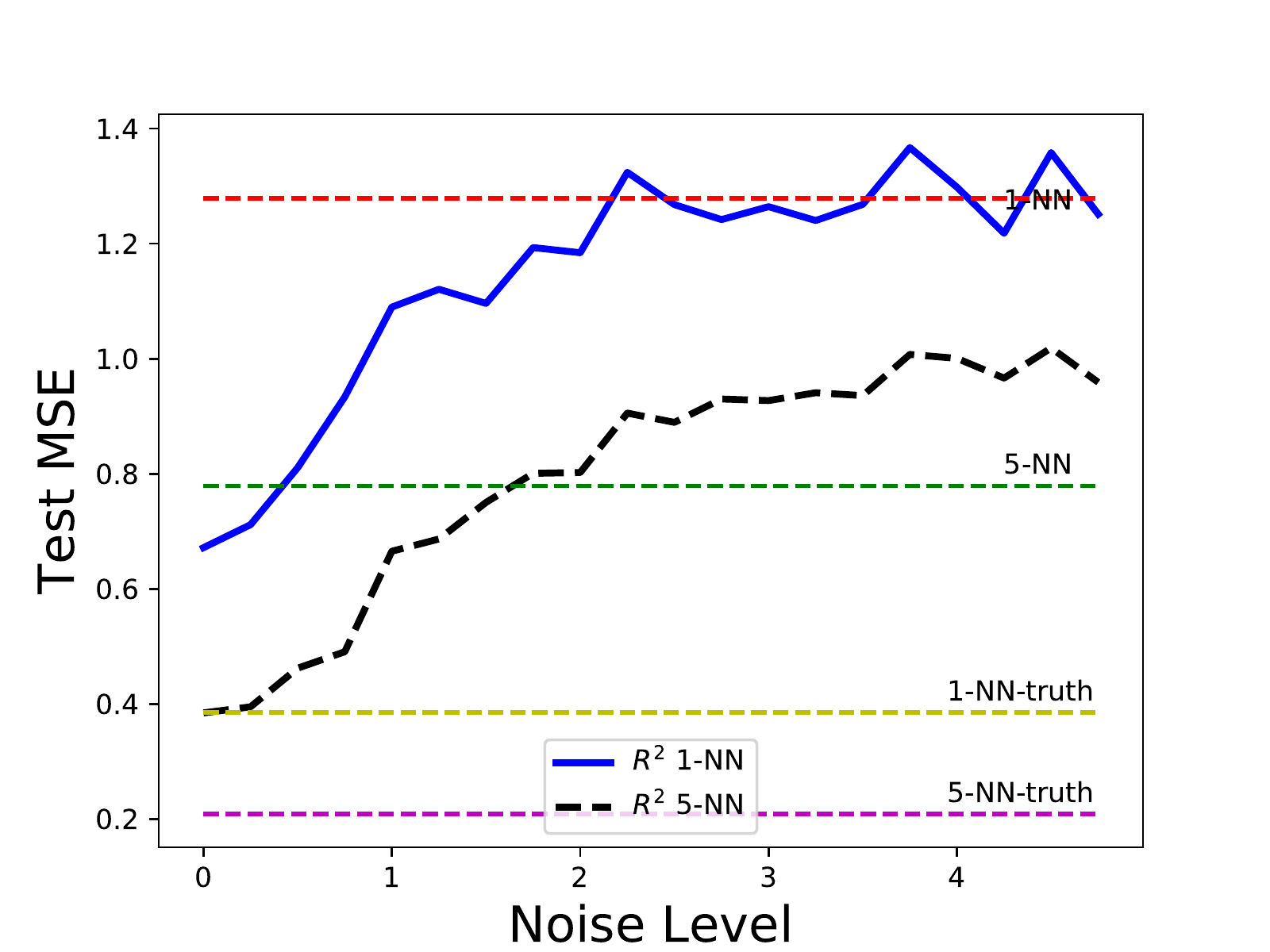}
		\caption{Effect of ranking noise, $m=10$}
		\label{fig:res_sim_rank_noise_m10}
	\end{subfigure}
		\caption{Experiments on simulated data for \RR. \textsf{1-NN} and \textsf{5-NN} represent algorithms using noisy label data only; \textsf{\RR 1-NN} and \textsf{\RR 5-NN} use noisy labels as well as rankings; \textsf{1-NN-truth} and \textsf{5-NN-truth} use perfect label data only.}
		\label{fig:expr_simulated}
		
	\end{figure}
	\textbf{Data Generation.} We generate simulated data following~\citet{hardle2012nonparametric}.
	We take $d=8$, and sample $X$ uniformly random from $[0,1]^d$. Our target function is $f(x)=\sum_{i=1}^d f^{(i \text{ mod } 4)}(x_i)$, where $x_i$ is $x$'s $i$-th dimension, and
	\begin{align*}
	f^{(1)}(x)=px-1/2, ~~~~&f^{(2)}(x)=px^3-1/3,\\
	f^{(3)}(x)=-2\sin(-px), ~~~~~&f^{(4)}(x)=e^{-px}+e^{-1}-1
	\end{align*}
	with $p$ sampled uniformly at random in $[0,10]$. We generate a training and a test set each of size $n=1000$ samples respectively. We rescale $f$ so that it has 0 mean and unit variance over the training set. This makes it easy to control the noise that we add relative to the function value. For training data, we generate the labels as $y_i=f(X_i)+\varepsilon$ where $\varepsilon\sim \mathcal{N}(0,0.5^2)$, for the training data $\{X_1,\ldots,X_n\}$. 
	At test time, we compute the MSE $\frac{1}{n}\sum_{i=1}^{n} (f(X^{\text{test}}_i)-\widehat{f}(X^{\text{test}}_i))^2$ for the test data $\{X^{\text{test}}_1,\ldots,X^{\text{test}}_{n}\}$.
	
	\vspace{0.1cm}
	
	
\noindent	\textbf{Setup and Baselines.} We test \RR with either 1-NN or 5-NN for our simulated data, denoted as \RR 1-NN and \RR 5-NN respectively. We compare them with the following baselines: i) 1-NN and 5-NN using noisy labels $(X,y)$ only. Since \RR uses ordinal data in addition to labels, it should have lower MSE than 1-NN and 5-NN. ii) 1-NN and 5-NN using perfect labels $(X,f(X))$. Since these algorithms use perfect labels, when all sample points are labeled they serve as a benchmark for our algorithms.

	\vspace{0.1cm}

\noindent	\textbf{\RR with rankings.} Our first set of experiments suppose that \RR has access to the ranking over all 1000 training samples, while the $k$-NN baseline algorithms only have access to the labeled samples. We measure the cost as the number of labels in this case. Figure \ref{fig:res_sim_perf_rank} compares \RR with baselines when the ranking is perfect. R$^2$ 1-NN and R$^2$ 5-NN exhibited better performance than their counterparts using only labels, whether using noisy or perfect labels; in fact, we find that R$^2$ 1-NN and R$^2$ 5-NN perform nearly as well as 1-NN or 5-NN using all 1000 perfect labels, while only requiring around 50 labeled samples.
	
	In Figure \ref{fig:res_sim_noisy_rank}, \RR uses an input ranking obtained from noisy labels $\{y_1,...,y_n\}$. In this case, the noise in the ranking is dependent on the label noise, since the ranking is directly derived from the noisy labels. This makes the obtained labels consistent with the ranking, and thus eliminates the need for isotonic regression in Algorithm \ref{algo:compreg_iso}. Nevertheless, we find that the ranking still provides useful information for the unlabeled samples.
	In this setting, \RR outperformed its 1-NN and 5-NN counterparts using noisy labels. However, \RR was outperformed by algorithms using perfect labels when $n=m$. As expected, \RR and $k$-NN with noisy labels achieved identical MSE when $n=m$, as ranking noise is derived from noise in labels. 
	

	We consider the effect of \emph{independent} ranking noise in Figure \ref{fig:res_sim_rank_noise}. We fixed the number of labeled/ranked samples to 100/1000 and varied the noise level of ranking. For a noise level of $\sigma$, the ranking is generated from
	\begin{align}
	y'=f(X)+\varepsilon' \label{eqn:comp_noisy}
	\end{align}
	where $\varepsilon'\sim \mathcal{N}(0,\sigma^2)$. We also plot the performance of 1-NN and 5-NN using 100 noisy labels and 1,000 perfect labels for comparison. 
	We varied $\sigma$ from 0 to 5 and plotted the MSE. We repeat these experiments 50 times.
	
	For both \RR 1-NN and 5-NN -- despite the fact that they use noisy labels -- their performance is close to the NN methods using noiseless labels.
As $\sigma'$ increases, both methods start to deteriorate, with \RR 5-NN hitting the naive 5-NN method at around $\sigma'=1$ and \RR 1-NN at around $\sigma'=2.5$. 
This shows that \RR is robust to ranking noise of comparable magnitude to the label noise. We show in Figure \ref{fig:res_sim_rank_noise_m10} the curve when we use 10 labels and 1000 ranked samples, where a larger amount of ranking noise can be tolerated.

\vspace{0.1cm}

\noindent	\textbf{\RR with comparisons.} We also investigate the performance of \RR when we have pairwise comparisons instead of a total ranking. We train a rankSVM with an RBF kernel with 
a bandwidth of 1, i.e. $k(x,x')=\exp(-\|x-x'\|_2^2)$, as described in Section \ref{sec:mod_algo}. 
	Our rankSVM has a ranking error of $\nu=11.8\%$ on the training set and $\nu=13.8\%$ on the validation set. For simplicity, we only compare with 5-NN here since it gives best performance amongst the label-only algorithms.
	The results are depicted in Figure \ref{fig:expr_comp_ranksvm}. When comparisons are perfect, we first investigate the effect of the cost ratio in Figure \ref{fig:comp_nonoise_ranksvm_cost}. We fixed the budget to equal to $C=500c$ (i.e., we would have 500 labels available if we only used labels), and each curve corresponds to a value of $m\in \{50,100,200\}$ and a varied $n$ such that the total cost is $C=500c$. We can see for almost all choices of $m$ and cost ratio, \RR provides a performance boost. In Figure \ref{fig:comp_nonoise_ranksvm} we fix $c=5$, and vary the total budget $C$ from 500 to 4,000. We find that \RR outperforms the label-only algorithms in most setups.
	
	In Figures \ref{fig:comp_withnoise_ranksvm_cost} and \ref{fig:comp_withnoise_ranksvm}, we consider the same setup of experiments, but with comparisons generated from (\ref{eqn:comp_noisy}), where $\varepsilon'\sim \mathcal{N}(0,0.5^2)$. Note that here the noise $\varepsilon'$ is of the same magnitude but independent from the label noise. Although \RR gave a less significant performance boost in this case, it still outperformed label-only algorithms when $c\geq 2$.
	\begin{figure}[htb!]
	\centering
	\begin{subfigure}[b]{0.45\textwidth}
	\centering
	\includegraphics[width=\textwidth]{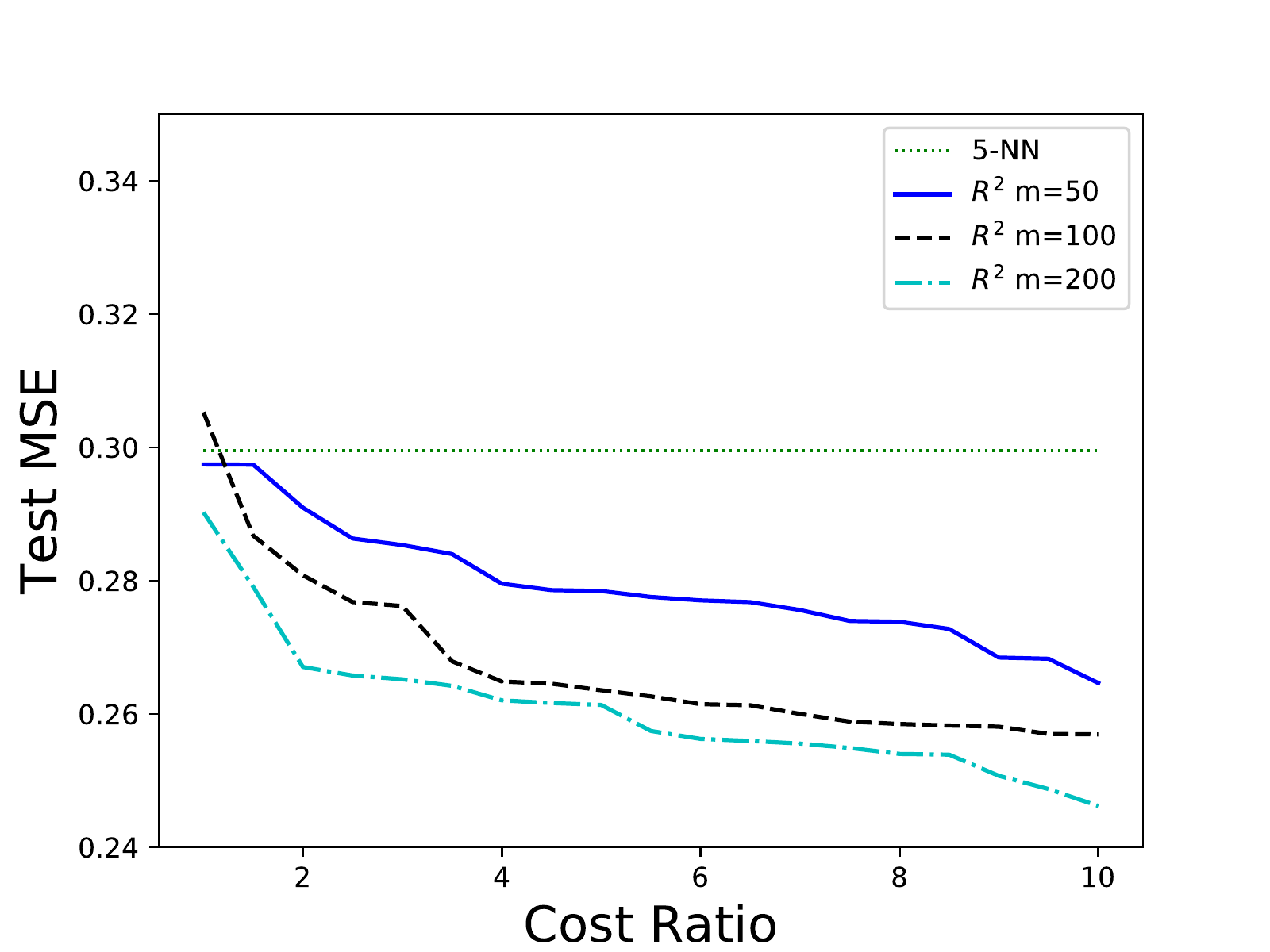}
	\caption{\label{fig:comp_nonoise_ranksvm_cost}Perfect comparisons, $C=500c$ }
\end{subfigure}
	\begin{subfigure}[b]{0.45\textwidth}
		\centering
		\includegraphics[width=\textwidth]{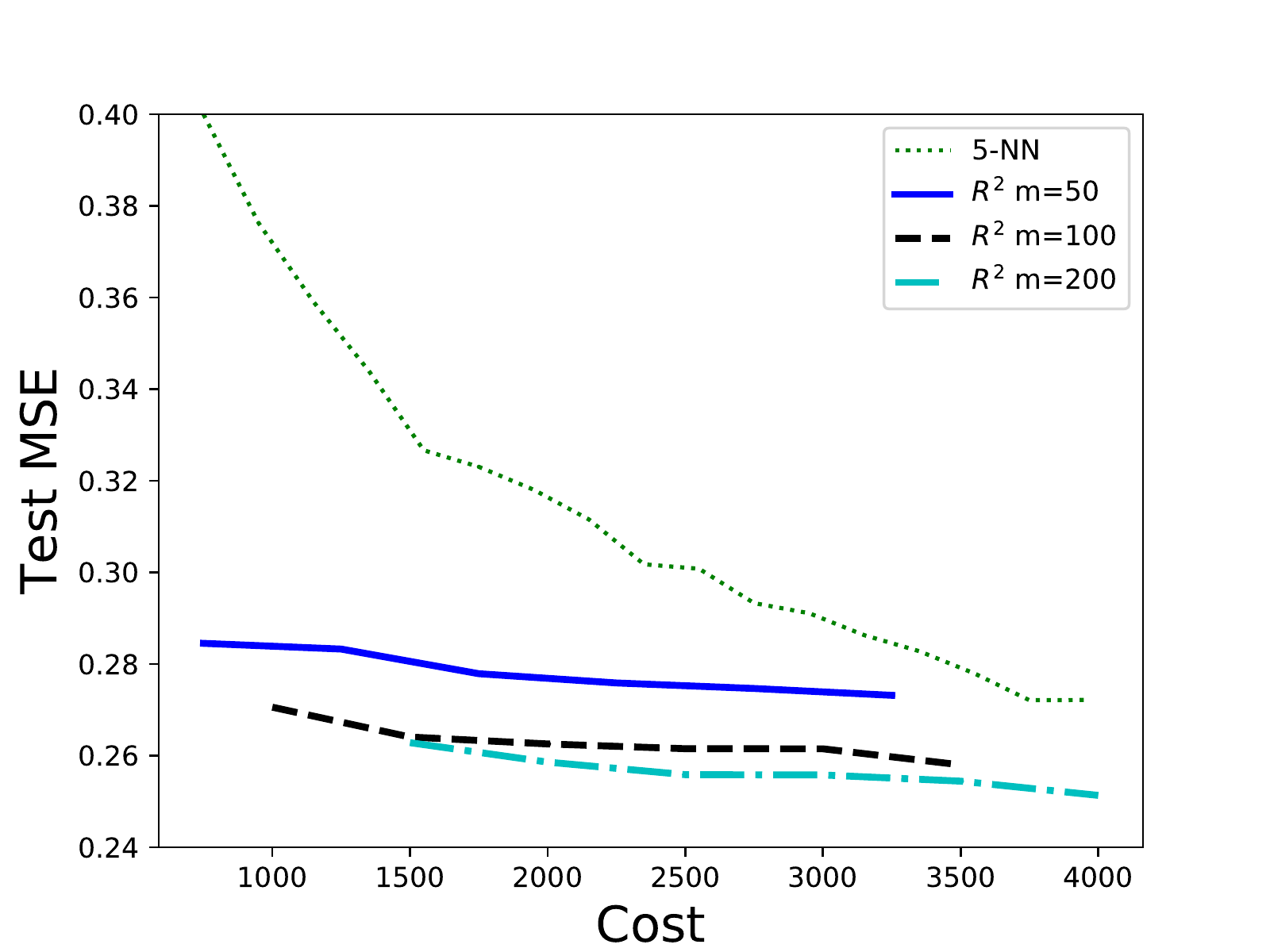}
		\caption{\label{fig:comp_nonoise_ranksvm}Perfect comparisons, $c=5$}
	\end{subfigure}\\		
	\begin{subfigure}[b]{0.45\textwidth}
	\centering
	\includegraphics[width=\textwidth]{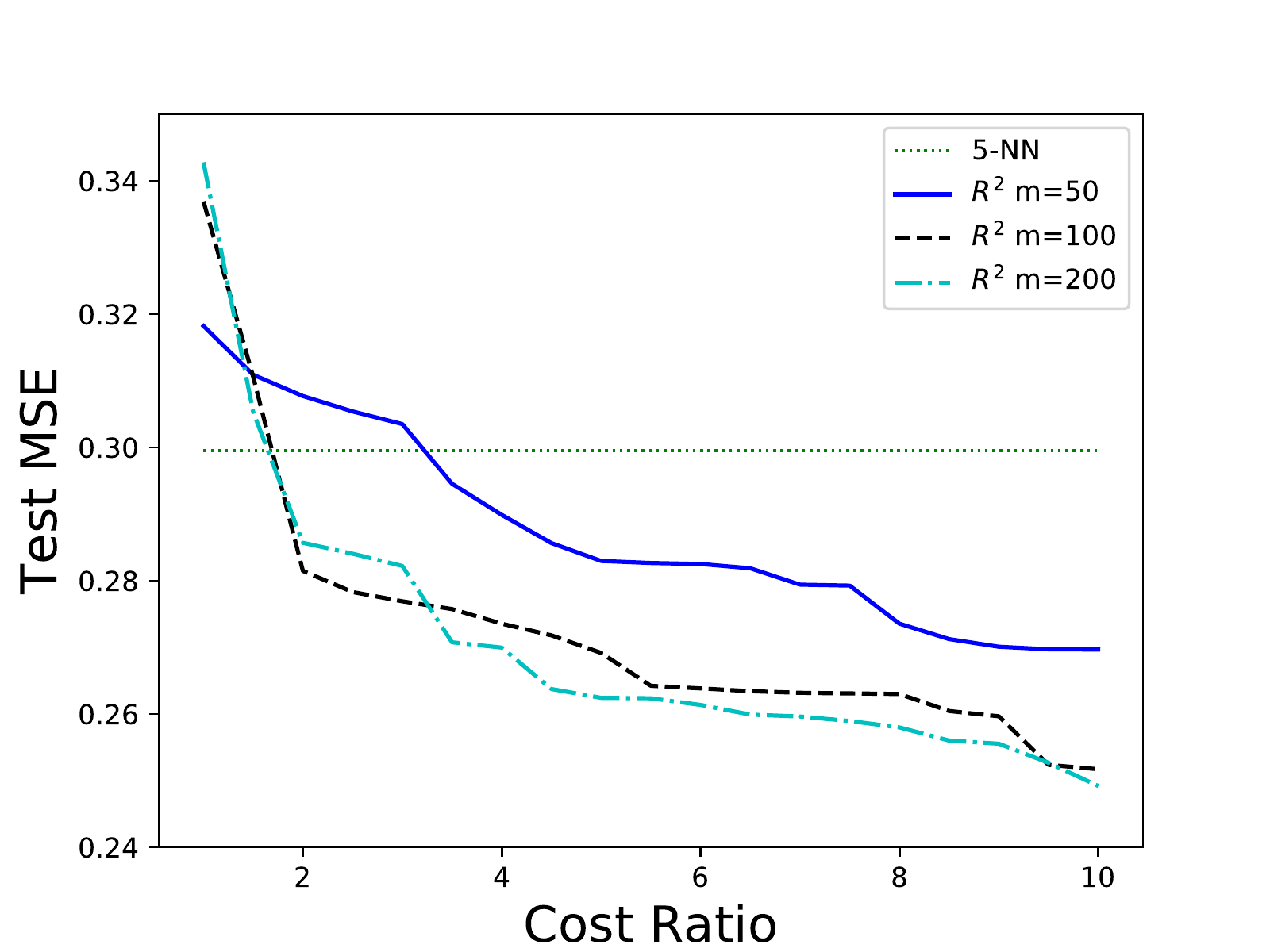}
	\caption{Noisy comparisons, $C=500c$}
	\label{fig:comp_withnoise_ranksvm_cost}
\end{subfigure}
	\begin{subfigure}[b]{0.45\textwidth}
		\centering
		\includegraphics[width=\textwidth]{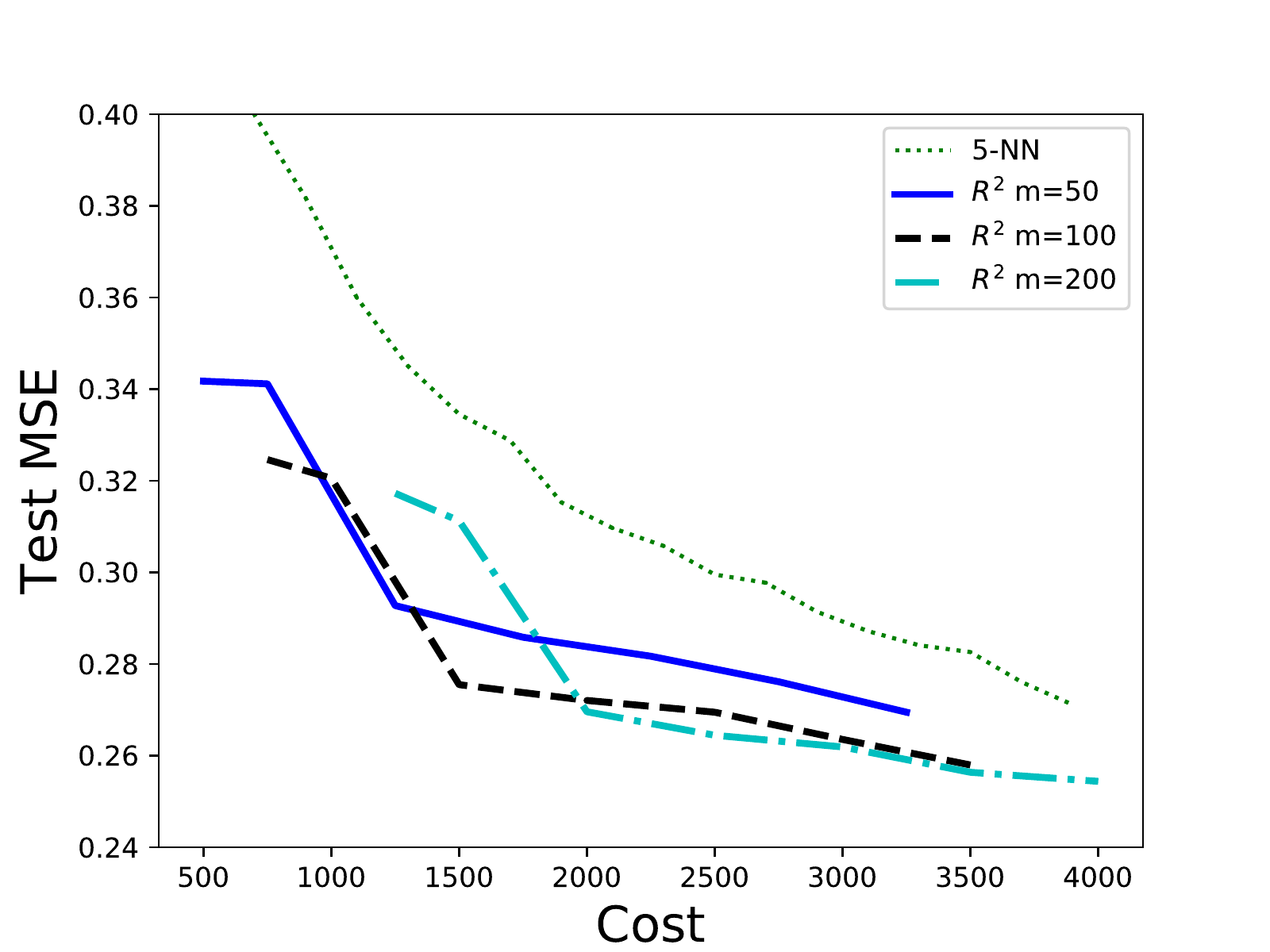}
		\caption{Noisy comparisons, $c=5$}
		\label{fig:comp_withnoise_ranksvm}
	\end{subfigure}	
	\caption{Experimental results on synthetic dataset for \RR with comparisons. Each curve corresponds to a fixed $m$, and we vary $n$ as the cost ratios or total budget change. Note that curves start at different locations because of different $m$ values. }
	\label{fig:expr_comp_ranksvm}
	
\end{figure}

	\subsubsection{Simulated Data for CLR}
	\begin{figure}[htb!]
	\centering
	\begin{subfigure}[b]{0.45\textwidth}
		\centering
		\includegraphics[width=\textwidth]{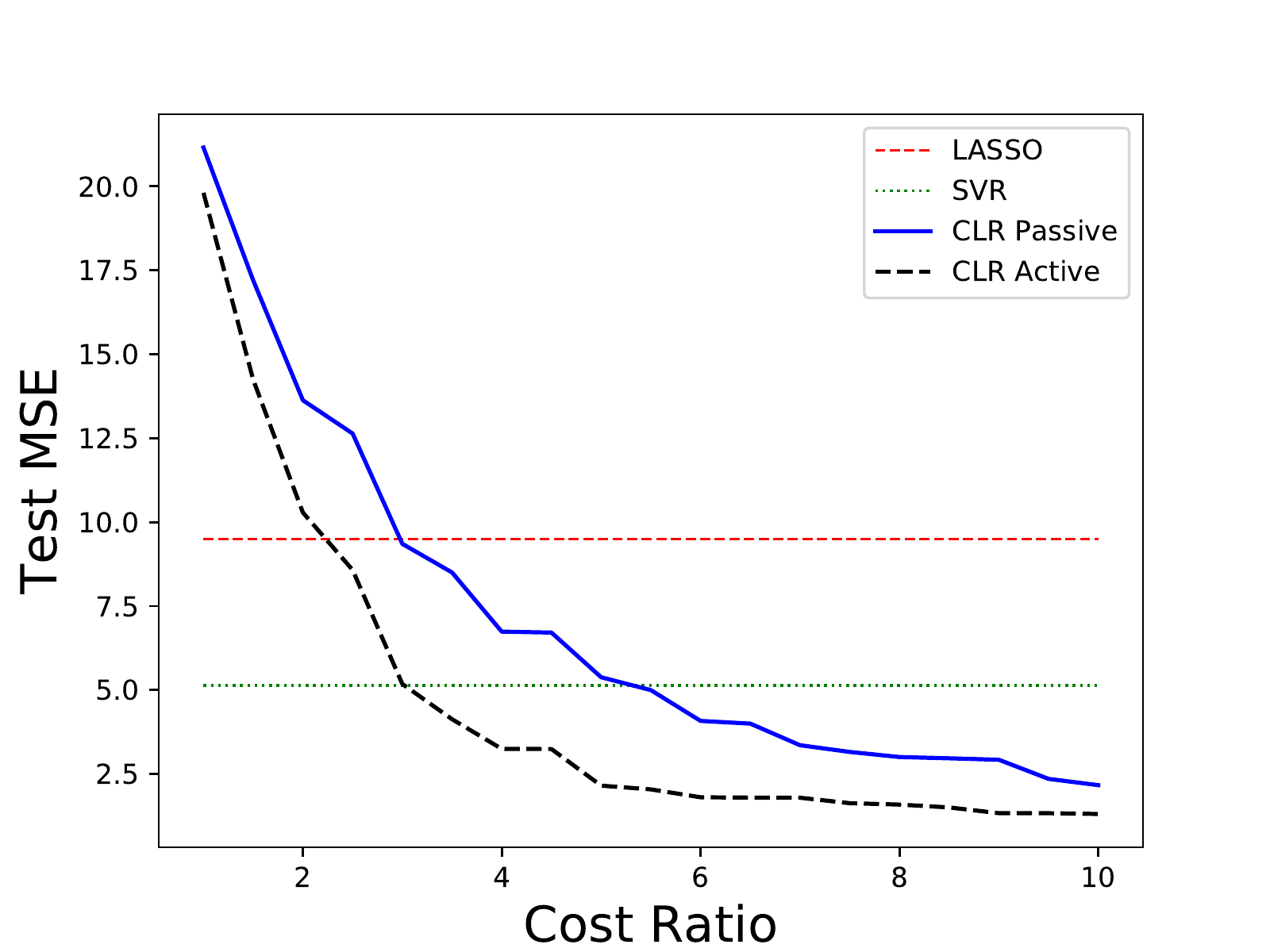}
		\caption{\label{fig:comp_linear_cost}Varying cost ratio $c$ with $C=50c$ }
	\end{subfigure}
	\begin{subfigure}[b]{0.45\textwidth}
		\centering
		\includegraphics[width=\textwidth]{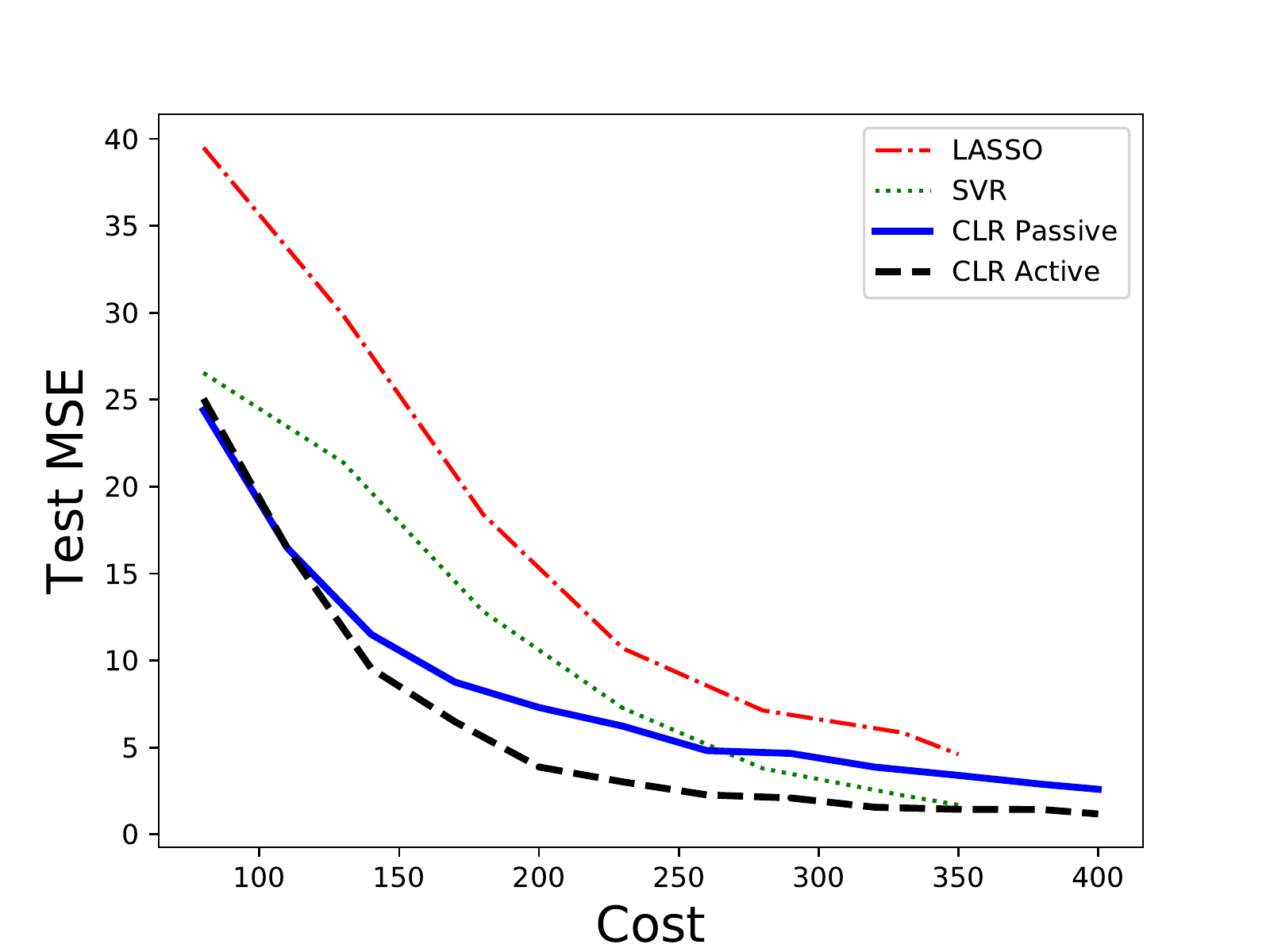}
		\caption{\label{fig:comp_linear_budget}Varying budget $C$ with $c=5$}
	\end{subfigure}\\		
	\caption{Experimental results on synthetic dataset for CLR.}
	\label{fig:expr_linear_synthetic}
	
\end{figure}
	\newcommand{\trainset}{\text{train}}
	\newcommand{\testset}{\text{test}}
	For CLR we only consider the case with a cost ratio, because we find that  a small number of comparisons already suffice to obtain a low error.
	We set $\dimension=50$ and generate both $\featureRV$ and $\gtweight$ from the standard normal distribution $\mathcal{N}(0,I_\dimension)$. We generate noisy labels with distribution $\labelnoise\sim \mathcal{N}(0,0.5^2)$. The comparison oracle generates response using the same noise model: $\compRV=\sign((\inprod{\gtweight}{\featureV_1}+\labelnoise_1)-(\inprod{\gtweight}{\featureV_2}-\labelnoise_2))$ for input $(\featureV_1,\featureV_2)$, with $\labelnoise_1,\labelnoise_2\sim \mathcal{N}(0,0.5^2)$ independent of the label noise.
	
	Model performances are compared in Figure \ref{fig:expr_linear_synthetic}. We first investigate the effect of cost ratio in Figure \ref{fig:comp_linear_cost}, where we fixed the budget to $C=50c$, 
i.e. if we only used labels we would have a budget of 50 labels. The passive comparison query version of CLR requires  roughly $c>5$ to outperform baselines, and the active comparison query version requires $c>3$. We also experiment with a fixed cost ratio $c = 5$ and varying budget $C$ in Figure \ref{fig:comp_linear_budget}. The active version outperformed all baselines in most scenarios, whereas the passive version gave a performance boost when the budget was less than 250 (i.e. number of labels is restricted to less than 250 in label only setting). We note that the active and passive versions of CLR only differ in their collection of comparisons; both algorithms (along with the baselines) are given a random set of labeled samples, making it a fair competition.
	
	\subsection{Predicting Ages from Photographs}
	
	\begin{figure}
		\centering
		\begin{subfigure}[b]{0.5\textwidth}
			\centering
			\includegraphics[width=\textwidth]{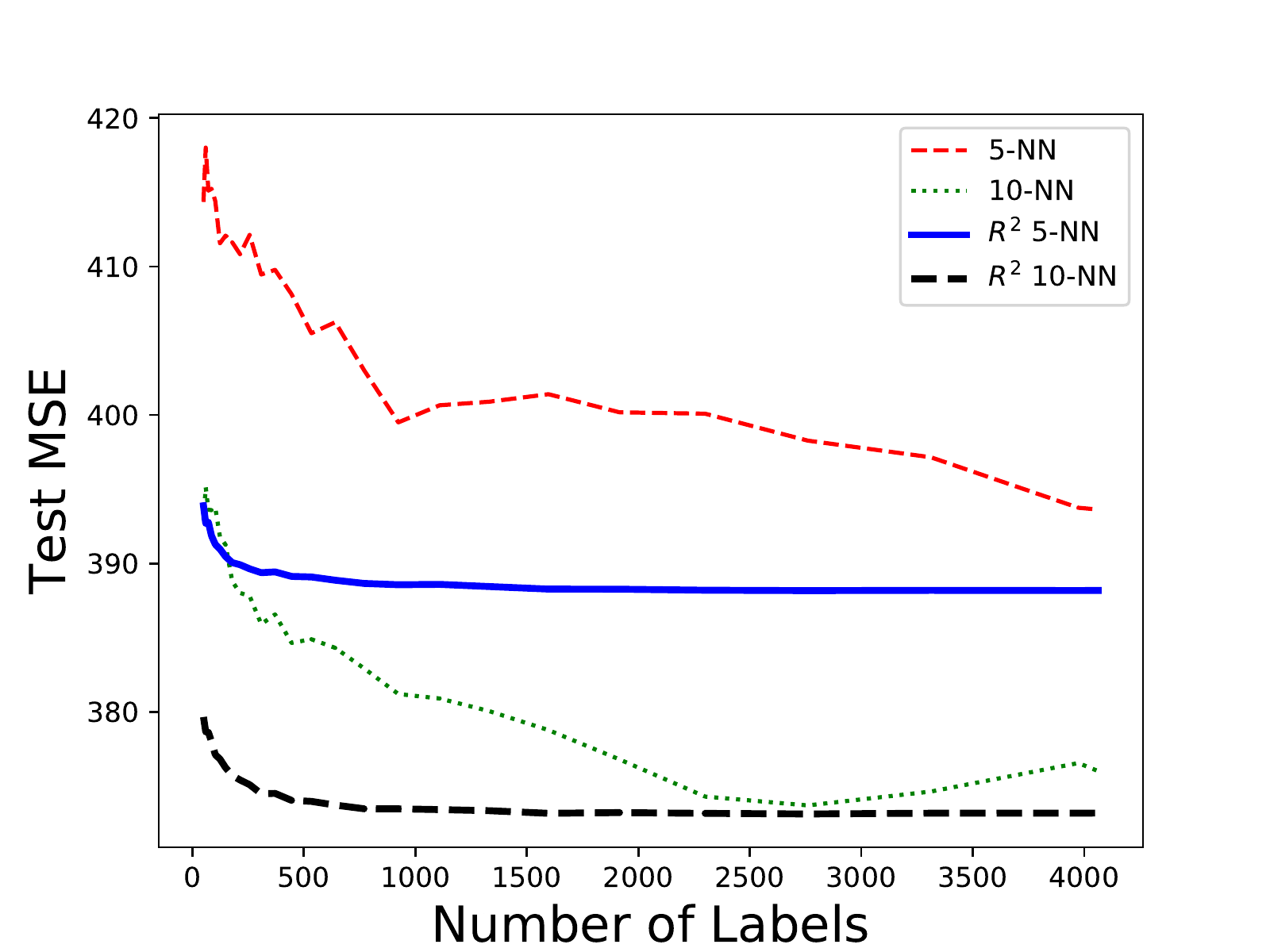}
			\caption{\RR}
			\label{fig:res_age_task1}
		\end{subfigure}%
		\hfill
		\begin{subfigure}[b]{0.5\textwidth}
			\centering
			\includegraphics[width=\textwidth]{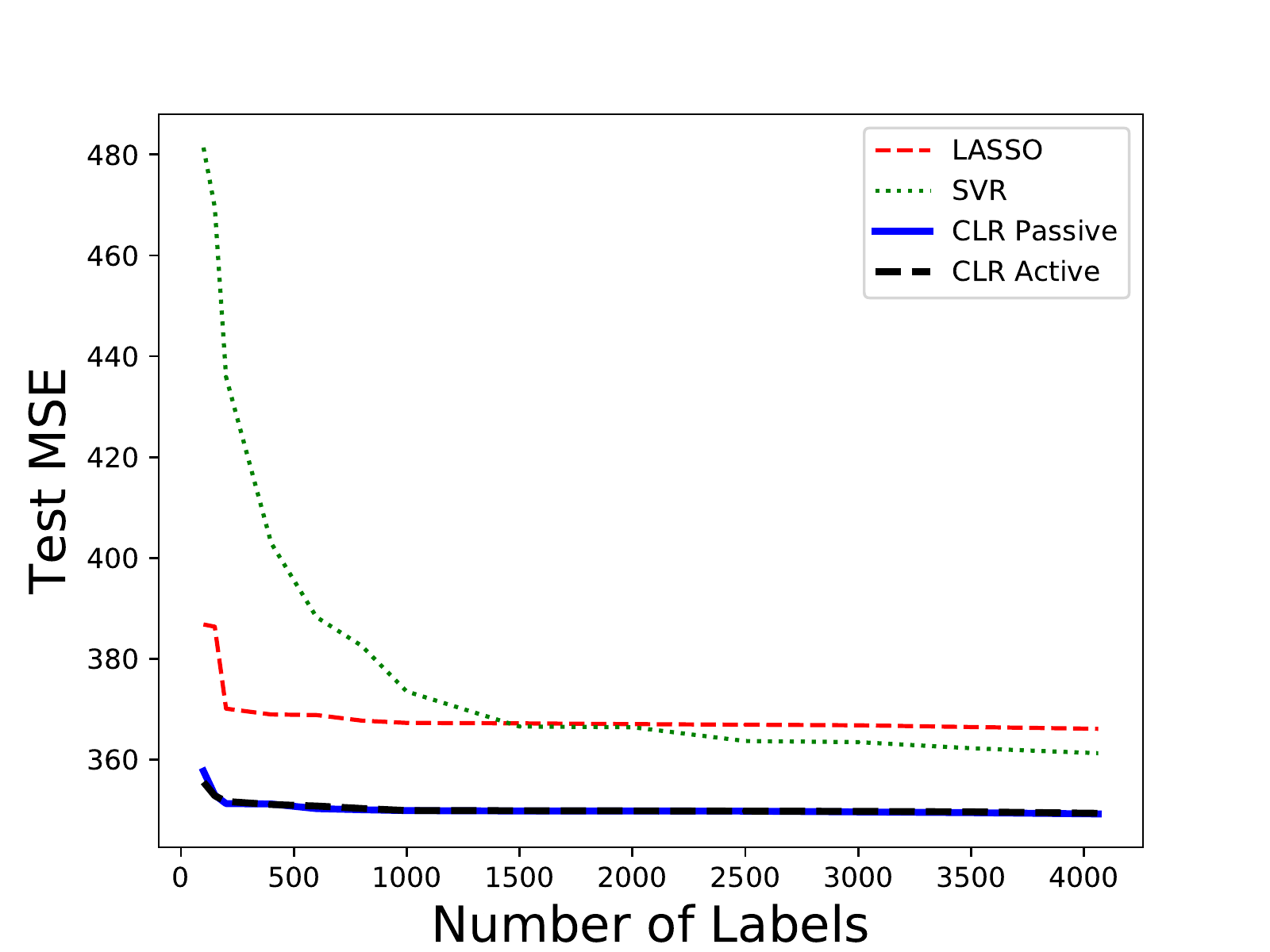}
			\caption{Linear Regression}
			\label{fig:res_linear_age}	
		\end{subfigure}
		
		\caption{\label{fig:res_age}Experimental results on age prediction.
		}
		
	\end{figure}

	To further validate \RR in practice, we consider the task of estimating people's ages from photographs. We use the APPA-REAL dataset \citep{agustsson2017appareal}, which contains 7,591 images, and each image is associated with a biological age and an apparent age. The biological age is the person's actual age, whereas the apparent ages are collected by asking crowdsourced workers to estimate their apparent ages. Estimates from (on average 38 different) labelers are averaged to obtain the apparent age.
	APPA-REAL also provides the standard deviation of the apparent age estimates. The images are divided into 4,063 train, 1,488 validation and 1,962 test samples, and we choose the best hyperparameters using the validation samples.

	
\vspace{0.1cm}

\noindent	\textbf{Task.} We consider the task of predicting biological age. Direct labels come from biological age, whereas the ranking is based on apparent ages. This is motivated by the collection process of many modern datasets:
	for example, we may have the truthful biological age only for a fraction of samples, but wish to collect more through crowdsourcing. In crowdsourcing, people give comparisons based on apparent age instead of biological age. As a consequence, in our experiments we assume additional access to a ranking that comes from the apparent ages. Since collecting crowdsourced data can be much easier than collecting the real biological ages, as discussed earlier, we define the cost as the number of direct labels used in this experiment.
	
\vspace{0.1cm}

\noindent	\textbf{Features and Models.} We extract a 128-dimensional feature vector for each image using the last layer of FaceNet~\citep{schroff2015facenet}. We rescale the features so that every $X\in [0,1]^d$ for \RR, or we centralize the feature to have zero mean and unit variance for CLR. We use 5-NN and 10-NN to compare with \RR in this experiment. Utilizing extra ordinal information, \RR has additional access to the ranking of apparent ages; since CLR does not use a ranking directly we provide it access to 4,063 comparisons (the same size as the training set) based on apparent ages.	
	
	Our results are depicted in Figure \ref{fig:res_age}. The 10-NN version of \RR gave the best overall performance amongst nonparametric methods. \RR 5-NN and \RR 10-NN both outperformed other algorithms when the number of labeled samples was less than 500. 
	Interestingly, we observe that there is a gap between \RR and its nearest neighbor counterparts even when $n=m$, i.e. the ordinal information continues to be useful even when all samples are labeled; this might be because the biological ages are ``noisy'' in the sense that they are also determined by factors not present in image (e.g., lifestyle).
	Similarly, for linear regression methods, our method also gets the lowest overall error for any budget of direct labels. In this case, we notice that active comparisons only have a small advantage over passive comparisons, since the comparison classifiers both converge with a sufficient number of comparisons.
	
	\subsection{Estimating AirBnB Listing Prices \label{sec:expr_AirBnB} }
	In our third set of experiments, we consider the cost of both comparisons and direct labels.
	We use data of AirBnB listings and ask Amazon Mechanical Turk (AMT) workers to estimate their price. To measure the cost, 
	we collect not only the labels and comparisons but also the time taken to answer the question. We use the time as an estimate of the cost. 
	
	\vspace{0.1cm}
	
	\noindent\textbf{Data Collection.} Our data comes from Kaggle\footnote{\url{https://www.kaggle.com/AirBnB/seattle/home}} and contains information about AirBnB postings in Seattle. We use 357 listings as the training set and 93 as the test set. We additionally pick 50 listings from the training set as validation data to select hyperparameters.
	We use the following features for our experiments:
	
	\vspace{0.2cm}
	\hspace{-1cm}
	\begin{tabular}{lll}
		(1) Host response rate, & (2) Host acceptance rate, & (3) Host listings count,\\
		(4) Number of reviews per month,& (5) Number of bathrooms,& (6) Number of bedrooms, \\
		(7) Number of beds, & (8) Number of reviews, & (9) Review scores rating, \\
		\multicolumn{3}{c}{(10) Number of people the house accommodates.}
	\end{tabular}
	\vspace{.2cm}
	
	Each worker from AMT is asked to do either of the following two tasks: i) given the description of an AirBnB listing, estimate the per-night price on a regular day in 2018; ii) given the description of two AirBnB listings, select the listing that  has a higher price. We collect 5 direct labels for each data point in the training set and 9 labels for each data point in the test set. For comparisons, we randomly draw 1,841 pairs from the training set and ask 2 workers to compare their prices.
	
	\vspace{0.1cm}
	
	\noindent\textbf{Tasks.} We consider two tasks, motivated by real-world applications.
	\begin{enumerate}[wide, labelwidth=!, labelindent=0pt]
		\item In the first task, the goal is to predict the real listing price. This is motivated by the case where collecting the real price might be difficult to obtain or involves privacy issues. We assume that our training data, including both labels and comparisons, comes from AMT workers.
		\item In the second task, the goal is to predict the user-estimated price. This can be of particular interest to AirBnB company and house-owners, for deciding the best price of the listing. We do not use the real prices in this case; we use the average of 10 worker estimated prices for each listing in the test set as the ground truth label, and the training data also comes from our AMT results.
	\end{enumerate}


	\vspace{0.1cm}
	
\noindent\textbf{Raw Data Analysis.} Before we proceed to the regression task, we analyze the workers' performance for both tasks based on raw data in Table \ref{tab:comp-acc}. 
Our first goal is to compare a pairwise comparison to an induced comparison, where the induced comparison is obtained by 
making two consecutive direct label queries and subtracting them. Similar to \cite{shah2015estimation}, we observe that comparisons are more accurate than the induced comparisons.

We first convert labels into pairwise comparisons by comparing individual direct labels: namely, for each obtained labeled sample pair $(x_1,y_1),(x_2,y_2)$ where $y_1,y_2$ are the \emph{raw} labels from workers, we create a pairwise comparison that corresponds to comparing $(x_1,x_2)$ with label being sign$(y_1-y_2)$. We then compute the error rate of raw and label-induced comparisons for both Task 1 and 2. For Task 1, we directly compute the error rate w.r.t. the true listing price. For Task 2, we do not have the ground truth user prices; we instead follow the approach of \cite{shah2015estimation} to compute the fraction of disagreement between comparisons. Namely, in either the raw or label-induced setting, for every pair of samples $(x_i,x_j)$ we compute the majority of labels $z_{ij}$ based on all comparisons on $(x_i,x_j)$. The disagreement on $(x_i,x_j)$ is computed as the fraction of comparisons that disagrees with $z_{ij}$, and we compute the overall disagreement by averaging over all possible $(x_i,x_j)$ pairs.

If an ordinal query is equivalent to two consecutive direct queries and subtracting the labels, we would expect a similar accuracy/disagreement for the two kinds of comparisons. However our results in Table \ref{tab:comp-acc} show that this is not the case: direct comparison queries have better accuracy for Task 1, as well as a lower disagreement within collected labels. This shows that a comparison query cannot be replaced by two consecutive direct queries. We do not observe a large difference in the average time to complete a query in Table \ref{tab:comp-acc}; however the utility of comparisons in predicting price can be higher since they yield information about two labels. 
Further analysis of the raw data is given in Appendix \ref{sec:app_expr}.

\begin{table}[htb!]
	\begin{center}
		\begin{tabular}{l | c | c}
			\hline  Performance &  Comparisons & Labels  \\ \hline
			Task 1 Error & \bf 31.3\% & 41.3\%\\
			Task 2 Disagreement & \bf 16.4\% & 29.5\%\\
			Average Time & 64s & \bf 63s\\
			\hline
		\end{tabular}
	\end{center}
	\caption{\label{tab:comp-acc} Performance of comparisons versus labels for both tasks. 
	}
\end{table}



	\noindent\textbf{Results.} We plot the experimental results in Figure \ref{fig:res_AirBnB}. For nonparametric regression, \RR had a significant performance boost over the best nearest neighbor regressor under the same total worker time, especially for Task 1. For Task 2, we observe a smaller improvement, but \RR is still better than pure NN methods for all total time costs.
	
	For linear regression, we find that the performance of CLR varies greatly with $m$ (number of labels), whereas its performance does not vary as significantly with the number of comparisons. In fact, the errors of both CLR passive and active already plateau 
with a mere 50 comparisons, since the dimension of data is small ($d=10$). So deviating from our previous experiments, in this setting, we vary the number of labels in Figure \ref{fig:res_linear_actual} and \ref{fig:res_linear_useravg}.
	 As in the nonparametric case, CLR also outperforms the baselines in both tasks. For Task 1, the active and passive versions of CLR perform similarly, whereas active queries lead to a moderate performance boost on Task 2. This is probably because the error on Task 2 is much lower than that on Task 1 (see Table \ref{tab:comp-acc}), and active learning typically has an advantage over passive learning when the noise is not too high.
	
	\begin{figure}[ht!]
		\centering
		\begin{subfigure}[b]{0.5\textwidth}
			\centering
			\includegraphics[width=\textwidth]{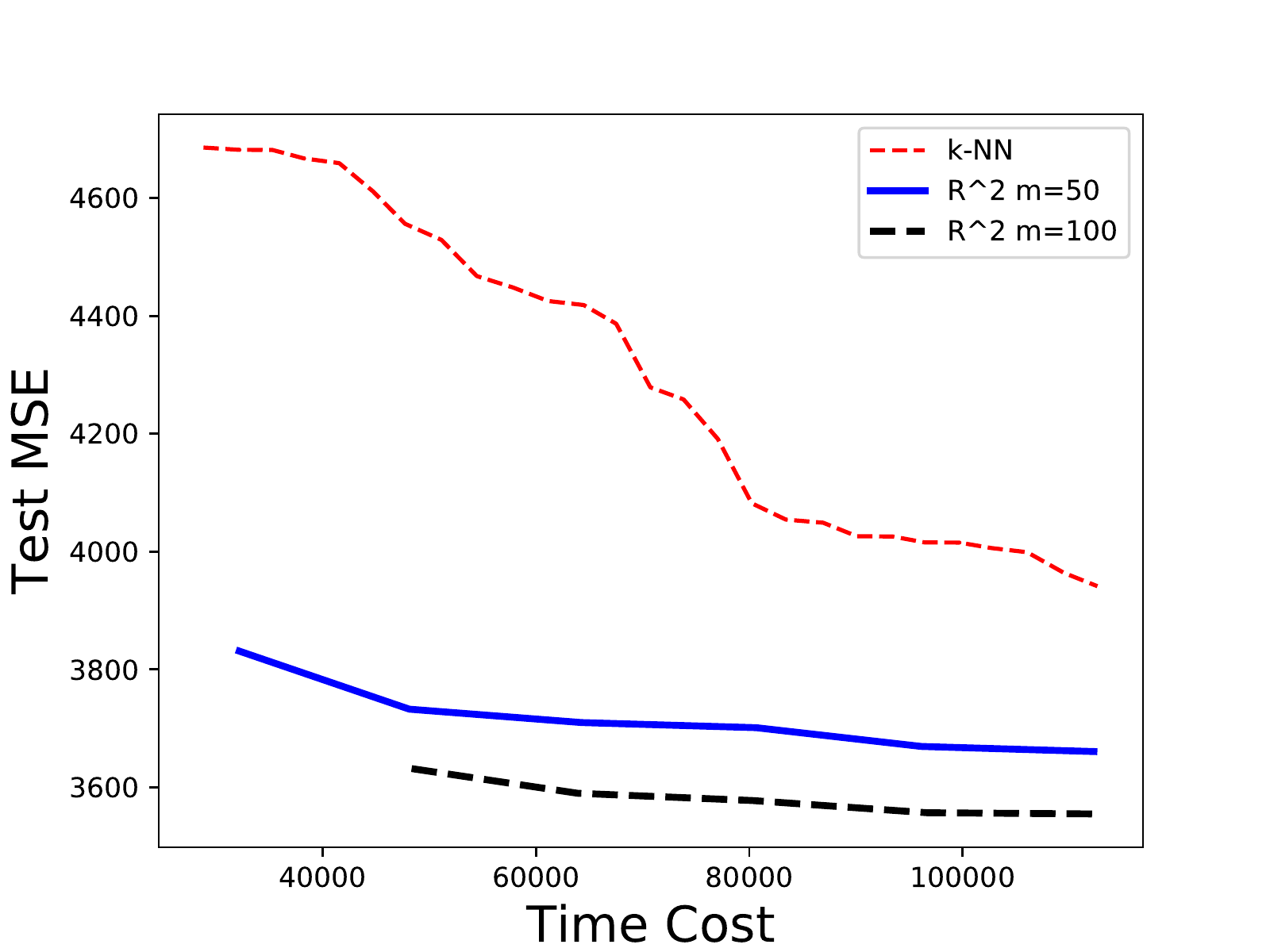}
			\caption{Nonparametric, Task 1}
			\label{fig:res_NN_actual}
		\end{subfigure}%
		\hfill
		\begin{subfigure}[b]{0.5\textwidth}
			\centering
			\includegraphics[width=\textwidth]{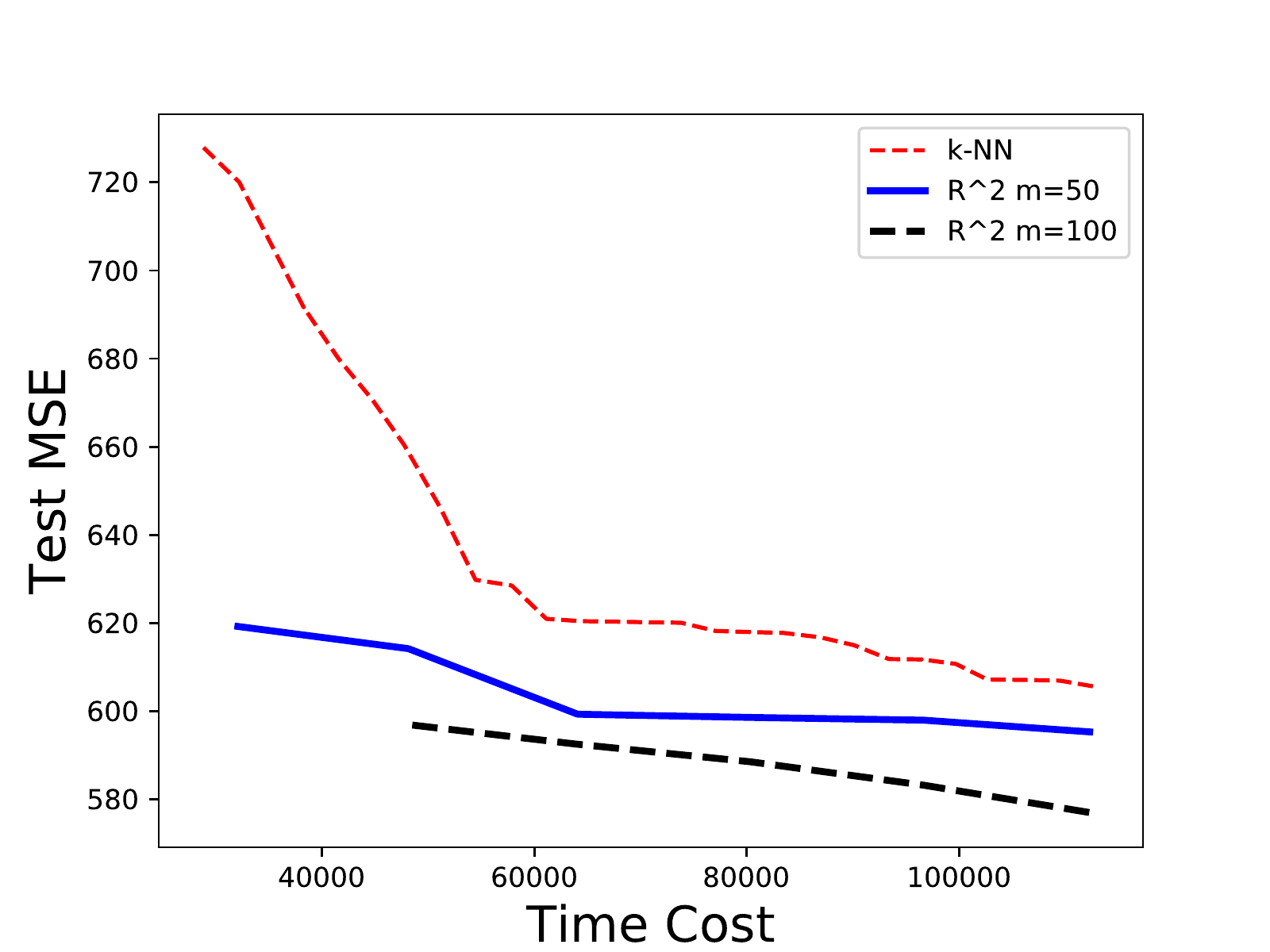}
			\caption{Nonparametric, Task 2}
			\label{fig:res_NN_useravg}	
		\end{subfigure}
		\begin{subfigure}[b]{0.5\textwidth}
	\centering
	\includegraphics[width=\textwidth]{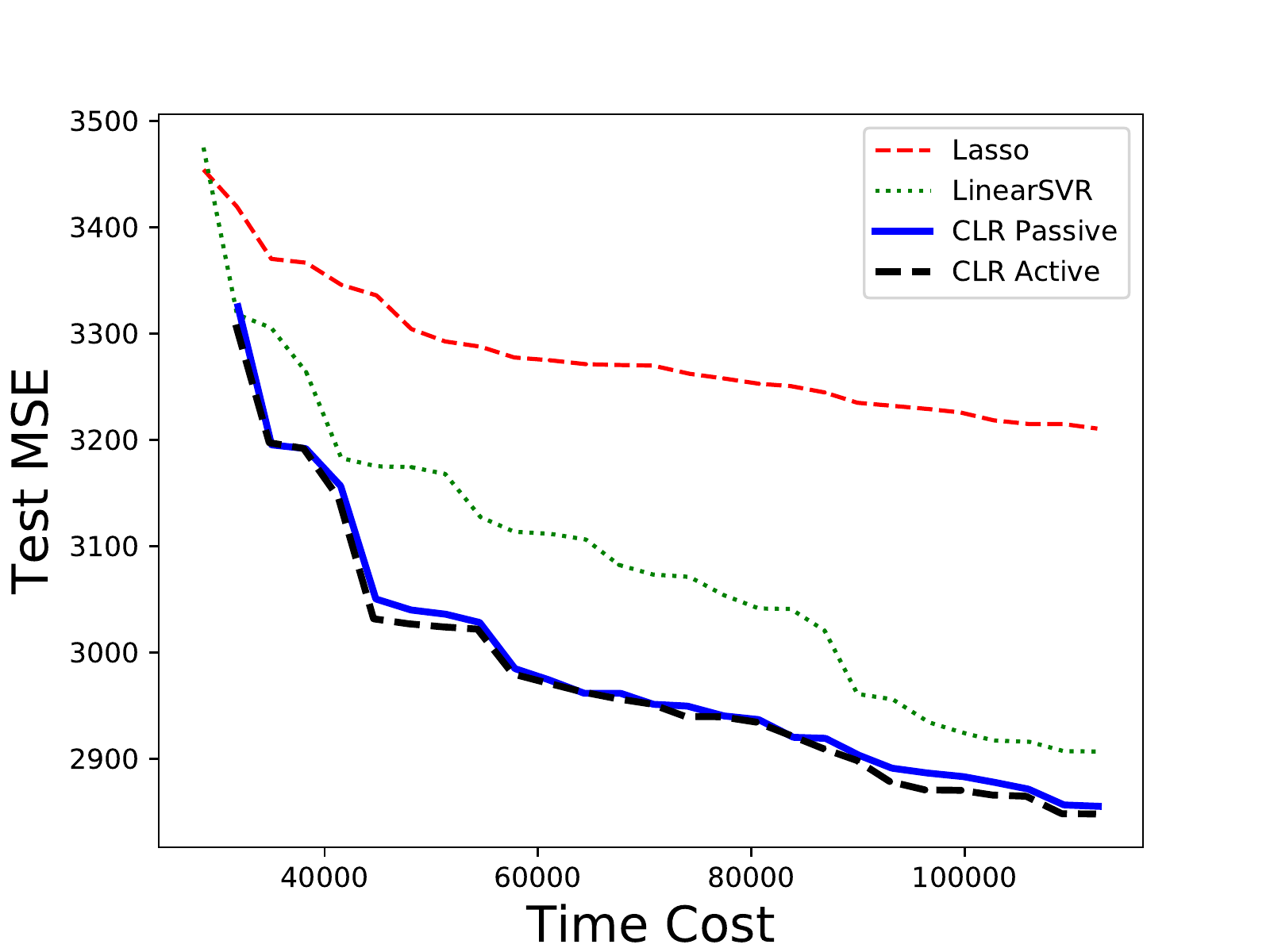}
	\caption{Linear, Task 1}
	\label{fig:res_linear_actual}
\end{subfigure}%
\hfill
\begin{subfigure}[b]{0.5\textwidth}
	\centering
	\includegraphics[width=\textwidth]{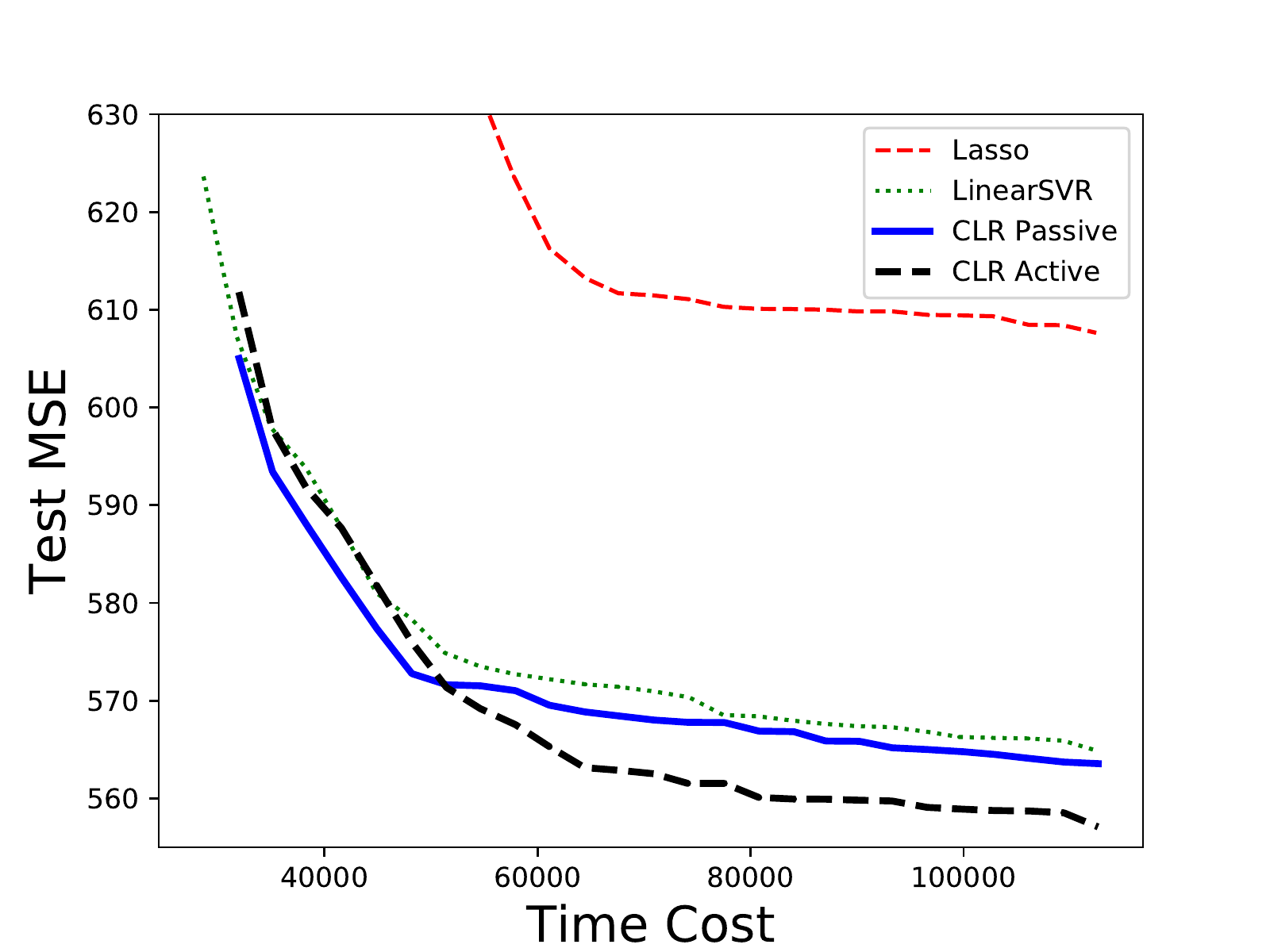}
	\caption{Linear, Task 2}
	\label{fig:res_linear_useravg}	
\end{subfigure}		
		\caption{\label{fig:res_AirBnB}Results for AirBnB price estimation. In (a)(b), each curve has a fixed $m$ with varied $n$; for (c)(d), each curve uses only 50 comparisons with a varied number of labels. For Figure (d), LASSO performs much worse than a LinearSVR and we only show part of the curve.
		}
		
	\end{figure}


	\section{\label{sec:discussion}Discussion} 
	We design (near) minimax-optimal algorithms for nonparametric and linear regression using additional ordinal information. 
	In settings where large amounts of ordinal information are available, we find that limited direct supervision suffices to obtain accurate estimates. We provide complementary minimax lower bounds, and illustrate our proposed algorithms on real and simulated datasets. 
	Since ordinal information is typically easier to obtain than direct labels, one might expect in these favorable settings the \RR algorithm to have lower effective cost than an algorithm based purely on direct supervision.
	
	Several directions exist for future work. On nonparametric regression side, it remains to extend our results to the case where the H\"{o}lder exponent $s>1$. In this setting the optimal rate $O\left(\numlabel^{-\frac{-2s}{2s+d}} \right)$ can be faster than the convergence rate of isotonic regression, which can make our algorithm sub-optimal. It is also important to address the setting where both direct and ordinal supervision are actively acquired. For linear regression, an open problem is to consider the bounded noise model for comparisons. Our results can be easily extended to the bounded noise case using the algorithm in \cite{hanneke2009theoretical}, however that algorithm is computationally inefficient. The best efficient active learning algorithm in this bounded noise setting \citep{awasthi2016learning} requires $m\geq O\left(d^{O(\frac{1}{(1-2\lambda)^4})}\right)$ comparisons, and a large gap remains between what can be achieved in the presence and absence of computational constraints.
	
	Motivated by practical applications in crowdsourcing, we list some further extensions to our results:
	
\noindent	\textbf{Partial orders:} In this paper, we focus on ordinal information either in the form of a total ranking or pairwise comparisons. In practice, ordinal information might come in the form of partial orders, where we have several subsets of unlabeled data ranked, but the relation between these subsets is unknown. A straightforward extension of our results in the nonparametric case leads to the following result: if we have $k$ (partial) orderings, each with $n_1, \ldots, n_k$ samples, and $\numlabel_1,\ldots,\numlabel_k$ samples in each ordering are labeled, we can show an upper bound on the MSE of $\numlabel_1^{-2/3}+\cdots+\numlabel_k^{-2/3}+(n_1+\cdots +n_k)^{-2s/d}$. It would be interesting to study the optimal rate, as well as to consider other smaller partial orders.

\noindent	\textbf{Other models for ordinal information:} Beyond the bounded noise model for comparison, we can consider other pairwise comparison models, like Plackett-Luce \citep{plackett1975analysis,luce2005individual} and Thurstone \citep{thurstone1927law}. These parametric models can be quite restrictive and can lead to unnatural results that we can recover the function values even \emph{without} querying any direct labels (see for example \cite{shah2015estimation}).
One might also consider pairwise comparisons with Tsybakov-like noise~\cite{tsybakov2004optimal} which have been studied in the classification setting~\cite{xu2017noise}; the main obstacle here is the lack of computationally-efficient algorithms that aggregate pairwise comparisons into a complete ranking under this noise model.

\noindent \textbf{Other classes of functions:} Several recent papers \citep{chatterjee2015,bellec2015sharp,bellec2018sharp,han2017isotonic} demonstrate the adaptivity (to ``complexity'' of the unknown parameter) of the MLE in shape-constrained problems. Understanding precise assumptions on the underlying smooth function which induces a low-complexity isotonic regression problem is interesting future work. 

	
\section*{Acknowledgements}
We thank Hariank Muthakana for his help on the age prediction experiments. This work has been partially supported by
		the Air Force Research Laboratory (8750-17-2-0212), the National Science Foundation (CIF-1763734 and DMS-1713003), Defense Advanced Research
		Projects Agency (FA8750-17-2-0130),
		and the Multidisciplinary Research Program of the Department
		of Defense (MURI N00014-00-1-0637).
	
	\vskip 0.2in
	\bibliography{yichongref}
	
	\appendix
	
	\section{Additional Experimental Results }
	\phantomsection
	\label{sec:app_expr}
	\textbf{Relation between true and user estimated price.} Figure \ref{fig:scatter_prices} shows a scatter plot of the true listing prices of AirBnB data with respect to the user estimated prices. Although the true prices is linearly correlated with the user prices (with $p$-value $6e-20$), the user price is still very different from true price even if we take the average of 5 labelers. The average of all listings' true price is higher than the average of all user prices by 25 dollars, partially explaining the much higher error when we use user prices to estimate true prices.
	
	\begin{figure}[htb!]
	\centering
	\includegraphics[width=0.7\linewidth]{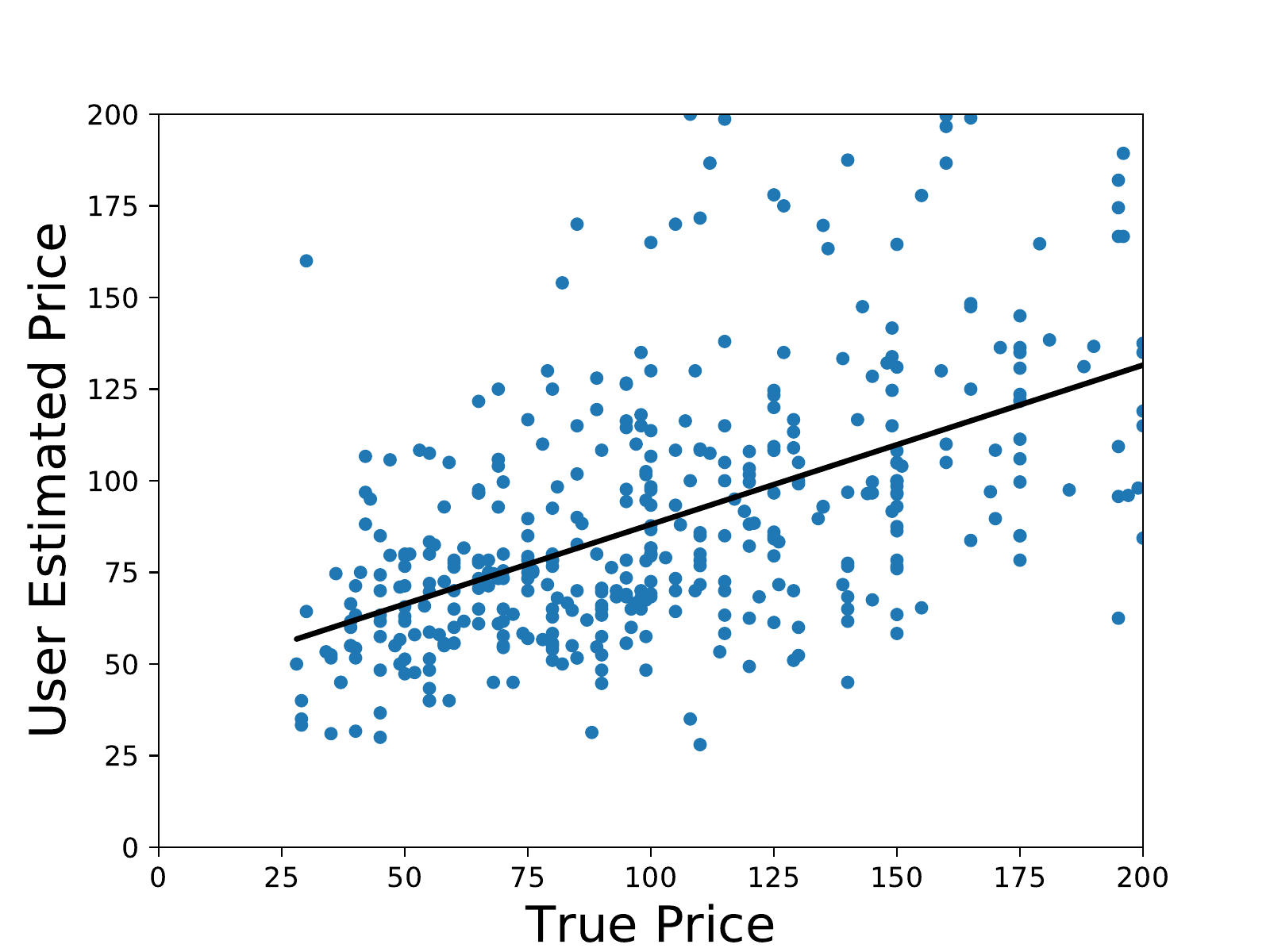}
	\caption{Scatter plot of true prices w.r.t average user estimated prices, along with the ordinary least square result. We only show prices smaller than 200 to make the relation clearer.}
	\label{fig:scatter_prices}
\end{figure}	

\noindent \textbf{An alternative to using RankSVM.} As an alternative to training RankSVM, we can also use nearest neighbors on Borda counts to take into account the structure of feature space: for each sample $x$, we use the score $s(x)=\frac{1}{k} \sum_{x'\in k\text{-NN}(x)} \text{Borda}(x')$, where $k$-NN$(x)$ is the $k$-th nearest neighbor of $x$ in the feature space, including $x$ itself. The scores are then used to produce a ranking. When $c$ is large, this method does provide an improvement over the label-only baselines, but generally does not perform as well as our rankSVM method. The results when cost ratio $c=10$ and comparisons are perfect are depicted in Figure \ref{fig:comp_nonoise_nnborda}.	We use $k=25$ for deciding the nearest neighbors for Borda counts, and 5-NN as the final prediction step. While using \RR with Borda counts do provide a gain over label-only methods, the improvement is less prominent than using rankSVM.
\begin{figure}[htb!]
	\centering
	\includegraphics[width=0.5\textwidth]{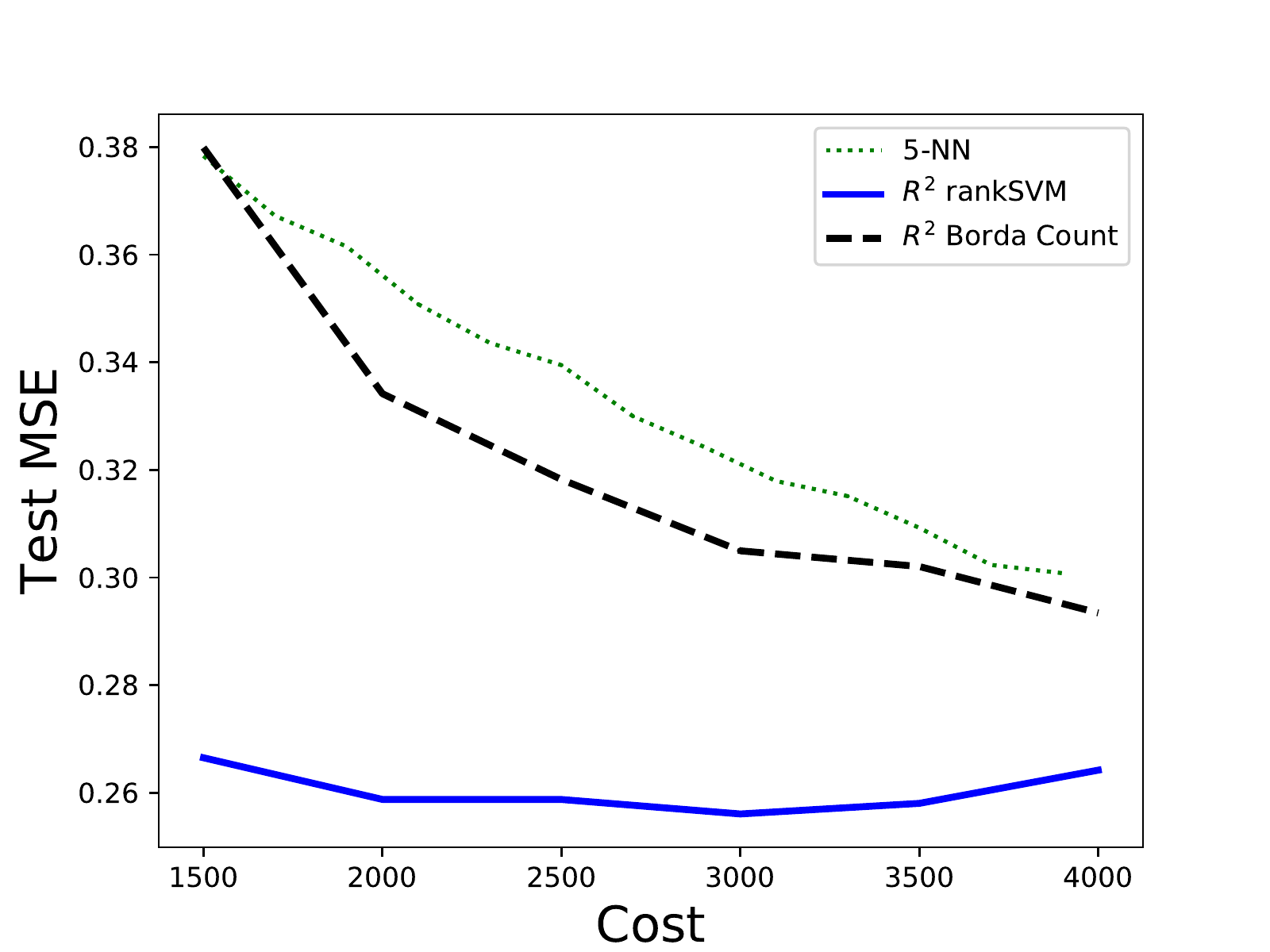}
	
	\caption{\label{fig:comp_nonoise_nnborda} Experimental results on synthetic dataset for \RR with comparisons, using nearest neighbor with Borda counts. Both \RR algorithms uses 5-NN as the final prediction step.
	}
	
\end{figure}

For estimating AirBnB price the results are shown in Figure \ref{fig:res_NN_AirBnB_nnborda}. For task 1, NN of Borda counts introduces an improvement similar to (or less than) RankSVM, but for task 2 it is worse than nearest neighbors. We note that for task 1, the best number of nearest neighbors of Borda counts is 50, whereas for task 2 it is 5 (close to raw Borda counts). We suspect this is due to the larger noise in estimating true price, however a close examination for this observation remains as future work.
\begin{figure}[htb!]
	\centering
	\begin{subfigure}[b]{0.5\textwidth}
		\centering
		\includegraphics[width=\textwidth]{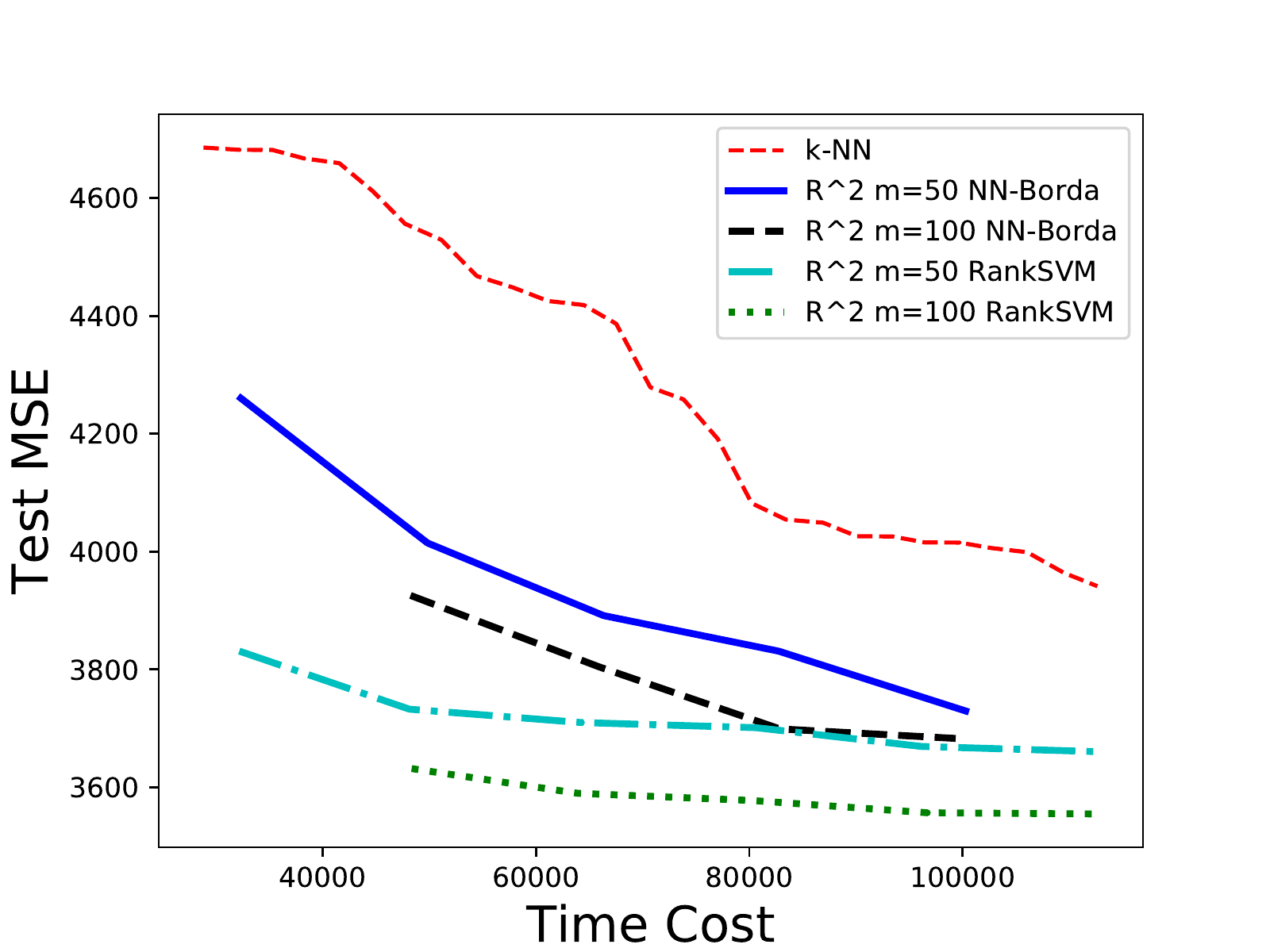}
		\caption{Task 1}
		\label{fig:res_NN_actual_nnborda}
	\end{subfigure}%
	\hfill
	\begin{subfigure}[b]{0.5\textwidth}
		\centering
		\includegraphics[width=\textwidth]{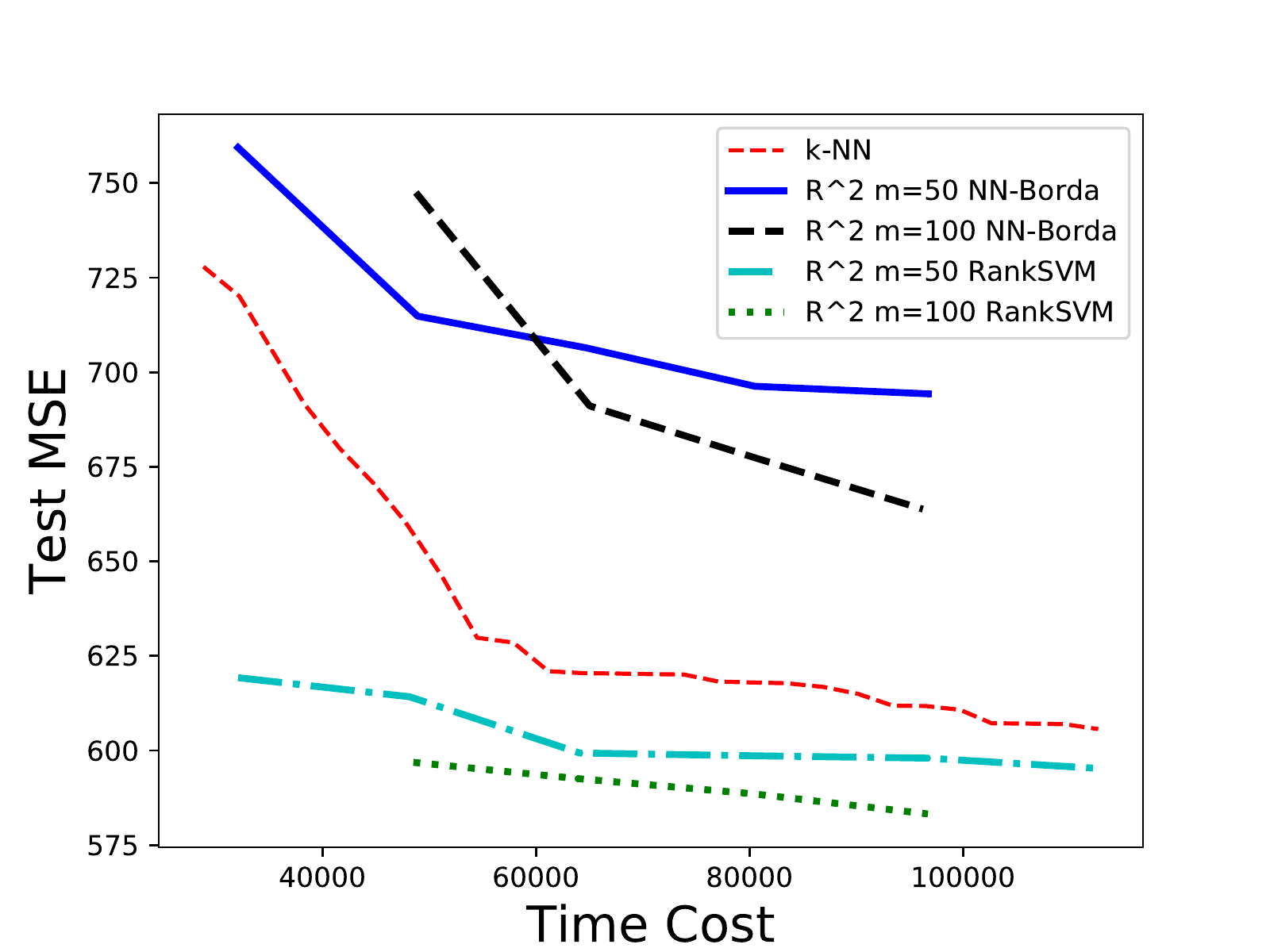}
		\caption{Task 2}
		\label{fig:res_NN_useravg_nnborda}	
	\end{subfigure}
	
	\caption{\label{fig:res_NN_AirBnB_nnborda}Experiments on AirBnB price estimation for nonparametric methods, using nearest neighbors of Borda count. Figure (a) uses 50-NN for averaging Borda counts, while Figure (b) uses 5-NN.
	}
	
\end{figure}

	\section{Detailed Proofs}
	\subsection{Proof of Theorem \ref{thm:mainupper}}
	
	Without loss of generality we assume throughout the proof that we re-arrange the samples so that the true ranking of the samples $\pi$ is the identity permutation, i.e. that $f(X_1) \leq f(X_2) \leq \ldots \leq f(X_n)$.
	We let $C, c, C_1, c_1, \ldots$ denote universal positive constants. As is standard in nonparametric regression
	these constants may depend on the dimension $d$ but we suppress this dependence.

	For a random point $X\in \mathcal{X}$, let $X_\alpha$ be the nearest neighbor of $X$ in the labeled set $\mathcal{L}$.
	We decompose the MSE as
	\begin{align*}
	\mathbb{E}\left[(\widehat{f}(X)-\regfunc(X) )^2\right]
	&\leq 2\mathbb{E}\left[(\widehat{f}(X)-\regfunc(X_\alpha) )^2\right]+2\mathbb{E}\left[(\regfunc(X_\alpha)-\regfunc(X))^2 \right].\\
	\end{align*}
	Under the assumptions of the theorem we have the following two results which provide bounds on the two terms
	in the above decomposition.
	\begin{lemma}
		\label{lem:aa}
		For a constant $C > 0$ we have that,
		\begin{align*}
		\mathbb{E}\left[(\regfunc(X_\alpha)-\regfunc(X))^2 \right] \leq C n^{-2s/d}.
		\end{align*}
	\end{lemma}
	\begin{lemma}
		\label{lem:siva_tech}
		For any $0 < \delta \leq 1/2$ we have that there is a constant $C >0$ such that:
		\begin{align*}
		&\mathbb{E}\left[(\widehat{f}(X)-\regfunc(X_\alpha) )^2\right] \leq 4 \delta M^2+\\
		& \frac{C  \log(m/\delta) \log n \log(1/\delta)}{m} \left[\sum_{k=1}^m \left(\mathbb{E}\left[(\widehat{y}_{t_k}-\regfunc(X_{t_k}) )^2\right]\right)+ \sum_{k=0}^m\left(\mathbb{E}\left[(\regfunc(X_{t_{k+1}})-\regfunc(X_{t_k}) )^2\right]\right)\right].
		\end{align*}

	\end{lemma}
	Taking these results as given we can now complete the proof of the theorem. We first note that the first term in the upper bound in Lemma~\ref{lem:siva_tech} is simply the MSE in an isotonic regression problem, and using standard risk bounds for isotonic regression (see for instance, Theorem~2.2 in~\citet{zhang2002risk}) we obtain that for a constant $C > 0$:
	\begin{align*}
	\sum_{k=1}^m \mathbb{E}\left[(\widehat{y}_{t_k}-\regfunc(X_{t_k}) )^2\right] \leq C m^{2/3}.
	\end{align*}
	Furthermore, since $f(X_{\labsam{m+1}})-f(X_{\labsam{0}}) \leq 2M$, and the function values are increasing we obtain that:
	\begin{align*}
	\sum_{k=0}^m\left(\mathbb{E}\left[(\regfunc(X_{t_{k+1}})-\regfunc(X_{t_k}) )^2\right]\right) \leq 4M^2.
	\end{align*}
	Now, choosing $\delta=\max\{n^{-2s/d},1/m\}$ we obtain:
	\begin{align*}
	\mathbb{E}\left[(\widehat{f}(X)-\regfunc(X) )^2\right]
	&\leq C_1 m^{-2/3}\log^2 n\log m+ C_2 n^{-2s/d},
	\end{align*}
	as desired. We now prove the two technical lemmas to complete the proof.

	\subsubsection{Proof of Lemma~\ref{lem:aa}}
	The proof of this result is an almost immediate consequence of the following result from~\cite{gyorfi2006distribution}.
	\begin{lemma}[\cite{gyorfi2006distribution}, Lemma 6.4 and Exercise 6.7]
		\label{lemma:cite_nn} Suppose that there exist positive
		constants $p_{\min}$ and $p_{max}$ such that
		$p_{\min} \leq p(x) \leq p_{\max}$. Then, there is a constant $c > 0$, such that
		\begin{align*}
		\mathbb{E}[\|X_\alpha-X\|_2^2]\leq cn^{-2/d}.
		\end{align*}
	\end{lemma}
	Using this result  and the H\"{o}lder condition we have
	\begin{align*}
	\mathbb{E}\left[(\regfunc(X_\alpha)-\regfunc(X))^2 \right]&\leq L\mathbb{E}\left[\|X_\alpha-X\|_2^{2s} \right]\\
	&\stackrel{\text{(i)}}{\leq} L\left(\mathbb{E}\left[\|X_\alpha-X\|_2^{2} \right]\right)^s\\
	&\leq c n^{-2s/d},
	\end{align*}
	where (i) uses Jensen's inequality. We now turn our attention to the remaining technical lemma.
	
	\subsubsection{Proof of Lemma~\ref{lem:siva_tech}}
	We condition on a certain favorable configuration of the samples that holds with high-probability.
	For a sample $\{X_1,\ldots,X_n\}$ let us denote by
	\begin{align*}
	q_i :=\mathbb{P}(X_\alpha=X_i),
	\end{align*}where $X_\alpha$ is the nearest neighbor of $X$. Furthermore, for each
	$k$ we recall that since we have re-arranged the samples so that $\pi$ is the identity permutation we can measure the distance between adjacent labeled samples in the ranking by $t_k - t_{k+1}.$
	The following result shows that the labeled samples are roughly uniformly spaced (up to a logarithmic factor) in the ranked sequence, and that each point $X_i$ is roughly equally likely (up to a logarithmic factor) to be the nearest neighbor of a randomly chosen point.
	\begin{lemma}
		\label{lem:cond_event}
		There is a constant $C > 0$ such that with probability at least $1 - \delta$ we have that the following two
		results hold:
		\begin{enumerate}
			\item We have that, 
			\begin{align}
			\label{eqn:tryone}
			\max_{1\leq j\leq n} q_i \leq \frac{C d\log(1/\delta)\log n}{n}.
			\end{align}
			\item Let us take $\labsam{m+1} :=n+1$, then
			\begin{align}
			\label{eqn:trytwo}
			\max_{k\in [m+1]} \labsam{k}-\labsam{k-1}\leq \frac{C n\log (m/\delta)}{m}.
			\end{align}
		\end{enumerate}
	\end{lemma}
	Denote the event, which holds with probability at least $1-\delta$ in the above Lemma by $\mathcal{E}_0$.
	By conditioning on $\mathcal{E}_0$ we obtain the following decomposition:
	\begin{align*}
	\mathbb{E}\left[(\widehat{f}(X)-\regfunc(X_\alpha) )^2\right]
	&\leq~ \mathbb{E}\left[(\widehat{f}(X)-\regfunc(X_\alpha) )^2|\mathcal{E}_0\right]+\delta\cdot 4M^2
	\end{align*}
	because both $f$ and $\widehat{f}$ are bounded in $[-M,M]$. 
	We condition all calculations below on $\mathcal{E}_0$. Now we have
	\begin{align}
	\mathbb{E}\left[(\widehat{f}(X)-\regfunc(X_\alpha) )^2 \vert \mathcal{E}_0\right]&= 
	\sum_{i=1}^n \mathbb{P}[X_\alpha=X_i \vert \mathcal{E}_0]\mathbb{E}\left[(\widehat{f}(X_i)-\regfunc(X_i) )^2 \vert \mathcal{E}_0\right] \nonumber\\
	&\leq  \sum_{i=1}^n \max_{1\leq j\leq n}\mathbb{P}[X_\alpha=X_j \vert \mathcal{E}_0]\mathbb{E}\left[(\widehat{y}_i-\regfunc(X_i) )^2 \vert \mathcal{E}_0\right] \nonumber \\
	&\leq \frac{C d\log(1/\delta)\log n}{n}\sum_{i=1}^n \mathbb{E}\left[(\widehat{y}_{\widetilde{i}}-\regfunc(X_i) )^2 \vert \mathcal{E}_0\right] \label{eqn:eightpm},
	\end{align}
	where we recall that $\widehat{y}_{\widetilde{i}}$ (defined in Algorithm~\ref{algo:compreg_iso}) denotes the de-noised (by isotonic regression) $y$ value at the nearest labeled left-neighbor of the point $X_i$.
	
	For convenience we define $f(X_{t_0}) = 0$ (equivalently as in Algorithm~\ref{algo:compreg_iso} we are assigning a 0 value to any point with no labeled point with a smaller function value according to the permutation $\widehat{\pi}$). With this definition in place we have that,
	\begin{align*}
	&\sum_{i=1}^n \mathbb{E}\left[(\widehat{y}_{\widetilde{i}}-\regfunc(X_i) )^2 \vert \mathcal{E}_0 \right]\\
	\leq& \sum_{i=1}^n \left(2\mathbb{E}\left[(\widehat{y}_{\widetilde{i}}-\regfunc(X_{\widetilde{i}}) )^2 \vert \mathcal{E}_0\right]+2\mathbb{E}\left[(\regfunc(X_i)-\regfunc(X_{\widetilde{i}}) )^2 \vert \mathcal{E}_0 \right]\right)\\
	=& \sum_{i=1}^n\sum_{k=0}^m \mathbb{I}[\widetilde{i}=t_k] \left(2\mathbb{E}\left[(\widehat{y}_{\labsam{k}}-\regfunc(X_{\labsam{k}}) )^2 \vert \mathcal{E}_0 \right]\mathbb{I}[k\ne 0]+2\mathbb{E}\left[(\regfunc(X_i)-\regfunc(X_{\labsam{k}}) )^2 \vert \mathcal{E}_0 \right]\right)\\
	\stackrel{\text{(i)}}{\leq}& \sum_{i=1}^n\sum_{k=0}^m \mathbb{I}[\widetilde{i}=t_k] \left(2\mathbb{E}\left[(\widehat{y}_{\labsam{k}}-\regfunc(X_{\labsam{k}}) )^2 \vert \mathcal{E}_0 \right]\mathbb{I}[k\ne 0]+2\mathbb{E}\left[(\regfunc(X_{\labsam{k+1}})-\regfunc(X_{\labsam{k}}) )^2 \vert \mathcal{E}_0 \right]\right)\\
	\stackrel{\text{(ii)}}{=}& \sum_{k=0}^m \left(2\mathbb{E}\left[(\widehat{y}_{\labsam{k}}-\regfunc(X_{\labsam{k}}) )^2 \vert \mathcal{E}_0 \right]\mathbb{I}[k\ne 0]+ 2\mathbb{E}\left[(\regfunc(X_{\labsam{k+1}})-\regfunc(X_{\labsam{k}}) )^2 \vert \mathcal{E}_0 \right]\right)\sum_{i=1}^n \mathbb{I}[\widetilde{i}=t_k] \\
	\stackrel{\text{(iii)}}{\leq} & \frac{Cn\log(m/\delta)}{m}\sum_{k=1}^m \left(2\mathbb{E}\left[(\widehat{y}_{t_k}-\regfunc(X_{t_k}) )^2 \vert \mathcal{E}_0 \right]\right)+ \frac{Cn\log(m/\delta)}{m}\sum_{k=0}^m\left(2\mathbb{E}\left[(\regfunc(X_{t_{k+1}})-\regfunc(X_{t_k}) )^2 \vert \mathcal{E}_0\right]\right) \\
	\leq& \frac{2 Cn\log(m/\delta)}{m} \left[\sum_{k=1}^m \left(\mathbb{E}\left[(\widehat{y}_{t_k}-\regfunc(X_{t_k}) )^2\vert \mathcal{E}_0\right]\right)+ \sum_{k=0}^m\left(\mathbb{E}\left[(\regfunc(X_{t_{k+1}})-\regfunc(X_{t_k}) )^2 \vert \mathcal{E}_0\right]\right)\right].
	\end{align*}
	The inequality (i) follows by noticing that if $\widetilde{i}=t_k$, $f(X_i)-f(X_{\widetilde{i}})$ is upper bounded by $f(X_{t_{k+1}})-f(X_{t_k})$. We interchange the order of summations to obtain the inequality~(ii). The inequality in~(iii) uses Lemma~\ref{lem:cond_event}. 
	We note that, 
	\begin{align*}
	\mathbb{E}\left[(\widehat{y}_{t_k}-\regfunc(X_{t_k}) )^2\vert \mathcal{E}_0\right] \leq \frac{\mathbb{E}\left[(\widehat{y}_{t_k}-\regfunc(X_{t_k}) )^2\right]}{\mathbb{P}(\mathcal{E}_0)} \leq 2 \mathbb{E}\left[(\widehat{y}_{t_k}-\regfunc(X_{t_k}) )^2\right],
	\end{align*}
	since $\delta \leq 1/2$, and a similar manipulation for the second term yields that,
	\begin{align*}
	\sum_{i=1}^n \mathbb{E}\left[(\widehat{y}_{\widetilde{i}}-\regfunc(X_i) )^2 \vert \mathcal{E}_0 \right]
	\leq&  \frac{4 Cn\log(m/\delta)}{m} \left[\sum_{k=1}^m \mathbb{E}\left[(\widehat{y}_{t_k}-\regfunc(X_{t_k}) )^2\right]+ \sum_{k=0}^m\mathbb{E}\left[(\regfunc(X_{t_{k+1}})-\regfunc(X_{t_k}) )^2 \right]\right].
	\end{align*}
	Plugging this expression back in to~\eqref{eqn:eightpm} we obtain the Lemma.		
	Thus, to complete the proof it only remains to establish the result in Lemma~\ref{lem:cond_event}.
	
	\subsubsection{Proof of Lemma~\ref{lem:cond_event}}
	We prove each of the two results in turn.
	
	\vspace{.2cm}
	
	\noindent{\bf Proof of Inequality~\eqref{eqn:tryone}: } As a preliminary, we need the following
	Vapnik-Cervonenkis result from \cite{chaudhuri2010rates}:
	\begin{lemma}\label{lemma:ball}
		Suppose we draw a sample $\{X_1,\ldots,X_n\}$, from a distribution $\mathbb{P}$, then 
		there exists a universal constant $C'$ such that with probability $1-\delta$, every ball $B$ with
		probability:
		\begin{align*}
		\mathbb{P}(B) \geq \frac{C'\log(1/\delta)d \log n}{n},
		\end{align*} 
		contains at least one of the sample points. 
	\end{lemma}
	We now show that under this event we have
	\[\max_i q_i \leq \frac{C' p_{\max}\log(1/\delta) d \log n}{p_{\min}n}. \]
	Fix any point $X_i\in T$, and for a new point $X$, let $r=\|X_i-X\|_2$. If $X_i$ is $X$'s nearest neighbor in $T$, there is no point in the ball $B(X,r)$. Comparing this with the event in Lemma \ref{lemma:ball} we have
	\[p_{\min}v_dr^d\leq \frac{C'\log(1/\delta)d \log n}{n},  \]
	where $v_d$ is the volume of the unit ball in $d$ dimension.
	
	Hence we obtain an upper bound on $r$. Now since $p(x)$ is upper and lower bounded we can bound the largest $q_i$ as
	\[\max_i q_i\leq p_{\max} v_dr^d\leq \frac{C' p_{\max}\log(1/\delta) d \log n}{p_{\min}n}. \]
	Thus we obtain the inequality~\eqref{eqn:tryone}.		
	
	\vspace{0.2cm}

	\noindent{\bf Proof of Inequality~\eqref{eqn:trytwo}: }
	Recall that we define $\labsam{m+1} := n+1.$
	Notice that $\labsam{1},\ldots,\labsam{m}$ are randomly chosen from $[n]$. So for each $k\in [m]$ we have
	\[\mathbb{P}[\labsam{k}-\labsam{k-1}\geq t]\leq \frac{n-t+1}{n}\left(\frac{n-t}{n}\right)^{m-1}\leq \left(\frac{n-t}{n}\right)^{m-1}, \]
	since we must randomly choose $\labsam{k}$ in $X_{t},X_{t+1},\ldots,X_{n}$, and choose the other $m-1$ samples in $X_1,\ldots,X_{\labsam{k}-t},X_{\labsam{k}+1},\ldots,X_n$. Similarly we also have
	\[\mathbb{P}[\labsam{m+1}-\labsam{m}\geq t]\leq \left(\frac{n-t}{n}\right)^{m-1}. \]
	So
	\begin{align*}
	\mathbb{P}[\max_{k\in [m+1]} \labsam{k}-\labsam{k-1}\geq t]&\leq \sum_{k=1}^{m+1}\mathbb{P}[\max_{k\in [m+1]} \labsam{k}-\labsam{k-1}\geq t] \\
	&\leq (m+1)\left(\frac{n-t}{n}\right)^{m-1}.
	\end{align*}
	Setting this to be less than or equal to $\delta$, we have
	\[\frac{t}{n}\geq {1-\left(\frac{\delta}{m+1}\right)^{\frac{1}{m-1}}}. \]
	Let $u=\log\left(1-\left(\frac{\delta}{m+1}\right)^{\frac{1}{m-1}}\right)=-C\frac{\log (m/\delta)}{m}$, we have $1-e^u=O(-u)$ since $u$ is small and bounded. So it suffices for
	\(t\geq C\frac{n\log (m/\delta)}{m} \)
	such that
	\[\mathbb{P}[\max_{k\in [m+1]} t_{k}-t_{k-1}\geq t]\leq \delta. \]
	
	\phantomsection
	\subsection{Proof of Theorem \ref{thm:lowerboundMSE} \label{sec:proof_lb_np_perfect}}
	To prove the result we separately establish lower bounds on the size of the labeled set of samples $m$ and the size of
	the ordered set of samples $n$. Concretely, we show the following pair of claims, for a positive constant $C > 0$,
	\begin{align}
	\label{eqn:lower_bound_m} 
	\inf_{\widehat{f}} \sup_{f \in \mathcal{F}_{s,L}}\mathbb{E}\left[(f(X)-\widehat{f}(X))^2\right]\geq Cm^{-2/3} \\
	\label{eqn:lower_bound_n}
	\inf_{\widehat{f}} \sup_{f \in \mathcal{F}_{s,L}}\mathbb{E}\left[(f(X)-\widehat{f}(X))^2\right]\geq Cn^{-2s/d}. 
	\end{align}
	We prove each of these claims in turn and note that together they establish Theorem~\ref{thm:lowerboundMSE}.
	
	\vspace{.2cm}
	
	\noindent {\bf Proof of Claim~\eqref{eqn:lower_bound_m}: } We establish the lower bound in this case by constructing a suitable packing set of functions, and using Fano's inequality. The main technical novelty, that allows us to deal with the setting where both direct and ordinal information is available, is that we construct functions that are all increasing functions of the first covariate $x_1$, and for these functions the ordinal measurements provide no additional information.
	
	Without loss of generality we consider the case when $d=1$, and note that the claim follows in general by simply extending our construction using functions for which $f(x) = f(x_1)$.	We take the covariate distribution $\mathbb{P}_{\mathcal{X}}$ to be uniform on $[0,1]$. 
	For a kernel function $K$ that is $1$-Lipschitz on $\mathbb{R}$, bounded and supported on $[-1/2,1/2]$, with 
	\begin{align*}
	\int_{-1/2}^{1/2} K^2(x) dx > 0
	\end{align*} we define:
	\begin{align*}
	u &=\lceil m^{\frac{1}{3}}\rceil, ~~~ h = 1/u, \\
	x_k &=\frac{k-1/2}{u}, ~~~ \phi_k(x) =\frac{Lh}{2} K\left(\frac{x-x_k}{h}\right),~~~\text{for}~~~k= \{1,2,\ldots,u\}, \\
	\Omega &=\{\omega:  \omega \in \{0,1\}^u \}.
	\end{align*}
	With these definitions in place we consider the following class of functions:
	\begin{align*}
	\mathcal{G}=\left\{f_\omega: [0,1] \mapsto \mathbb{R}: f_\omega(x) =\frac{Lx}{2}+\sum_{i=1}^k \omega_i \phi_i(x),~ \text{for}~x \in [0,1] \right\}.
	\end{align*}
	We note that the functions in $\mathcal{G}$ are $L$-Lipschitz, and thus satisfy the H\"{o}lder constraint (for any $0 < s \leq 1$).
	
	We note that these functions are all increasing so the permutation $\pi$ contains no additional information and can be obtained simply by sorting the samples (according to their first coordinate). 
	Furthermore, in this case the unlabeled samples contribute no additional information as their distribution $\mathbb{P}_{\mathcal{X}}$ is known (since we take it to be uniform).
	Concretely, for any estimator in our setup
	which uses $\{(X_1,y_1),\ldots,(X_m,y_m),X_{m+1},\ldots,X_n, \pi\}$ with
	\begin{align*}
	\sup_{f \in \mathcal{G}} \mathbb{E}\left[(f(X)-\widehat{f}(X))^2\right]<Cm^{-2/3},
	\end{align*}
	we can construct an equivalent estimator that uses only $\{(X_1,y_1),\ldots,(X_m,y_m)\}$. In particular, we can simply augment the sample by sampling $X_{m+1},\ldots,X_n$ uniformly on $[0,1]$ and generating $\pi$ by ranking $X$ in increasing order.
	
	In light of this observation, in order to complete the proof of Claim~\eqref{eqn:lower_bound_m} it sufficies to show that $Cm^{-2/3}$ is a lower bound for estimating functions in $\mathcal{G}$ with access to only noisy labels. For any pair $\omega, \omega'\in \Omega$ we have that,
	\begin{align*}
	\mathbb{E}[(f_\omega(X)-f_{\omega'}(X))^2]&=\sum_{k=1}^u (\omega_k-\omega_k')^2 \int \phi^2_k(x) dx\\
	&= L^2h^{3}\|K\|^2_2 \rho(\omega,\omega'),
	\end{align*}
	where $\rho(\omega,\omega')$ denotes the hamming distance between $x$ and $x'$. 
	
	Denote by $P_0$ the 
	distribution induced by the function $f_0$ with $\omega = (0, \ldots, 0)$, with the covariate distribution being uniform on $[0,1].$ 
	We can upper bound the KL divergence between the distribution induced by any function in $\mathcal{G}$ and $P_0$ as:
	\begin{align*}
	\text{KL}(P_j^m,P_0^m)&= m\int_{\mathcal{X}}p(x)\int_{\mathbb{R}} p_j(y|x)\log \frac{p_0(y|x)}{p_j(y|x)}dydx\\
	&= m\int_{\mathcal{X}} p(x) \sum_{i=1}^u\omega^{j}_i \phi^2_i(x)dx\\
	&\leq m L^2h^{3}\|K\|^2_2 u.
	\end{align*}
	
	Now, the Gilbert-Varshamov bound~\citep{varshamov1957estimate,gilbert1952comparison}, ensures that if $u > 8$, then there is a subset 
	$\Omega' \subseteq \Omega$ of cardinality $2^{u/8}$, such that the Hamming distance between each pair of elements $\omega,\omega' \in \Omega'$ is at least $u/8$. A straightforward application of Fano's inequality (see for instance Theorem 2.5 in~\cite{tsybakov2009introduction}) shows that for small constants $c, c' > 0$, if:
	\begin{align*}
	m L^2h^{3}\|K\|^2_2 u \leq c u,
	\end{align*}
	then the error of any estimator is lower bounded as:
	\begin{align*}
	\sup_{f \in \mathcal{G}} \mathbb{E}[(\widehat{f}(X)-f(X))^2]&\geq c' L^2h^{3}\|K\|^2_2 u \geq c' m^{-2/3},
	\end{align*}
	establishing the desired claim.

	
	\vspace{.2cm}
	
	\noindent {\bf Proof of Claim~\eqref{eqn:lower_bound_n}: } We show this claim by reducing to the case where we have $n$ points with noiseless evaluations, i.e., we observe $\{(X_1,f(X_1)), (X_2,f(X_2)),\ldots,(X_n, f(X_n))\}$.
	We notice that the ranking $\pi$ provides no additional information when all the points are labeled without noise. Formally, if we have an estimator $\widehat{f}$ which when provided $\{(X_1,y_1),\ldots,(X_m,y_m),X_{m+1},\ldots,X_n, \pi\}$ obtains an error less than $c n^{-2s/d}$ (for some sufficiently small constant $c > 0$) then we can also use this estimator in setting where all $n$ samples are labeled without noise, by generating $\pi$ according to the noiseless labels, adding Gaussian noise to the labels of $m$ points, and eliminating the remaining labels before using the estimator $\widehat{f}$ on this modified sample.
	
	It remains to show that, $c n^{-2s/d}$ is a lower bound on the MSE of any estimator that receives $n$ noiseless labels. To simplify our notation we will assume that $(2n)^{1/d}$ is an integer (if not we can establish the lower bound for a larger sample-size for which this condition is true, and conclude the desired lower bound with an adjustment of various constants).
	For a given sample size $n$, we choose 
	\begin{align*}
	h &= (2n)^{-1/d},
	\end{align*}
	and consider the grid with $2n$ cubes with side-length $h$. Denote the centers of the cubes as $\{x_1,\ldots,x_{2n}\}$. 
	For a kernel function $K$ supported on $[-1/2,1/2]^d$, which is 1-Lipschitz on $\mathbb{R}^d$, bounded and satisfies:
	\begin{align*}
	\int_{[-1/2,1/2]^d} K^2(x) dx > 0
	\end{align*} 
	we 
	consider a class of functions:
	\begin{align*}
	\mathcal{G} = \left\{f_\omega: f_\omega(x) = Lh^s \sum_{i=1}^{2n} \omega_i K\left(\frac{x-x_i}{h}\right), \text{for}~\omega \in \{0,1\}^{2n} \right\}.
	\end{align*}
	We note that these functions are all in the desired H\"{o}lder class with exponent $s$. Given $n$ samples (these may be arbitrarily distributed) $\{(X_1,f(X_1)),\ldots,(X_n,f(X_n))\}$, we notice that we are only able to identify at most $n$ of the $\omega_i$ (while leaving at least $n$ of the $\omega_i$ completely unconstrained) and thus any estimator $\widehat{f}$
	must incur a large error on at least one of the functions $f_{\omega}$ consistent with the obtained samples. Formally, we have that
	\begin{align*}
	\sup_{f \in \mathcal{G}} \mathbb{E}(\widehat{f}(X) - f(X))^2 \geq \frac{n L^2 h^{2s + d} \|K\|_2^2}{4} \geq c n^{-2s/d},
	\end{align*}
	as desired. This completes the proof of Claim~\eqref{eqn:lower_bound_n}.

	\subsection{Proof of Theorem \ref{thm:agn_upper}}
	Throughout this proof without loss of generality we re-arrange the samples so that the estimated permutation $\widehat{\pi}$ is the identity permutation. 
	To simplify the notation further, we let $X_{(i)}=X_{\pi^{-1}(i)}$ be the $i$-th element according to true permutation $\pi$. This leads to $f(X_{(1)})\leq f(X_{(2)})\leq\cdots\leq f(X_{(n)})$. 
	
	We begin with a technical result that bounds the error of using $\widehat{\pi}$ instead of the true permutation $\pi$ in the \RR algorithm.

	
	The first part of the proof is the same as that of Theorem \ref{thm:mainupper}. We have
	\begin{align*}
	\mathbb{E}\left[(\widehat{f}(X)-\regfunc(X) )^2\right]
	&\leq 2\mathbb{E}\left[(\widehat{f}(X)-\regfunc(X_\alpha) )^2\right]+2\mathbb{E}\left[(\regfunc(X_\alpha)-\regfunc(X))^2 \right]\\
	&\leq 2\mathbb{E}\left[(\widehat{f}(X)-\regfunc(X_\alpha) )^2\right]+ Cn^{-2s/d}.
	\end{align*}
	And for event $\mathcal{E}_0$ we have (note that $\widetilde{i}$ is the nearest neighbor index defined in Algorithm \ref{algo:compreg_iso})
	\begin{align}
	\mathbb{E}\left[(\widehat{f}(X)-\regfunc(X_\alpha) )^2\right]\nonumber
	&\leq C\frac{d\log(1/\delta)\log n}{n}\sum_{i=1}^n \mathbb{E}\left[(\widehat{y}_{\widetilde{i}}-\regfunc(X_i) )^2|\mathcal{E}_0\right]+\delta\nonumber\\
	&\leq C\left(\frac{\log^2 n}{n}\sum_{i=1}^n \mathbb{E}\left[(\widehat{y}_{\widetilde{i}}-\regfunc(X_i) )^2|\mathcal{E}_0\right]+n^{-2s/d}\right).\label{eqn:total_agn}
	\end{align}
	The second inequality is obtained by letting $\delta=n^{-2s/d}$.
	To bound the sum of expectations above, we first prove a lemma bounding the difference between $X_1,\ldots,X_n$ and $X_{(1)},\ldots,X_{(n)}$:\\
	
	\textbf{Lemma \ref{lemma:rankerr_to_valueerr}(Restated).} Suppose the ranking $(X_1,\ldots,X_n)$ is of at most $\nu$ error with respect to the true permutation. Then
	\[\sum_{i=1}^n (f(X_i)-f(X_{(i)})^2\leq 8M^2\sqrt{2\nu}n. \]
	\begin{proof}
		Let $\theta_i=f(X_i)$ and $\theta_{(i)}=f(X_{(i)})$, and let $g(\theta_1,\ldots,\theta_n)=\sum_{i=1}^n (\theta_i-\theta_{(i)})^2$. Consider $g$ as a function of $\theta_i$; $\theta_i$ appears twice in $g$, one as $(\theta_i-\theta_{(i)})^2$, and the other as $(\theta_{\pi(i)}-\theta_{(\pi(i))})^2=(\theta_{\pi(i)}-\theta_{i})^2$. If $\pi(i)=i$, then $\theta_i$ does not influence value of $g$; otherwise, $g$ is a quadratic function of $\theta_i$, and it achieves maximum either when $\theta_i=M$ or $\theta_i=-M$. So when $g$ achieves maximum it must be $\theta_i\in \{-M,M\}$. Now notice that $\theta_{(1)}\leq \cdots \leq \theta_{(n)}$, so the maximum is achieved when for some $0\leq k\leq n$ that $\theta_{(i)}=-M$ for $i\leq k$, and $\theta_{(i)}=M$ for $i> k$.

		Note that $\sum_{i=1}^n (\theta_i-\theta_{(i)})^2=\sum_{i=1}^n (\theta_{\pi^{-1}(i)}-\theta_{(\pi^{-1}(i))})^2=\sum_{i=1}^n (\theta_{(i)}-\theta_{(\pi^{-1}(i))})^2 $. From the discussion above, in the maximum case $(\theta_{(i)}-\theta_{(\pi^{-1}(i))})^2=4M^2$ 
		iff $i$ and $\pi^{-1}(i)$ lies on different sides of $k$, and otherwise it is 0. To further bound the sum, we use the Spearman Footrule distance between $\pi$ and $(1,2,\ldots,n)$, which \cite{diaconis1977spearman} shows that it can be bounded as
		\[\sum_{i=1} |\pi(i)-i|\leq 2\sum_{1\leq i,j\leq n} I\left[(\pi(i)-\pi(j))(i-j)<0\right]. \]
		And the RHS can be bounded by $2\nu n^2$ since the agnostic error of ranking is at most $\nu$. We also have that
		\begin{align*}
		\sum_{i=1} |\pi(i)-i|=\sum_{i=1} |\pi(\pi^{-1}(i))-\pi^{-1}(i)|=\sum_{i=1} |i-\pi^{-1}(i)|.  
		\end{align*}
		Let $U_1=\{i:\pi^{-1}(i)\leq k, i>k \}$ and $U_2=\{\pi^{-1}(i)>k,i\leq k \}$. So in the maximum case we have
		\[\sum_{i=1}^n (\theta_{\pi(i)}-\theta_i)^2=4M^2(|U_1|+|U_2|). \]
		Now notice that for $i\in U_1$, we have $|\pi^{-1}(i)-i|\geq i-k$; and for $i\in U_2$ we have$|\pi^{-1}(i)-i|\geq k-i+1$. Considering the range of $i$ we have
		\[\sum_{j=1}^{|U_1|} j+\sum_{j=1}^{|U_2|} j\leq \sum_{i=1}^n |\pi^{-1}(i)-i|\leq 2\nu n^2. \]
		So $|U_1|+|U_2|\leq 2\sqrt{2\nu}n$. And
		\[\sum_{i=1}^n (f(X_i)-f(X_{(i)})^2=\sum_{i=1}^n (\theta_{\pi(i)}-\theta_i)^2\leq 4M^2(|U_1|+|U_2|)\leq 8M^2\sqrt{2\nu}n. \]
		Thus we prove the lemma.
	\end{proof}
	
	Now back to the original proof. Under event $\mathcal{E}_0$ we have
	\begin{align}
	&\sum_{i=1}^n E[(\widehat{y}_i-f(X_i))^2|\mathcal{E}_0]\nonumber\\
	=& \sum_{i=1}^n E[(\widehat{y}_{\widetilde{i}}-f(X_i))^2|\mathcal{E}_0]\nonumber\\
	\leq& \sum_{i=1}^n \mathbb{E}\left[2(\widehat{y}_{\widetilde{i}}-f(X_{\widetilde{i}}))^2+2(f(X_{\widetilde{i}})-f(X_i))^2|\mathcal{E}_0\right]\nonumber\\
	\leq&  \frac{Cn\log(m/\delta)}{m}\sum_{k=1}^m \mathbb{E}\left[(\widehat{y}_{t_k}-f(X_{t_k}))^2 |\mathcal{E}_0\right]+2\sum_{i=1}^n \mathbb{E}\left[(f(X_{\widetilde{i}})-f(X_i))^2|\mathcal{E}_0\right].\label{eqn:two_terms}
	\end{align}
	We omit the condition $\mathcal{E}_0$ in discussion below. We bound the two terms separately. For the first term, we use the following theorem adapted from \cite{zhang2002risk}:
	\begin{theorem}[\cite{zhang2002risk}, adapted from Theorem 2.3 and Remark 4.2]
		\label{thm:zhang_iso_agn}
		Suppose $X_{t_k},y_{t_k}$ are fixed for $k\in [m]$, and $f(X_{t_k})$ is arbitrary in order. Let 
		\[S=\min_{u} \sum_{k=1}^m (u_k-f(X_{t_k}))^2,  \]
		where the minimum is taken over all sequence of $u\in \mathbb{R}^m$ that is non-decreasing.
		The risk of isotonic regression satisfies
		\[\frac{1}{m^{2/3}M^{1/3}}\left( \mathbb{E}\left[\sum_{k=1}^m(\widehat{y}_{t_k}-f(X_{t_k}))^2\right]-S\right)\leq C \]
		for some universal constant $C$.
	\end{theorem}
	So from Theorem \ref{thm:zhang_iso_agn} we know that
	\[ \mathbb{E}\left[(\widehat{y}_{t_k}-f(X_{t_k}))^2\right]\leq\mathbb{E}[S]+Cm^{1/3},\]
	where the expectation in $E[S]$ is taken w.r.t. the randomness in $t_k$. From Lemma \ref{lemma:rankerr_to_valueerr} we know that
	\[\sum_{i=1}^n (f(X_i)-f(X_{(i)}))^2\leq C\sqrt{\nu}n. \]
	Note that since $t_k$ is taken at random, each element $X_i$ has equal probability $\frac{m}{n}$ to be picked; so
	\[\mathbb{E}[S]\leq \mathbb{E}\left[\sum_{i=1}^m (f(X_{(t_k)})-f(X_{t_k}))^2\right]\leq C\sqrt{\nu}m. \]
	
	Now we bound the second term in (\ref{eqn:two_terms}). We have
	\begin{align*}
	&\sum_{i=1}^n \mathbb{E}\left[(f(X_{\widetilde{i}})-f(X_i))^2\right]\\		
	\leq&\; 3\sum_{i=1}^n \mathbb{E}\left[(f(X_{\widetilde{i}})-f(X_{(\widetilde{i})}))^2\right]+3\sum_{i=1}^n \mathbb{E}\left[(f(X_{(\widetilde{i})})-f(X_{(i)}))^2\right]+3\sum_{i=1}^n \mathbb{E}\left[(f(X_{(i)})-f(X_{i}))^2\right]\\
	\leq&\; \frac{Cn\log m}{m} \sum_{k=1}^m \mathbb{E}\left[(f(X_{t_k})-f(X_{(t_k)}))^2\right]+\frac{Cn\log m}{m} \sum_{k=1}^m \mathbb{E}\left[(f(X_{(t_{k+1})})-f(X_{(t_k)}))^2\right]\\
	&\;~~~~~~~~~~~~~~+3\sum_{i=1}^n \mathbb{E}\left[(f(X_{(i)})-f(X_{i}))^2\right]\\
	\leq&~  \frac{Cn\log m}{m} \sqrt{\nu} m +\frac{Cn\log m}{m}\cdot 1+C\sqrt{\nu} n\\
	=&~ C\sqrt{\nu} n\log m.	
	\end{align*}
	The first inequality is by noticing $(x+y+z)^2\leq 3x^2+3y^2+3z^2$ for any number $x,y,z\in \mathbb{R}$; the second inequality is by grouping values of $\widetilde{i}$, and the choice of $t_k$; the third inequality comes from analysis of the first term on $\sum_{k=1}^m \mathbb{E}\left[(f(X_{t_k})-f(X_{(t_k)}))^2\right]$, the fact that $f(X_{(t_m)})-f(X_{(t_1)})\leq 1$, and Lemma \ref{lemma:rankerr_to_valueerr}.
	
	Summarizing the two terms we have
	\[\mathbb{E}\left[(\widehat{y}_{\widetilde{i}}-\regfunc(X_i) )^2|\mathcal{E}_0\right]\leq C(\sqrt{\nu}n+m^{-2/3}n)\log m. \]
	Take this back to (\ref{eqn:total_agn}) we prove the theorem.

	\subsection{Proof of Theorem \ref{thm:upper_cv}}

	To simplify notation, we suppose that we have $m$ labeled samples $T=\{(X_i,y_i)\}_{i=1}^m$ for training and another $m$ labeled samples $V=\{(X_i,y_i)\}_{i=m+1}^{2m}$ for validation. 
	We consider the following models:
	\begin{enumerate}
		\item \RR using both the ordinal data and the labeled samples in $T$, and we denote by $\widehat{f}_0$,
		\item $k$-NN regression using only the labeled samples in $T$ for $k \in [m]$ which we denote by $\{\widehat{f}_1,\ldots,\widehat{f}_m\}$.
	\end{enumerate}
	We select the best model according to performance on validation set.
	We further restrict all estimators to be bounded in $[-M,M]$; i.e., when $\widehat{f}_j(x)<-M$ for some $x$ and $j$, we clip its value by setting $\widehat{f}_j(x)=-M$ and we analogously clip the function when it exceeds $M$. We note in passing that this only reduces the MSE since the true function $f$ is bounded between $[-M,M]$. 
	Throughout the remainder of the proof we condition on the training set $T$ but suppress this in our notation. 
	We define the empirical validation risk of a function $\widehat{f}$ to be:
	\begin{align*}
	\widehat{R}^V\big(\widehat{f}\big)=\frac{1}{m}\sum_{i=m+1}^{2m} \big(y_i-\widehat{f}(X_i) \big)^2, 
	\end{align*}
	and the 
	population MSE of a function $\widehat{f}$ to be 
	\begin{align*}
	\err(\widehat{f})=\mathbb{E}\big[\big(\widehat{f}(X)-f(X)\big)^2 \big],
	\end{align*}
	where $\epsilon \sim N(0,1)$ denotes the noise in the direct measurements.
	Now let 
	\begin{align*}
	\widehat{f}^*=\argmin_{j=0,\ldots,m} \widehat{R}^V(\widehat{f}_j),
	\end{align*}
	be the best model selected 
	using cross validation and 
	\begin{align*}
	f^*=\argmin_{j=0,\ldots,m} \err(\widehat{f}_j),
	\end{align*}
	be the estimate with lowest MSE in $\widehat{f}_0,\ldots,\widehat{f}_m$. 
	Let us denote:
	\begin{align*}
	\mathcal{G} = \{\widehat{f}_0,\ldots,\widehat{f}_m\}.
	\end{align*}
	Recall that $f$ denotes the true unknown regression function. Then in the sequel we show the following result:
	\begin{lemma}
		\label{lem:main_cv}
		With probability at least $1 - \delta$, for any $\widehat{f} \in \mathcal{G}$ we have that the following hold for some constant $C > 0$,
		\begin{align}
		\label{eqn:cvone}
		\err(\widehat{f}) \leq 2\left[ \widehat{R}^V(\widehat{f}) - \widehat{R}^V(f) \right]+ \frac{C \log (m/\delta)}{m}, \\
		\label{eqn:cvtwo}
		\left[ \widehat{R}^V(\widehat{f}) - \widehat{R}^V(f) \right] \leq 2 \err(\widehat{f}) + \frac{C \log (m/\delta)}{m}.
		\end{align}
	\end{lemma}
	Since $\widehat{R}^V(\widehat{f}^*) \leq \widehat{R}^V(f^*)$
	we obtain using~\eqref{eqn:cvone} that with probability at least $1 - \delta$,
	\begin{align*}
	\err(\widehat{f}^*) \leq 2\left[\widehat{R}^V(f^*) - \widehat{R}^V(f) \right]+ \frac{C \log (m/\delta)}{m}.
	\end{align*}
	Since $f^* \in \mathcal{G}$, we can use~\eqref{eqn:cvtwo} to obtain that with probability $1 - \delta$,
	\begin{align*}
	\err(\widehat{f}^*) \leq 4 \err(f^*) + \frac{2 C\log (m/\delta)}{m}. 
	\end{align*}
	Since $\err(\widehat{f}^*)$ is a positive random variable, integrating this bound we obtain that,
	\begin{align*}
	\mathbb{E}[\err(\widehat{f}^*)] &= \int_{0}^\infty \mathbb{P}(\err(\widehat{f}^*) \geq t) dt, \\
	&\leq 4\err(f^*) + \frac{6C \log(m)}{m}. 
	\end{align*}
	So far we have implicitly conditioned throughout on the training set $T$. Taking an expectation over the training set yields:
	\begin{align*}
	\mathbb{E}[\err(\widehat{f}^*)] \leq 4\mathbb{E}[\err(f^*)]+ \frac{6C \log(m)}{m}. 
	\end{align*}
	We now note that,
	\begin{align*}
	\mathbb{E}[\err(f^*)]] \leq \min_{j \in \{0,\ldots,m\}}\mathbb{E}[\err(\widehat{f}_j)].
	\end{align*}
	Standard results on $k$-NN regression (for instance, a straightforward modification of Theorem~6.2 in \cite{gyorfi2006distribution} to deal with $0 < s \leq 1$ in the H\"{o}lder class) yield that for a constant $C > 0$,
	\begin{align*}
	\min_{j \in \{1,\ldots,m\}} \mathbb{E}[\err(\widehat{f}_j)] \leq C m^{-2s/(2s + d)}.
	\end{align*}
	Theorem~\ref{thm:agn_upper} yields that,
	\begin{align*}
	\mathbb{E}[\err(\widehat{f}_0)] \leq C_1\left(\log^2 n\log m \left(m^{-2/3}+\sqrt{\nu}\right)\right) + C_2 n^{-2s/d},
	\end{align*}
	and putting these together we obtain that,
	\begin{align*}
	\mathbb{E}[\err(f^*)]  \leq \widetilde{O}\left(m^{-2/3}+\min\{\sqrt{\nu},m^{-\frac{2s}{2s+d}} \} + n^{-2s/d}\right),
	\end{align*}
	and thus it only remains to prove Lemma~\ref{lem:main_cv} to complete the proof of the theorem.
	
	\subsubsection{Proof of Lemma~\ref{lem:main_cv}}
	For a fixed classifier $\widehat{f}$ and for samples in the validation set $i \in \{m+1,\ldots,2m\}$ we define the random variables:
	\begin{align*}
	Z_i = (y_i - \widehat{f}(X_i))^2 - (y_i - f(X_i))^2 = (\widehat{f}(X_i) - f(X_i))^2 + 2 \epsilon_i (f(X_i) - \widehat{f}(X_i)),
	\end{align*}
	and note that $\mathbb{E}[Z_i] = \err(\widehat{f}).$ 
	In order to obtain tail bounds on the average of the $Z_i$ let us bound the absolute central moments of $Z_i$. 
	Using the inequality that $(x+y)^k \leq 2^{k-1} (x^k + y^k)$, for $k > 2$ we obtain that,
	\begin{align}
	\mathbb{E}|Z_i - \mathbb{E}[Z_i]|^k &= \mathbb{E}| (\widehat{f}(X_i) - f(X_i))^2 + 2 \epsilon_i (f(X_i) - \widehat{f}(X_i)) -  \err(\widehat{f})|^k \nonumber \\
	&\leq 2^{k-1} \mathbb{E}| (\widehat{f}(X_i) - f(X_i))^2 - \err(\widehat{f})|^k + 2^{k-1} \mathbb{E}|\epsilon_i (f(X_i) - \widehat{f}(X_i))|^k. \label{eqn:moment_bound}
	\end{align}
	We bound each of these terms in turn. Since $(\widehat{f}(X_i) - f(X_i))^2 \in [0,4M^2]$, we obtain that,
	\begin{align*}
	\mathbb{E}| (\widehat{f}(X_i) - f(X_i))^2 - \err(\widehat{f})|^k \leq \text{var}((\widehat{f}(X_i) - f(X_i))^2) (4M^2)^{k-2},
	\end{align*}
	and using the fact that $\epsilon_i$ are Gaussian we obtain that,
	\begin{align*}
	\mathbb{E}|\epsilon_i (f(X_i) - \widehat{f}(X_i))|^k &\leq \mathbb{E}|\epsilon_i|^{k-2} \mathbb{E}(f(X_i) - \widehat{f}(X_i))^2 (2M)^{k-2} \\
	&\leq \text{var}(\epsilon_i (f(X_i) - \widehat{f}(X_i))) k! \times (2M)^{k-2}.
	\end{align*}
	Since $\epsilon_i$ is independent of the other terms in $Z_i$ we have that,
	\begin{align*}
	\text{var}(Z_i) = \text{var}(\epsilon_i (f(X_i) - \widehat{f}(X_i))) +  \text{var}((\widehat{f}(X_i) - f(X_i))^2).
	\end{align*}
	Putting these pieces together with~\eqref{eqn:moment_bound} we obtain,
	\begin{align*}
	\mathbb{E}|Z_i - \mathbb{E}[Z_i]|^k &\leq 2^{k-1} \text{var}(Z_i) \left[k! \times (2M)^{k-2} + (4M^2)^{k-2}\right] \\
	&\leq \frac{ \text{var}(Z_i) }{2} k! (16M + 32M^2)^{k-2}.
	\end{align*}
	Let us denote $r := 16M + 32M^2$. It remains to bound the variance. We have that,
	\begin{align*}
	\text{var}(Z_i) \leq \mathbb{E}[Z_i^2] \leq 2\mathbb{E}( (\widehat{f}(X_i) - f(X_i))^4) + 8 \mathbb{E}(f(X_i) - \widehat{f}(X_i))^2,
	\end{align*}
	and using the fact that the functions are bounded in $[-M,M]$ we obtain that, 
	\begin{align}
	\label{eqn:var_bound}
	\text{var}(Z_i) \leq (8M^2 + 8) \err(\widehat{f}).
	\end{align}
	Now, applying the inequality in Lemma~\ref{lem:bernstein}, we obtain that for any $c < 1$ and for any $t \leq c/r$ that,
	\begin{align*}
	\err(\widehat{f}) \leq \frac{1}{m} \sum_{i=m+1}^{2m} Z_i +  \frac{\log(1/\delta)}{m t} + \frac{8t(M^2 + 1) \err(\widehat{f})}{2(1-c)},
	\end{align*}
	we choose $c = 1/2$ and $t = \min\{1/(2r), 1/(16(M^2 + 1))\}$, and rearrange to obtain that,
	\begin{align*}
	\err(\widehat{f}) \leq  \frac{2}{m} \sum_{i=m+1}^{2m} Z_i + \frac{2\log(1/\delta)}{m} \max\{2r, 16(M^2 + 1)\},
	&\leq  \frac{2}{m} \sum_{i=m+1}^{2m} Z_i + \frac{C \log(1/\delta)}{m},
	\end{align*}
	and using a union bound over the $m + 1$ functions $\widehat{f} \in \mathcal{G}$ we obtain~\eqref{eqn:cvone}. Repeating this argument with the random variables $-Z_i$ we obtain~\eqref{eqn:cvtwo} completing the proof of the Lemma.

	\subsection{Proof of Theorem \ref{thm:lowerboundMSE_agn}}
	We prove a slightly stronger result, and show Theorem \ref{thm:lowerboundMSE_agn} as a corollary.
	
	\begin{theorem}
		\label{thm:lowerboundMSE_agn_strong}
		Assume the same modeling assumptions for $X_1,\ldots,X_n, y_1,\ldots,y_m$ as in Theorem \ref{thm:lowerboundMSE}. Also permutation $\widehat{\pi}$ satisfies $\mathbb{P}[(f(X_i)-f(X_j))(\pi(i)-\pi(j))<0]\leq \nu$.
		Then for any estimator $\widehat{f}$ taking input $X_1,\ldots,X_n, y_1,\ldots,y_m$ and $\widehat{\pi}$, we have
		\begin{align*}
		\inf_{\widehat{f}} \sup_{f \in \mathcal{F}_{s,L}} \mathbb{E}\big(f(X)-\widehat{f}(X)\big)^2\geq C(m^{-\frac{2}{3}}+\min\{\nu^{\frac{d+2}{2d}}m^{\frac{1}{2d}},1\}m^{-\frac{2}{d+2}}+n^{-2s/d}).	
		\end{align*}
		
	\end{theorem}

	\begin{proofarg}{Proof of Theorem \ref{thm:lowerboundMSE_agn_strong}}
		In this proof, we use $x_i$ to represent $i$-th dimension of $x$, and upper script for different vectors $x^{(1)},x^{(2)},\ldots$. Let $u=\lceil m^{\frac{1}{2+d}}\rceil, h=1/u$, and $t=\min\left\{\left(\nu m^{\frac{1}{2+d}}\right)^{\frac{1}{2d}},1 \right\}$. Let $\Gamma=\{(\gamma_1,\ldots,\gamma_d), \gamma_i\in \{1,2,\ldots,u\} \}$. Choose an arbitrary order on $\Gamma$ to be $\Gamma=\{\gamma^{(1)},\ldots,\gamma^{(u^d)} \}$. Let $x^{(k)}=\frac{\gamma^{(k)}-1/2}{u}$, and $\phi_k(x)=\frac{L}{2}thK(\frac{x-tx^{(k)}}{th}), k=1,2,\ldots,u^d$, where $K$ is a kernel function in $d$ dimension supported on $[-1/2,1/2]^d$, i.e., $\int K(x)dx$ and $\max_x K(x)$ are both bounded, $K$ is 1-Lipschitz. So $\phi_k(x)$ is supported on $[thx^{(k)}-1/2th, thx^{(k)}+1/2th]$. Let $\Omega=\{\omega=(\omega_1,\ldots,\omega_{u^d}), \omega_i\in \{0,1\} \}$, and
		\[\mathcal{E}= \left\{f_\omega(x)=\frac{L}{2}x_1+\sum_{i=1}^k \omega_i \phi_k(x), x\in [0,1]^d \right\}.\]

		Functions in $\mathcal{E}$ are $L$-Lipschitz. The function value is linear in $x_1$ for $x\not\in [0,t]^d$ in all functions in $\mathcal{E}$. Consider the comparison function $Z(x,x')=I(x_1<x_1')$ that ranks $x$ according to the first dimension. Since $K$ is 1-Lipschitz, it only makes an error when both $x,x'$ lies in $[0,t]^d$, and both $x_1,x_1'$ lie in the same grid segment $[tk/u,t(k+1)/u]$ for some $k\in [u]$.
		So the error is at most $t^{2d}(1/u)^2\cdot u\leq \nu$ for any function $f\in \mathcal{E}$. 
		Thus, if there exists one estimator with $\sup_f \mathbb{E}[ (f-\widehat{f})^2]<C\min\{\nu^{\frac{d+2}{2d}}m^{\frac{1}{2d}},1\}m^{-\frac{2}{d+2}}$, then we can obtain one estimator for functions in $\mathcal{E}$ by using $\widehat{f}$ on $\mathcal{E}$, and responding to all comparisons and rankings as $Z(x,x')=I(x_1<x_1')$. So a lower bound on learning $\mathcal{E}$ is also a lower bound on learning any $f\in \mathcal{F}_{s,L}$ with $\nu$-agnostic comparisons.
		
		Now we show that $Ct^{d+2}h^{2}=C\min\{\nu^{\frac{d+2}{2d}}m^{\frac{1}{2d}},1\}m^{-\frac{2}{d+2}}$ is a lower bound to approximate functions in $\mathcal{E}$. For all $\omega, \omega'\in \Omega$ we have
		\begin{align*}
		\mathbb{E}[(f_\omega-f_{\omega'})^2]^{1/2}&=\left(\sum_{k=1}^{p^d} (\omega_k-\omega_k')^2\int \phi^2_k(x)dx\right)^{1/2}\\
		&=\left(\rho(\omega, \omega')L^2 t^{d+2} h^{d+2} \right)^{1/2}\\
		&=L(th)^{\frac{d+2}{2}} \|K\|_2\sqrt{\rho(\omega,\omega')},
		\end{align*}
		where $\rho(\omega,\omega')$ denotes the Hamming distance between $x$ and $x'$.
		
		By the Varshamov-Gilbert lemma, we can have a $M=O(2^{u^d/8})$ subset $\Omega'=\{\omega^{(0)},\omega^{(1)},\ldots,$\\$\omega^{(M)} \}$ of $\Omega$ such that the distance between each element $\omega^{(i)},\omega^{(j)}$ is at least $u^d/8$. So $d(\theta_i,\theta_j)\geq h^st^{(d+2)/2}$. Now for $P_j,P_0$ ($P_0$ corresponds to $f_\omega$ when $\omega=(0,0,\ldots,0)$) we have
		\begin{align*}
		KL(P_j,P_0)&= m\int_{\mathcal{X}}p(x)\int_{\mathcal{y}} p_j(y|x)\log \frac{p_0(y|x)}{p_j(y|x)}dydx\\
		&=m\int_{\mathcal{X}} p(x) \sum_{i=1}^{u^d}\omega^{(j)}_i \phi^2_{\omega^{(j)}}(x)\\
		&\leq m\cdot Ch^{d+2}t^{d+2}u^d=Cu^dt^{d+2}.
		\end{align*}
		We have $Cu^dt^{d+2}\leq cu^d\leq \alpha \log M$ (since $t\leq 1$), so again using Theorem 2.5 in \cite{tsybakov2009introduction} we obtain a lower bound of $d(\theta_i,\theta_j)^2=Ch^{2}t^{d+2}=\min\{\nu^{\frac{d+2}{2d}}m^{\frac{1}{2d}},1\}m^{-\frac{2}{d+2}}$.	
	\end{proofarg}
	Now we can prove Theorem \ref{thm:lowerboundMSE_agn}.
	\begin{proof}[Proof of Theorem \ref{thm:lowerboundMSE_agn}]
		We only need to show
		\begin{align}
		\min\{\nu^{\frac{d+2}{2d}}m^{\frac{1}{2d}},1\}m^{-\frac{2}{d+2}}\geq\min\{\nu^2,m^{-\frac{2}{d+2}}\}.\label{eqn:corol_min}
		\end{align}
		If $\nu^{\frac{d+2}{2d}}m^{\frac{1}{2d}}\geq 1$, we have $\nu^2\geq m^{-\frac{2}{d+2}}$. In this case both sides of (\ref{eqn:corol_min}) equals $m^{-\frac{2}{d+2}}$. If $\nu^{\frac{d+2}{2d}}m^{\frac{1}{2d}}\leq 1$, we have $m\leq \nu^{-(d+2)}$, and thus LHS of (\ref{eqn:corol_min}) have term $\nu^{\frac{d+2}{2d}}m^{\frac{1}{2d}}m^{-\frac{2}{d+2}}\geq \nu^2$, which equals RHS.
	\end{proof}
	\subsection{Proof of Theorem \ref{thm:upper_lr}}
	
	\begin{proof}
		We first list properties of log-concave distributions:
		\begin{theorem}[\cite{lovasz2007geometry,awasthi2014power}]
			\label{thm:logconcave}
			The following statements hold for an isotropic log-concave distribution $\distrX$:
			\begin{enumerate}
				\item Projections of $\distrX$ onto subspaces of $\mathbb{R}^d$ are isotropic log-concave. \label{point:proj}
				\item $\mathbb{P}[\|X\|_2\geq \alpha \sqrt{\dimension} ]\leq e^{1-\alpha}$. \label{point:1d}
				\item There is an absolute constant $C$ such that for any two unit vectors $u$ and $v$ in $\mathbb{R}^d$ we have
				$C\|v-u\|_2 \leq \mathbb{P}(\sign(u \cdot X) \ne \sign(v \cdot X))$. \label{point:angle}
			\end{enumerate}
		\end{theorem}

		From property of $\clsalgo$ and point \ref{point:angle} in Theorem \ref{thm:logconcave} we can get $\|\estweightvec-\gtweightvec\|_2\leq C\erralgo_{\clsalgo}(\numcomp, \delta/4)$ using $\numcomp$ comparisons, with probability $1-\delta/4$. We use a shorthand notion $\errorcomppart=C\erralgo_{\clsalgo}(\numcomp, \delta/4)$ for this error.
		Now consider estimating $\gtweightnorm$. The following discussion is conditioned on a fixed $\estweightvec$ that satisfies $\|\estweightvec-\gtweightvec\|_2\leq \errorcomppart$. For simplicity let $\baseval_i=\inprod{\estweightvec}{\featureRV}_i$. We have
		\begin{align*}
		\estweightnorm&= \frac{\sum_{i=1}^\numlabel \baseval_i \labelRV_i}{\sum_{i=1}^\numlabel \baseval_i^2 }\\
		&=\frac{\sum_{i=1}^\numlabel \baseval_i \gtweightnorm \inprod{\gtweightvec}{\featureRV_i}+\baseval_i \labelnoise_i}{\sum_{i=1}^\numlabel \baseval_i^2 }\\
		&=\gtweightnorm+\frac{\sum_{i=1}^\numlabel \baseval_i \gtweightnorm \inprod{\gtweightvec-\estweightvec}{\featureRV_i}+\baseval_i \labelnoise_i}{\sum_{i=1}^\numlabel \baseval_i^2 }.
		\end{align*}
		
		Now we have
		\begin{align*}
		\inprod{\gtweight - \estweight}{\featureRV}
		=&\;\gtweightnorm \inprod{\gtweightvec}{\featureRV}-\estweightnorm \inprod{\estweightvec}{\featureRV}\\
		=&\;\gtweightnorm \inprod{\gtweightvec-\estweightvec}{\featureRV}-\frac{\sum_{i=1}^\numlabel \baseval_i \gtweightnorm \inprod{\gtweightvec-\estweightvec}{\featureRV_i}+\baseval_i\labelnoise_i}{\sum_{i=1}^\numlabel \baseval_i^2 } \inprod{\estweightvec}{\featureRV}.
		\end{align*}
		So
		\begin{align}
		&\mathbb{E}\left[\inprod{\gtweight - \estweight}{\featureRV}^2\right]\nonumber\\
		\leq&\; 3 \mathbb{E}\left[\left(\gtweightnorm \inprod{\gtweightvec-\estweightvec}{\featureRV}\right)^2\right]+3\mathbb{E}\left[\left(\frac{\sum_{i=1}^\numlabel \baseval_i \gtweightnorm \inprod{\gtweightvec-\estweightvec}{\featureRV_i}}{\sum_{i=1}^\numlabel \baseval_i^2 } \inprod{\estweightvec}{\featureRV}\right)^2\right]+3\left[\left(\frac{\sum_{i=1}^\numlabel \baseval_i \labelnoise_i}{\sum_{i=1}^\numlabel \baseval_i^2 } \inprod{\estweightvec}{\featureRV}\right)^2\right]\label{eqn:terms}
		\end{align}

		The first term can be bounded by 
		\begin{align*}
		(\gtweightnorm)^2\mathbb{E}[\inprod{\estweightvec-\gtweightvec}{\featureRV}^2]=(\gtweightnorm)^2\|\estweightvec-\gtweightvec\|^2_2\leq (\gtweightnorm)^2\errorcomppart^2.
		\end{align*}
		For the latter two terms, we first bound the denominator $\sum_{i=1}^\numlabel \baseval_i^2$ using Hoeffding's inequality. 
		Firstly since $\|\estweightvec\|_2=1$, from point \ref{point:proj} 
		in Theorem \ref{thm:logconcave}, each $\baseval_i$ is also isotropic log-concave. 
		Now using point \ref{point:1d} in Theorem \ref{thm:logconcave} with $\alpha=1-\log (\delta/(4e\numlabel))$ we get that with probability $1-\delta/4$, $\baseval_i\leq \log(4e\numlabel/\delta)$ for all $i\in \{1,2,\ldots,\numlabel\}$. Let $E^\baseval_\delta$ denote this event, and $\distrX'$ is the distribution of $X$ such that $\baseval_i\leq \log(4e\numlabel/\delta)$.
		Now using Hoeffding's inequality, under $E^\baseval_\delta$  for any $t>0$
		\begin{align*}
		\mathbb{P}\left[\left|\frac{1}{\numlabel}\sum_{i=1}^\numlabel \baseval_i^2-\mathbb{E}_{\distrX'}[\inprod{\estweightvec}{X}^2]\right|\geq t\right]
		\leq \exp\left(-\frac{2\numlabel t^2}{\log^2(4e\numlabel/\delta)}\right).
		\end{align*}
		Note that $\mathbb{E}_{\distrX'}[\inprod{\estweightvec}{X}^2]\leq \mathbb{E}_{\distrX}[\inprod{\estweightvec}{X}^2]=1$. Also we have
		\begin{align*}
		1=\mathbb{E}[\baseval_i^2]&\leq \mathbb{E}_{\distrX'}[T_i^2]+\int_{\log^2(4em/\delta)}^{+\infty} t\mathbb{P}[\baseval_i^2\geq t]dt\\
		&\leq \mathbb{E}_{\distrX'}[T_i^2]+\int_{\log^2(4em/\delta)}^{+\infty} te^{-\sqrt{t}+1}dt\\
		&\leq \mathbb{E}_{\distrX'}[T_i^2]+\frac{3\delta}{2m}
		\end{align*}
		Let $t=\frac{1}{4}\mathbb{E}_{\distrX'}[\inprod{\estweightvec}{X}^2]$, we have
		\begin{align}
		\sum_{i=1}^\numlabel \baseval_i^2\leq \left[\frac{3\numlabel}{4}\mathbb{E}_{\distrX'}[\inprod{\estweightvec}{X}^2],\frac{5\numlabel}{4}\mathbb{E}_{\distrX'}[\inprod{\estweightvec}{X}^2]\right]\subseteq [\numlabel/2,2\numlabel].  \label{eqn:hoeffding}
		\end{align}
		with probability $1-\delta/4$, when $\numlabel=\Omega(\log^3(1/\delta))$. 
		
		
		\newcommand{\event}{E}
		Let $\event_\delta$ denote the event when $\numlabel/2\leq \sum_{i=1}^\numlabel \baseval_i^2\leq 2\numlabel$ and $\baseval_i$ are bounded by $\log(4e\numlabel/\delta)$ for all $i$ . Condition on $\event_\delta$ for the second term in (\ref{eqn:terms}) we have
		\begin{align*}
		\mathbb{E}\left[\left(\frac{\sum_{i=1}^\numlabel \baseval_i \gtweightnorm\inprod{\gtweightvec-\estweightvec}{\featureRV_i}}{\sum_{i=1}^\numlabel \baseval_i^2 }\inprod{\estweightvec}{\featureRV}\right)^2\right]
		\leq&\; \frac{\mathbb{E}\left[\left(\sum_{i=1}^\numlabel \baseval_i \gtweightnorm\inprod{\gtweightvec-\estweightvec}{\featureRV_i}\right)^2\right]}{\frac{\numlabel^2}{4}}\mathbb{E}[(\inprod{\estweightvec}{\featureRV})^2] \\
		=&\;\frac{4\mathbb{E}\left[\left(\sum_{i=1}^\numlabel \baseval_i \gtweightnorm\inprod{\gtweightvec-\estweightvec}{\featureRV_i}\right)^2\right]}{\numlabel^2}\\
		\end{align*}
		Now notice that $\frac{\estweightvec-\gtweightvec}{\|\estweightvec-\gtweightvec\|_2}X$ is also isotropic log-concave; using point \ref{point:1d} in Theorem \ref{thm:logconcave} we have with probability $1-\delta/4$, $(\estweightvec-\gtweightvec)^TX_i\leq \|\estweightvec-\gtweightvec\|_2\log(4e\numlabel/\delta)$ for all $i\in \{1,2,\ldots,\numlabel\}$. So
		\begin{align*}
		\mathbb{E}\left[\left(\sum_{i=1}^\numlabel \baseval_i \gtweightnorm\inprod{\gtweightvec-\estweightvec}{\featureRV_i}\right)^2\right]
		\leq &\; (\gtweightnorm)^2\errorcomppart^2\log^2(4e\numlabel/\delta)\mathbb{E}\left[\left(\sum_{i=1}^\numlabel \left|\baseval_i\right|\right)^2\right]\\
		\leq &\;(\gtweightnorm)^2\errorcomppart^2\log^2(4e\numlabel/\delta)\mathbb{E}\left[\numlabel\sum_{i=1}^\numlabel \baseval^2_i\right]\\
		=&\;(\gtweightnorm)^2\errorcomppart^2\log^2(4e\numlabel/\delta)\numlabel^2
		\end{align*}
		
		For the third term in (\ref{eqn:terms}), also conditioning on $\event_\delta$ we have
		\begin{align*}
		\mathbb{E}\left[\left(\frac{\sum_{i=1}^\numlabel \baseval_i \labelnoise_i}{\sum_{i=1}^\numlabel \baseval_i^2 }\inprod{\estweightvec}{\featureRV}\right)^2\right]
		=&\;\mathbb{E}\left[\left(\frac{\sum_{i=1}^\numlabel \baseval_i \labelnoise_i}{\sum_{i=1}^\numlabel \baseval_i^2 }\right)^2\right]\mathbb{E}\left[\inprod{\estweightvec}{\featureRV}^2\right]\\
		\leq &\; \frac{\mathbb{E}\left[\left(\sum_{i=1}^\numlabel \baseval_i \labelnoise_i\right)^2\right]}{\frac{\numlabel^2}{4}}\\
		\leq &\; \frac{4\mathbb{E}\left[\sum_{i=1}^\numlabel \baseval^2_i \sigma^2\right]}{\numlabel^2}=\frac{4\sigma^2}{\numlabel}.
		\end{align*}
		Combining the three terms and considering all the conditioned events, we have
		\begin{align*}
		\mathbb{E}\left[\left(\inprod{\gtweight}{\featureRV}- \inprod{\estweight}{\featureRV}\right)^2\right]
		\leq &4(\gtweightnorm)^2\errorcomppart^2+(\gtweightnorm)^2\errorcomppart^2\log^2(4e\numlabel/\delta)+\frac{4\sigma^2}{\numlabel}+C'\delta\\
		\leq &O\left(\frac{1}{\numlabel}+\log^2(\numlabel/\delta)\erralgo_{\clsalgo}(\numcomp, \delta/4)+\nu^2+\delta \right)
		\end{align*}
		Taking $\delta=\frac{4}{\numlabel}$ obtain our desired result.

	\end{proof}
	\subsection{Proof of Theorem \ref{thm:lowern} \label{sec:proof_lower_n_lr}}
		Our proof ideas come from \cite{castro2008minimax}.
		We use Le Cam's method, explained in the lemma below:
		\newcommand{\KL}{\text{KL}}
		\begin{lemma}[Theorem 2.2, \cite{tsybakov2009introduction}]
			\label{lemma:lecam}
			Suppose $\mathcal{P}$ is a set of distributions parametrized by $\theta\in \Theta$. $P_0,P_1\in \mathcal{P}$ are two distributions, parametrized by $\theta_0,\theta_1$ respectively, and $\KL(P_1,P_2)\leq \alpha\leq \infty$. Let $d$ be a semi-distance on $\Theta$, and $d(\theta_0,\theta_1)=2a$. Then for any estimator $\widehat{\theta}$ we have
			\begin{align*}
			\inf_{\widehat{\theta}} \sup_{\theta\in \Theta}\mathbb{P}[d(\widehat{\theta},\theta)\geq a]
			\geq&\; \inf_{\widehat{\theta}} \sup_{j\in \{0,1\}}\mathbb{P}[d(\widehat{\theta},\theta_j)\geq a] \\
			\geq&\; \max\left(\frac{1}{4}\exp(-\alpha),-\frac{1-\sqrt{\alpha/2}}{2} \right)
			\end{align*}
		\end{lemma}
		\newcommand{\vecpre}[1]{\mathbf{#1}}
		\newcommand{\dimvec}[1]{^{(#1)}}
		\newcommand{\smallconst}{\xi}
		We consider two functions: $\gtweight_0=(\smallconst,0,0,\ldots,0)^T$ and $\gtweight_1=(\frac{1}{\sqrt{\numlabel}},0,0,\ldots,0)^T$, where $\smallconst$ is a very small constant. Note that for these two functions comparisons provide no information about the weights (comparisons can be carried out directly by comparing $x\dimvec{1}$, the first dimension of $x$). So differentiating $\gtweight_0$ and $\gtweight_1$ using two oracles is the same as that using only active labels. We have $d(\gtweight_0,\gtweight_1)=E\left[\left((\gtweight_0-\gtweight_1)^TX\right)^2\right]=\left(\frac{1}{\sqrt{\numlabel}}-\smallconst\right)^2$. For any estimator $\estweight$, let $\{(\featureRV_i,\labelRV_i)\}_{i=1}^{\numlabel}$ be the set of samples and labels obtained by $\estweight$. Note that $\featureRV_{j+1}$ might depend on $\{(\featureRV_i,\labelRV_i)\}_{i=1}^j$. Now for KL-divergence we have
		\begin{align*}
		\KL(P_1,P_0)=&\;\mathbb{E}_{P_1}\left[\log\frac{P_1\left(\{(\featureRV_i,\labelRV_i)\}_{i=1}^{\numlabel}\right)}{P_0\left(\{(\featureRV_i,\labelRV_i)\}_{i=1}^{\numlabel}\right)} \right]\\
		=&\;\mathbb{E}_{P_1}\left[\log\frac{\prod_{j=1}^\numlabel P_1(Y_j|X_j)P(X_j|\{(\featureRV_i,\labelRV_i)\}_{i=1}^{j})}{\prod_{j=1}^\numlabel P_1(Y_j|X_j)P(X_j|\{(\featureRV_i,\labelRV_i)\}_{i=1}^{j})} \right]\\
		=&\;\mathbb{E}_{P_1}\left[\log\frac{\prod_{j=1}^\numlabel P_1(\labelRV_j|\featureRV_j)}{\prod_{j=1}^\numlabel P_0(\labelRV_j|\featureRV_j)} \right]\\
		=&\;\sum_{i=1}^{\numlabel} \mathbb{E}_{P_1}\left[\mathbb{E}_{P_1}\left[\log\left.\frac{\prod_{j=1}^\numlabel P_1(\labelRV_j|\featureRV_j)}{\prod_{j=1}^\numlabel P_0(\labelRV_j|\featureRV_j)}\right|\featureRV_1,\ldots,\featureRV_\numlabel\right]\right]\\
		\leq&\; n\max_{x}\mathbb{E}_{P_1}\left[\log\left.\frac{\prod_{j=1}^\numlabel P_1(\labelRV_j|\featureRV_j)}{\prod_{j=1}^\numlabel P_0(\labelRV_j|\featureRV_j)}\right|\featureRV_1=x\right].
		\end{align*}
		The third equality is because decision of $\featureRV_j$ is independent of the underlying function giving previous samples. Note that given $\featureRV=\featureV$, $\labelRV$ is normally distributed; by basic properties of Gaussian we have
		\[\mathbb{E}_{P_1}\left[\log\left.\frac{\prod_{j=1}^\numlabel P_1(\labelRV_j|\featureRV_j)}{\prod_{j=1}^\numlabel P_0(\labelRV_j|\featureRV_j)}\right|\featureRV_1=x\right]=\frac{(\frac{1}{\sqrt{\numlabel}-\smallconst})^2}{2\sigma^2}. \]
		Now by taking $\smallconst$ sufficiently small we have for some constants $C_1,C_2$,
		\[\KL(P_1,P_0)\leq C_1, d(\theta_0,\theta_1)\geq \frac{C_2}{\numlabel}. \]
		Combining with Lemma \ref{lemma:lecam} we obtain the lower bound.

	\subsection{Lower Bounds for Total Number of Queries under Active Case\label{sec:proof_lowerall_lr}}
	\begin{theorem}
		\label{thm:lowerall}
		For any (active) estimator $\estweight$ with access to $\numlabel$ labels and $\numcomp$ comparisons, there exists a ground truth weight $\existweight$ and a global constant $C$, such that when $\gtweight=\existweight$ and $2\numcomp+\numlabel< d$,
		$$\mathbb{E}\left[\inprod{\estweight-\gtweight}{X}^2\right]\geq C.$$
	\end{theorem}
	\medsp
	Theorem \ref{thm:lowerall} shows a lower bound on the total number of queries in order to get low error. Combining with Theorem \ref{thm:lowern}, in order to get a MSE of $\targeterror$  for some $\targeterror<C$, we need to make at least $O(1/\targeterror+\dimension)$ queries (i.e., labels+comparisons). Note that for the upper bound in Theorem \ref{thm:upper_lr}, we need $\numlabel+\numcomp=\Ologlog(1/\targeterror+\dimension\log(\dimension/\targeterror))$ for Algorithm \ref{algo:linreg} to reach $\targeterror$ MSE, when using \cite{awasthi2014power} as $\clsalgo$ (see Table \ref{tab:cls_results}). So Algorithm \ref{algo:linreg} is optimal in terms of total queries, up to log factors.
	
	The proof of Theorem \ref{thm:lowerall} is done by considering an estimator with access to $\numlabel+2\numcomp$ \emph{noiseless} labels $\{(\featureV_i, \gtweight\cdot \featureV_i)\}_{i=1}^{\numlabel+2\numcomp}$, which can be used to generate $\numcomp$ comparisons and $\numlabel$ labels. 
	We sample $\gtweight$ from a prior distribution in $\ball{0}{1}$, and show that the expectation of MSE in this case is at least a constant. Thus there exists a weight vector $\existweight$ that leads to constant error. 

	\begin{figure}
		\centering
		\includegraphics[width=0.23\textwidth]{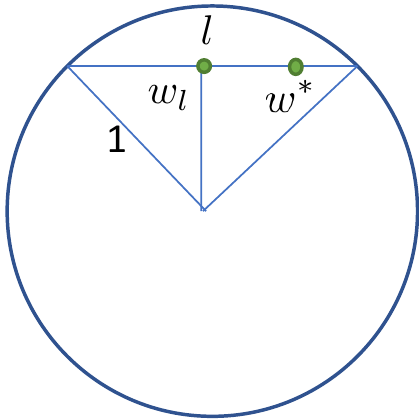}
		\caption{Graphic illustration about the sampling process in the proof of Theorem \ref{thm:lowerall}.\label{fig:prooflowerall}}
	\end{figure}
	
	\begin{proof}
	We just prove the theorem for $2\numcomp+\numlabel=d-1$. Note that this case can be simulated by considering an estimator with access to $2\numcomp+\numlabel$ \emph{truthful} samples; that is, $Y_i=\gtweight\cdot X_i$ for $i=1,2,\ldots,2\numcomp+\numlabel$. In this way truthful comparisons can be simulated by query two labels. We now prove a lower bound for this new case with $2\numcomp+\numlabel=d-1$ truthful samples.
	
	\newcommand{\distrwgt}{\mathcal{P}_{\gtweight}}
	We randomly sample $\gtweight$ as below: first uniformly sample $\gtweightvec$ on the surface of $\ball{0}{1}$, and then uniformly sample $\gtweightnorm\in [0,1]$. Let this distribution be $\distrwgt$. Since we only have $\dimension-1$ labels, for any set of samples $\featureV_1,\ldots,\featureV_{\dimension-1}$ there exists a line $l\in \ball{0}{1}$ such that every $w\in l$ produces the same labels on all the samples. Not losing generality, suppose such $l$ is unique (if not, we can augment $\estweight$ such that it always queries until $l$ is unique).
	Now for any active estimator $\estweight$, let $\featureRV_1,\ldots,\featureRV_{\dimension-1}$ denote the sequence of queried points when $\gtweight$ is randomly picked as above. Now note that for every $w$,
	\begin{align*}
	\mathbb{P}[\gtweight=w|\{\featureRV_i,\labelRV_i\}_{i=1}^{d-1},l]
	=\mathbb{P}[\gtweight=w|\{\featureRV_i,\labelRV_i\}_{i=1}^{d-1}]
	\propto \mathbb{P}[\gtweight=w]I(w\in l).
	\end{align*}
	The first equality is because $l$ is a function of $\{\featureRV_i,\labelRV_i\}_{i=1}^{d-1}$; the second statement is because all $w\in l$ produces the same dataset on $\featureRV_1,\ldots,\featureRV_{\dimension-1}$, and every $w\not\in l$ is impossible given the dataset. Notice that $\gtweightnorm$ is uniform on $[0,1]$; so with probability at least a half, the resulting $l$ has distance less than $1/2$ to the origin (since $l$ contains $\gtweight$, and $\|\gtweight\|$ is uniform on $[0,1]$). Denote by $w_l$ the middle point of $l$ (see Figure \ref{fig:prooflowerall}). For any such line $l$, the error is minimized by predicting the middle point of $l$: Actually we have
	\begin{align}
	\mathbb{E}\left[\inprod{\gtweight-\estweight}{\featureRV}^2|\text{$l$ has distance less than 1/2}\right]
	\geq \int_{u=0}^{|l|/2}u^2dP(\|\gtweight-\widehat{w}\|_2\geq u|\gtweight\in l)\\. \nonumber\\
	\geq \int_{u=0}^{|l|/2}u^2dP(\|\gtweight-w_l\|_2\geq u|\gtweight\in l)
	\label{eqn:mselower}
	\end{align}
	Note that the distribution of $\gtweight\in l$ is equivalent to that we sample from the circle containing $l$ and centered at origin, and then condition on $\gtweight\in l$ (see Figure \ref{fig:prooflowerall}). Notice that this sampling process is the same as when $d=2$; and with some routine calculation we can show that (\ref{eqn:mselower}) is a constant $C$. So overall we have
	\begin{align*}
	\mathbb{E}\left[\inprod{\gtweight-\estweight}{X}^2\right]\geq \frac{1}{2}C,
	\end{align*}
	where the expectation is taken over randomness of $\gtweight$ and randomness of $\estweight$. Now since the expectation is a constant, there must exists some $w$ such that
	\begin{align*}
	\mathbb{E}\left[\left. \inprod{\gtweight-\estweight}{X}^2\right|\gtweight=w\right]\geq \frac{1}{2}C,
	\end{align*}
	which proves the theorem.
	\end{proof}
	
	\section{Auxiliary Technical Results}
	We use some well-known technical results in our proofs and collect them here to improve readability.
	We use the Craig-Bernstein inequality~\citep{craig1933tchebychef}:
	\begin{lemma}\label{lem:bernstein}
		Suppose we have $\{X_1,\ldots,X_n\}$ be independent random variables and suppose that for $k \ge 2$, for some $r > 0$,
		\begin{align*}
		\mathbb{E}[\left|X_i - \mathbb{E}[X_i]\right|^k] \le \frac{\var(X_i)}{2}k!r^{k-2}.
		\end{align*} Then with probability at least $1-\delta$, for any $c < 1$ and for any $t \leq c/r$ 
		we have that:
		\begin{align*}
		\frac{1}{n}\sum_{i=1}^{n}\left(\mathbb{E}[X_i] - X_i\right) \le \frac{\log(1/\delta)}{n t} + \frac{t~\var(X_i)}{2(1-c)}.
		\end{align*}
	\end{lemma}
	
\end{document}